\documentclass[11pt,letterpaper]{article}

\usepackage{arxiv_deepthink/deepthink}
\usepackage[nameinlink]{cleveref}
\usepackage[
  backend=biber,
  style=alphabetic,
  maxbibnames=100,
  minbibnames=100,
  maxcitenames=2,
  mincitenames=2
]{biblatex}
\usepackage{latexsym}
\usepackage{color}
\usepackage{bbm}
\usepackage{algorithm}
\usepackage{algpseudocode}
\usepackage{subcaption}
\usepackage[mathscr]{euscript}
\usepackage{dsfont}
\usepackage[overload]{empheq}

\usepackage{hyperref}       % hyperlinks
\usepackage{cleveref}
\usepackage{url}            % simple URL typesetting
\usepackage{booktabs}       % professional-quality tables
\usepackage{bm}
\usepackage{graphicx}
\usepackage{makecell}

\newtheorem{theorem}{Theorem}%[section] %(If you want theorem numbered
\newtheorem{lemma}{Lemma}%[section] %%    with section number.
\newtheorem{prop}{Proposition}%[section]
\newtheorem{definition}{Definition}%[section]
\newtheorem{assum}{Assumption}
\newcommand{\R}{\mathbb{R}}

\newcommand{\e}{\begin{equation}}
\newcommand{\ee}{\end{equation}}
\newcommand{\en}{\begin{equation*}}
\newcommand{\een}{\end{equation*}}
\newcommand{\eqn}{\begin{eqnarray}}
\newcommand{\eeqn}{\end{eqnarray}}
\newcommand{\bmat}{\begin{bmatrix}}
\newcommand{\emat}{\end{bmatrix}}
\DeclareMathAlphabet\mathbfcal{OMS}{cmsy}{b}{n}
\newcommand{\E}{\operatorname{\mathbb{E}}}

\newcommand{\vct}[1]{\boldsymbol{#1}}
\newcommand{\mtx}[1]{\boldsymbol{#1}}
\newcommand{\rank}{\operatorname{rank}}

\DeclareMathOperator*{\argmin}{\text{arg~min}}

\hypersetup{
    colorlinks=true,%
    citecolor=blue,%
    filecolor=blue,%
    linkcolor=blue,%
    urlcolor=blue
}

\newcommand{\vx}{\vct{x}}

\newcommand{\mN}{\mtx{N}}
\newcommand{\mO}{\mtx{O}}

\graphicspath{{./figs/}}

\newlength{\imgwidth}
\newboolean{twoColVersion}
\setboolean{twoColVersion}{false}
\newcommand{\twoCol}[2]{\ifthenelse{\boolean{twoColVersion}} {#1} {#2} }

\long\def\comment#1{}

\def\E{\mathop{\rm E\,\!}\nolimits}

\newcommand{\ba}{\boldsymbol{a}}

\newcommand{\bx}{\boldsymbol{x}}

\newcommand{\bI}{\boldsymbol{I}}

\long\def\red#1{\bgroup\color{red}#1\egroup}

\definecolor{mich-blue}{HTML}{0027CC}
\definecolor{mich-blue-high}{HTML}{0027CC}
\definecolor{red-high}{HTML}{CA2020}
\definecolor{green-high}{HTML}{20A520}
\definecolor{mich-maize}{HTML}{FFCB05}
\definecolor{law-stone}{HTML}{655A52}
\definecolor{burton-beige}{HTML}{9B9A9D}
\definecolor{arch-ivy}{HTML}{7E732F}

 \colorlet{color1}{gray!15}

\newenvironment{proof}{\par\noindent{\bf Proof\ }}{\hfill\BlackBox\\[2mm]}
\newcommand{\BlackBox}{\rule{1.5ex}{1.5ex}}  % end of proof
\def\MoLRG{\texttt{MoLRG}}
\def\MoG{\texttt{MoG}}
\newcommand{\rebuttal}[1]{\textcolor{black}{ #1}}
\title{Breaking the Curse of Dimensionality: Diffusion Models Efficiently Learn Low-Dimensional Distributions}

\newcommand{\jointfirst}{\textsuperscript{\dag}}
\newcommand{\corrauth}{\textsuperscript{\ddag}}

\authorblock{
  \href{https://peng8wang.github.io/}{Peng Wang}\jointfirst\textsuperscript{1},
  \href{https://www.huijiezh.com/}{Huijie Zhang}\jointfirst\textsuperscript{2},
  \href{https://la0ka1.github.io/}{Zekai Zhang}\textsuperscript{2},
  \href{https://chicychen.github.io/}{Siyi Chen}\textsuperscript{2},
  \href{https://people.eecs.berkeley.edu/~yima/}{Yi Ma}\textsuperscript{3},
  \href{https://qingqu.engin.umich.edu/}{Qing Qu}\corrauth\textsuperscript{2}
}

\affiliation{
  \textsuperscript{1}University of Macau \quad $\cdot$ 
  \quad  \textsuperscript{2}University of Michigan \quad $\cdot$
  \quad \textsuperscript{3}University of California, Berkeley
}

\authornote{
  \jointfirst\ Joint first author \quad
  \corrauth\ Corresponding author
}

\abstracttext{
\noindent Despite their empirical success across a wide range of generative tasks, the fundamental principles underlying the ability of diffusion models to learn data distributions are poorly understood. In this work, we develop a new mathematical framework that explains how diffusion models can effectively learn low-dimensional distributions from a finite number of training samples without suffering from the curse of dimensionality. Specifically, motivated by the intrinsic low-dimensional structure of image data, we theoretically analyze a setting in which the data distribution is modeled as a mixture of low-rank Gaussians. Under suitable network parameterization, we show that optimizing the training objective of diffusion models is equivalent to solving the canonical subspace clustering problem over the training samples, where each subspace basis corresponds to the low-rank covariance of a Gaussian component. This equivalence allows us to show that the sample complexity for learning the underlying distribution scales linearly with the intrinsic dimension of the data, rather than exponentially with the ambient dimension. Our theoretical findings are further supported by empirical evidence that demonstrates phase transition phenomena in generalization on both synthetic and real-world image datasets. Moreover, we establish a correspondence between the learned subspace bases and semantic attributes of image data, providing a principled foundation for controllable image generation.
}


\date{\today}
\correspondence{\href{mailto:peng8wang@gmail.com}{peng8wang@gmail.com}}
\resources{\quad \href{https://www.huijiezh.com/diffusion-model-subspace-clustering/index.html}{Project page}  \quad $\cdot$ \quad \href{https://github.com/huijieZH/Diffusion-Model-Generalizability}{Code}}

\headerlogo{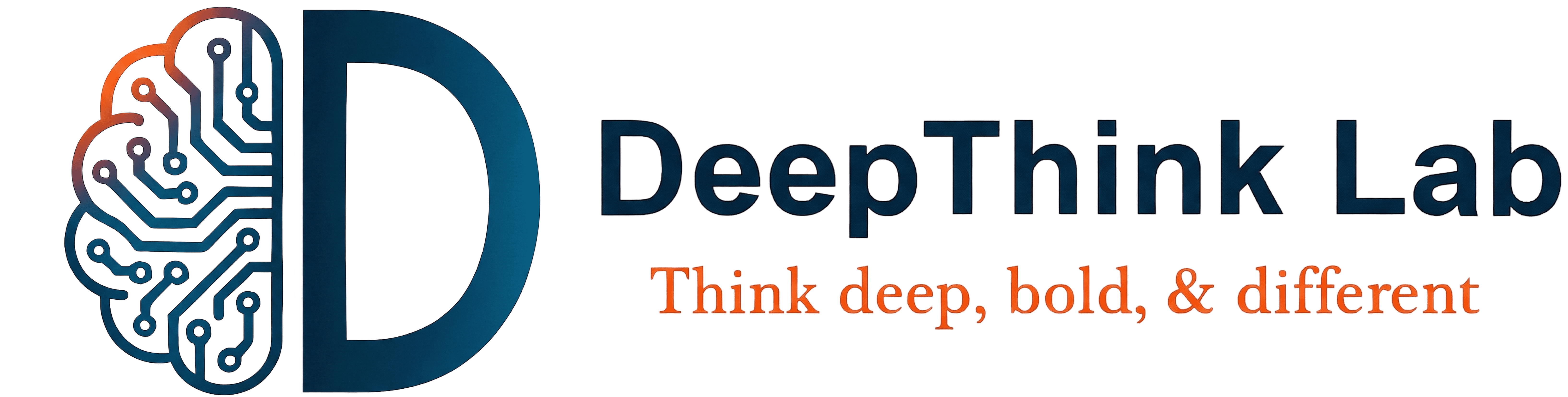}{https://deepthink-umich.github.io}

\begin{document}
\vspace*{-1in}
\makeDeepthinkHeader

% ------------------------------------------------------------
% Optional: teaser figure
% ------------------------------------------------------------
% \begin{figure}[h]
%   \centering
%   \IfFileExists{Deepthink_landscape_do_not_delete_compressed.png}{%
%     \includegraphics[width=0.9\linewidth]{Deepthink_landscape_do_not_delete_compressed.png}
%   }{%
%     \fbox{\parbox[c][3.2cm][c]{0.9\linewidth}{\centering Teaser image placeholder}}
%   }
%   \caption{Teaser caption (replace me).}
% \end{figure}

\vspace{-0.1in}
\begin{figure}[h]
\begin{center}
    \includegraphics[width = 1\linewidth]{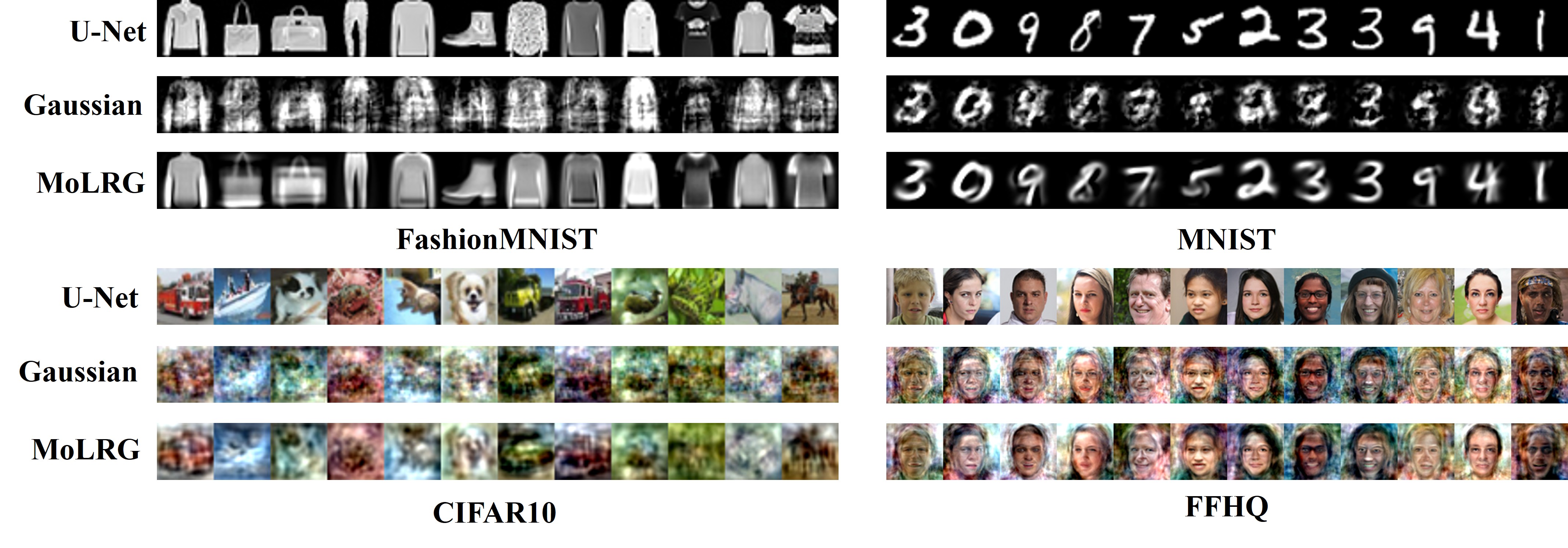}\vspace{-0.1in}
    \caption{\textbf{Comparison of images generated from the Gaussian, \MoLRG, and the distribution learned by diffusion models across different datasets.} 
    Each row displays images generated from different distributions using the reverse-time ODE sampler, including the Gaussian, \MoLRG, and the distribution learned by U-Net. 
    The columns represent images generated from the same initial noise. The results are shown for four datasets: FashionMNIST (top left), MNIST (top right), CIFAR-10 (bottom left), and FFHQ (bottom right).
    \label{fig:MoLRG_data_model_verfication}}%\qq{I feel we should remove the closest point in the training data, the 2nd row}  
\end{center}
\end{figure}

\newpage
\tableofcontents
\newpage

\section{Introduction}\label{sec:intro}

Generative modeling is a fundamental task in deep learning that seeks to learn the underlying data distribution from training samples to generate new and realistic data. Among recent advances, diffusion models have emerged as a powerful class of generative models, achieving remarkable performance across a wide range of domains, including image generation \cite{ho2020denoising,wang2023dr2}, video synthesis \cite{bar2024lumiere,xing2024survey}, speech and audio generation \cite{kong2020hifi,kongdiffwave2021}, and solving inverse problems \cite{fabian2024adapt,chung2023diffusion}. In general, diffusion models learn a data distribution from training samples through a process that imitates the nonequilibrium thermodynamic diffusion process \cite{ho2020denoising,sohl2015deep,song2020score}. Specifically, a diffusion model operates in two stages: (i) a forward process, in which Gaussian noise is gradually added to the training data over a sequence of time steps; and (ii) a reverse process, in which the noise is progressively removed by a neural network trained to approximate the score function, i.e., the gradient of the log probability density function (pdf) of the data, at each time step.

Despite the great empirical success of diffusion models and recent advances in understanding their sampling convergence \cite{bentonnearly,lee2022convergence,li2025convergence,li2024sharp,li2025unifed}, distribution approximation \cite{chen2023score,oko2023diffusion,wang2024the}, memorization \cite{gu2023memorization,somepalli2023diffusion,wen2023detecting}, and generalization \cite{han2024neural,kadkhodaie2023generalization,zhang2024emergence}, the mechanisms underlying their performance remain poorly understood. This is primarily due to the black-box nature of neural networks and the inaccessibility of real-world data distributions. In particular, a fundamental question arises: Can diffusion models truly learn the underlying data distribution? If so, how many samples are required to achieve this? Recent theoretical studies \cite{oko2023diffusion,wibisono2024optimal} have shown that learning an arbitrary probability distribution using diffusion models inevitably suffers from the curse of dimensionality. Specifically, if the underlying density belongs to a generic class of probability distributions, obtaining an $\epsilon$-accurate estimate of the corresponding score function requires the number of training 
samples that scales as $O(\epsilon^{-n})$, where $n$ is the ambient data dimension. However, recent empirical studies \cite{zhang2024emergence,kadkhodaie2023generalization} have shown that diffusion models can effectively learn image data distributions and generate novel and semantically meaningful samples distinct from the training data,  even when trained on far fewer samples than those suggested by the existing theoretical bounds. As such, the gap between theory and practice raises a key question: 
\begin{center}\em
    When and why can diffusion models learn data distributions \\ without suffering from the curse of dimensionality? 
\end{center}

\subsection{Our Contributions}\label{subsec:cont}

In this work, we address the above question by investigating how diffusion models learn data distributions with intrinsic low-dimensional structures. Unlike previous studies \cite{oko2023diffusion, wibisono2024optimal}, which considered arbitrary distributions, our study focuses on low-dimensional distributions, motivated by the observation that real-world image data often lie on a union of low-dimensional manifolds despite their high ambient dimensions \cite{brownverifying, kamkari2024geometric,loaiza-ganem2024deep}. These structures arise from underlying symmetries, repetitive patterns, and local regularities in natural images, which reduce the degrees of freedom in the data \cite{gong2019intrinsic, pope2020intrinsic}. To effectively capture the low-dimensional structure of real-world image data while offering analytical tractability, we focus on a mixture of low-rank Gaussians (\MoLRG; see \Cref{def:MoG}). Notably, our focus is further supported by empirical evidence in \Cref{fig:MoLRG_data_model_verfication}, which demonstrates that samples generated via the diffusion reverse sampling process—using our theoretically constructed model—closely resemble those produced by U-Net \cite{ronneberger2015u} trained on the same dataset and initialized with the same noise. 

Theoretically, we show that diffusion models can learn the \MoLRG\;distribution, provided that the minimum number of training samples scales linearly with the intrinsic dimension of the data, thereby overcoming the curse of dimensionality. Our result is established by demonstrating the equivalence between the training loss of diffusion models and the canonical subspace clustering problem \cite{vidal2011subspace, wang2022convergence} (see \Cref{thm:3}) under an appropriate parameterization for the denoising autoencoder. {\color{black} Suppose that each component in the \MoLRG\;distribution has zero mean and covariance matrix $\bm U_k\bm U_k^T$, where each low-rank matrix $\bm U_k$ has orthonormal columns and the subspaces spanned by $\{\bm U_k\}$ are mutually orthogonal. Notably, these assumptions imply that the data are distributed around a union of low-dimensional linear subspaces, with isotropic variation within each subspace and no cross-subspace correlations. Here, the orthogonality assumption serves as an idealized yet analytically tractable abstraction that captures well-separated components in high-dimensional spaces. Under these assumptions,} our theory demonstrates a \emph{phase transition} in the ability of diffusion models to learn the underlying distributions, which occurs when the number of training samples exceeds the intrinsic dimensionality of the data-generating subspaces (see \Cref{thm:4}). Moreover, our theoretical analysis offers valuable practical insights, as highlighted below. 
\begin{itemize}[leftmargin=*]
    \item \emph{The phase transition of generalization on image datasets.} As shown in \Cref{sec:appli_1}, when training diffusion models on real-world image datasets, we observe a similar phase transition from failure to success in generalization, where the model begins to generate new and sensible images distinct from the training data once the number of training samples exceeds a threshold that scales linearly with the intrinsic dimension of the data. Our study of \MoLRG~offers key insights into understanding this phenomenon.
 %   a critical point where the number of training samples \pw{scales linearly with the intrinsic dimension}. 
    \item \emph{Correspondence between subspace bases and semantic task vectors.\footnote{A semantic task vector is a direction in the latent or intermediate feature space such that traversing along it causes a controlled and interpretable change in the output.}} We find that the basis vectors of the subspaces identified through our theoretical analysis align with semantically meaningful directions, that is, task vectors, in diffusion models pretrained on real-world image datasets. These semantic task vectors enable control over attributes such as gender, hairstyle, and color in the generated images (see \Cref{fig:meta_1}). This insight has inspired new training-free image editing methods on pretrained diffusion models \cite{Chen2024}. 
\end{itemize}

Our study of distribution learning is closely related to recent studies on the generalization of diffusion models. It is well understood that when generative models successfully learn the true underlying data distribution, they exhibit strong generalization capabilities, enabling them to generate new samples that differ from the training data \cite{arora2017generalization,li2024generalization,kadkhodaie2023generalization}. 
% In classical machine learning theory, accurately learning the underlying distribution typically guarantees good generalization \cite{arora2017generalization}. This principle also applies to diffusion models: when they successfully learn the true underlying data distribution, they can generate new samples that are distinct from the training data, indicating a strong generalization capability \pw{edit}. Conversely,
Moreover, recent empirical studies \cite{zhang2024emergence,kadkhodaie2023generalization} have shown that strong generalization in diffusion models often corresponds to an accurate approximation of the underlying distribution, as evidenced by reproducibility. Specifically, these studies observed that different diffusion models can reproduce each other’s outputs while generating new samples distinct from the training data, even when trained with different architectures, loss functions, and non-overlapping subsets of the training data.  Motivated by these discussions, this work considers generalization in diffusion models as their ability to accurately capture the underlying data distribution. In this sense, our work also contributes to the theoretical understanding of generalization by characterizing the sample complexity required for diffusion models to learn the underlying distribution. 

% \subsection{Relationship with Prior Arts} \label{subsec:prior-work}

% Note that we defer the discussion on the relationship between our results and the related works to \Cref{sec:prior-work}. 

\subsection{Notation and Organization}

\paragraph{Notation.} We write matrices in bold capital letters, such as $\bm A$, vectors in bold lower-case letters, such as $\bm a$, and scalars in plain letters, such as $a$. Given a matrix $\bm A$, we use $\|\bm A\|$ to denote its largest singular value (i.e., spectral norm), $\sigma_{i}(\bm A)$ its $i$-th largest singular value, $a_{ij}$ its $(i,j)$-th entry, $\mathrm{rank}(\bm A)$ its rank, and $\|\bm A\|_F$ its Frobenius norm. Given a vector $\bm a$, we use $\|\bm a\|$ to denote its Euclidean norm and $a_i$ to denote its $i$-th entry. Let $\mathcal{O}^{n\times d}$ denote the set of all $n\times d$ {\color{black} matrices that have orthogonal columns}.  We simply write the score function $\nabla_{\bm x} \log p(\bm x)$ of a distribution with pdf $p(\bm x)$ as $\nabla \log p(\bm x)$. We denote by $\mathcal{N}(\bm \mu, \bm \Sigma)$  a multivariate Gaussian distribution with mean $\bm \mu$ and covariance $\bm \Sigma \succeq \bm 0$. 

\paragraph{Organization.} In \Cref{sec:setup}, we introduce the preliminaries of diffusion models and state our assumptions regarding the data and model. In \Cref{sec:results}, we present the main results of this study. In \Cref{sec:prior-work}, we discuss how our results relate to the existing literature. In \Cref{sec:experiments}, we conduct numerical experiments to support our theory and demonstrate its practical implications. Finally, in \Cref{sec:conclusion}, we summarize our work and discuss potential directions for future research. All proofs are presented in the appendices.

\section{Problem Setup}\label{sec:setup}

In this section, we introduce preliminaries on diffusion models and state our assumptions on the data and model. We consider a training dataset $\{\bm x^{(i)}\}_{i=1}^N \subseteq \mathbb{R}^n$, where the data points are independently and identically distributed (\emph{i.i.d.}) samples from the underlying data distribution $p_{\rm data}(\bm x)$; that is, $\bm x^{(i)} \overset{\emph{i.i.d.}}{\sim} p_{\rm data}(\bm x)$. 

\subsection{Preliminaries on Score-Based Diffusion Models}\label{subsec:prelim} 

\paragraph{Forward and reverse processes of diffusion models.}
In general, diffusion models aim to learn a data distribution and generate new samples through forward and reverse processes indexed by a continuous time variable $t \in [0,1]$. Specifically, the forward process progressively injects noise into the data, which can be described by the following stochastic differential equation (SDE):  
\begin{align}\label{eq:forw}
    \mathrm{d}\bm x_t =  f(t) \bm x_t \mathrm{d}t + g(t) \mathrm{d} \bm{w}_t,
\end{align}
where $\bm x_0 \sim p_{\rm data}(\bm x)$, scalar functions $f(t), g(t): [0,1] \to \R$ denote the drift and diffusion coefficients, respectively, and $\{\bm{w}_t\}_{t \in [0,1]} $ is the standard Wiener process. 
For ease of exposition, let $p_t(\bm x)$ denote the \emph{pdf} of $\bm x_t$ and $p_{t}(\bm x_t\mid \bm x_0)$ be the transition kernel from $\bm x_0$ to $\bm x_t$.\footnote{Note that $p_0 := p_{\rm data}$.}  According to \eqref{eq:forw}, one can verify that 
\begin{align}\label{eq:trans}
   p_{t}(\bm x_t\mid \bm x_0) = \mathcal{N}\left(\bm x_t; s_t\bm x_0, s^2_t\sigma_t^2\bm I_n\right), 
\end{align}
where $s_t := \exp\left(\int_0^t f(\xi)\mathrm{d}\xi\right)$ and $\sigma_t := \sqrt{\int_{0}^{t} {g^2(\xi)}/{s^2({\xi})}\mathrm{d}\xi}.$\footnote{With a slight abuse of notation, we denote $s_t:=s(t)$ and $\sigma_t:=\sigma(t).$} 

The reverse process gradually removes the noise from $\bm x_{1}$ using the following probability-flow ordinary differential equation (ODE) backward in time:
\begin{align}\label{eq:reve}
   \mathrm{d}\bm{x}_t = \left(  f(t)\bm{x}_t - \frac{1}{2}g^2(t) \nabla \log p_t(\bm{x}_t)\right) \mathrm{d}t. 
\end{align}
If $\bm x_1 \sim p_1$ and the score function $\nabla \log p_t(\bm{x}_t)$ is known for all $t\in [0,1]$, then this reverse-time ODE has the same marginal distribution $p_t$ as the forward process at each time $t\in[0,1]$ \cite{song2021scorebased}.

\paragraph{Training loss of diffusion models.} Unfortunately, the score function $\nabla \log p_t$ at each $t \in [0,1]$ is typically unknown, as the marginal distribution $p_t$ is induced by the unknown data distribution $p_{\rm data}$. To enable data generation via the reverse-time ODE in \eqref{eq:reve}, we train a neural network to approximate the score function from the training data. In addition, Tweedie's formula \cite{efron2011tweedie} relates the score function $\nabla \log p_t(\bm x_t)$ to the posterior mean $\mathbb{E}\left[ \bm x_0 \mid \bm x_t \right]$ as follows: 
\begin{align} \label{eq:Tweedie}
    s_t\mathbb{E}\left[ \bm x_0 \mid \bm x_t \right] = \bm x_t + s_t^2\sigma_t^2 \nabla \log p_t(\bm x_t),\ \forall t \in (0,1].
\end{align}
This allows us to estimate the posterior mean $\mathbb{E}\left[ \bm x_0 \mid \bm x_t \right]$ as an alternative approach for estimating the score function $\nabla \log p_t(\bm x_t)$. Leveraging the strong function approximation capabilities of neural networks \cite{hornik1989multilayer}, recent studies \cite{chen2024deconstructing,kadkhodaie2023generalization,karras2022elucidating,wang2023patch,vincent2011connection} have explored training a time-dependent neural network $\bm x_{\bm \theta}(\cdot, t):\R^{n} \times [0,1] \to \R^n$  with parameters $\bm \theta$, referred to as the  {\em denoising autoencoder} (DAE), to estimate the posterior mean $\mathbb{E}\left[ \bm x_0 \mid \bm x_t \right]$. To learn the network parameters $\bm \theta$, we minimize the following empirical loss over the training samples $\{\bm x^{(i)}\}_{i=1}^N$:  
\begin{align}\label{eq:em loss}
\min_{\bm \theta}\ \ell(\bm \theta) = \frac{1}{N}\sum_{i=1}^N \int_0^1 \lambda_t \mathbb{E}_{\bm \epsilon \sim \mathcal{N}(\bm 0, \bm I_n)} & \left[\left\| \bm x_{\bm \theta}( s_t\bm x^{(i)} + \gamma_t \bm \epsilon,t)   -  \bm x^{(i)}\right\|^2\right] \mathrm{d}t,  
\end{align}
where $\gamma_t := s_t\sigma_t$ and $\lambda_t:[0,1] \to \R^+$ is the weighting function.   

\subsection{Mixture of Low-Rank Gaussians}\label{subsec:MoG} 

% \qq{we just say that we are movtivated by the low-dimensionality of real-world data, we study mixture of low-rank gausian, which is more analytically tractable}
%Because the distribution of real-world image datasets is generally unknown and difficult to characterize, quantitatively studying the distribution learning behavior of diffusion models on  image distributions can be challenging. To address this, 
In this work, we consider learning a mixture of low-rank Gaussians (\MoLRG), which effectively captures the intrinsic low-dimensional structure of real-world image datasets while maintaining analytical tractability. Specifically, the \MoLRG~distribution is defined as follows. 
\begin{definition}[Mixture of Low-Rank Gaussians] \label{def:MoG}
We say that a random vector $\bm x_0 \in \R^n$ follows a mixture of $K$ low-rank Gaussian distributions with mixing proportions $\{\pi_k\}_{k=1}^K$, means $\{\bm \mu_k^\star\}_{k=1}^K \subseteq \R^n$, and covariance matrices $\{\bm \Sigma^\star_k \}_{k=1}^K \subseteq \mathbb{R}^{n\times n}$ if its distribution is given by
\begin{align}\label{eq:MoG}
    \bm x_0 \sim \sum_{k=1}^K \pi_k \mathcal{N}(\bm \mu_k^\star, \bm \Sigma_k^\star), 
\end{align}
where $\pi_k \ge 0$ is the mixing proportion of the $k$-th component satisfying $\sum_{k=1}^K \pi_k = 1$, and $\bm \mu^\star_k$ and $\bm \Sigma_k^\star \succeq \bm 0$ denote the mean and covariance matrix of the $k$-th component, respectively. In particular, the covariance matrix $\bm \Sigma_k^\star$ is low-rank with $\mathrm{rank}(\bm \Sigma_k^\star) = d_k < n$.
\end{definition}  

\paragraph{Remarks.} Intuitively, data drawn from a \MoLRG\ distribution lie on a union of low-dimensional linear subspaces, where the $k$-th subspace is characterized by the mean $\bm \mu_k^\star$ and low-rank covariance matrix $\bm \Sigma_k^\star$. We now discuss the motivation for studying this model and its connections to other distributions that have been theoretically analyzed.
\begin{itemize}[leftmargin=*]
    % \vspace{-0.1in}
    \item \emph{\MoLRG~captures the low-dimensional structure of real-world image datasets.} Recent studies \cite{brownverifying,kamkari2024geometric} conducted extensive numerical experiments and demonstrated that image datasets, such as MNIST \cite{lecun1998gradient}, CIFAR-10 \cite{cifar10}, and ImageNet \cite{russakovsky2015imagenet}, approximately reside on a union of low-dimensional manifolds. Locally, each nonlinear manifold can be effectively approximated using its tangent space (i.e., a linear subspace). Consequently, the \MoLRG~model, which represents the data as a union of linear subspaces, provides a suitable local approximation for real-world image data distributions. This claim is supported by the empirical studies in \Cref{subsec:exp-real}.  In addition, the latent distribution of real-world data can be well approximated by a Gaussian, as modern diffusion models typically employ autoencoders with KL regularization to encourage alignment with a standard Gaussian prior \cite{kingma2013auto,rombach2022high}. This latent Gaussian structure, as adopted in the \MoLRG~model, also facilitates theoretical analysis, allowing us to derive the closed-form expression for the posterior estimator at each time step, as shown in \Cref{lem:E[x_0]}. Therefore, studying the \MoLRG~model is a valuable starting point for theoretical studies on the distribution learning capability of diffusion models. 
    % \vspace{-0.1in}
    \item \emph{Comparison with recent studies on a mixture of Gaussians.} %\qq{single Gaussian learning with covariance; mixture of isotropic Gaussian}
    Many recent studies have investigated how diffusion models learn a mixture of \emph{isotropic} Gaussians (\MoG), that is, $\bm \Sigma_k^\star = \bm I_n$ in \eqref{eq:MoG}; see, e.g., \cite{chen2024learning,cole2024score,gatmiry2024learning,shah2023learning,wu2024theoretical}. These studies mainly focus on learning the means of the Gaussian components, provided that each covariance matrix is fixed as the identity. 
    In contrast, our work considers a mixture of \emph{low-rank} Gaussians, where the key challenge lies in learning the low-rank covariance matrices instead of the means. The low-rankness captures the inherent low-dimensionality of image datasets \cite{gong2019intrinsic,pope2020intrinsic,stanczukdiffusion} and offers deeper insight into why diffusion models learn data distributions in practice without suffering from the curse of dimensionality. In addition, several studies have investigated the reverse sampling process of diffusion models based on a mixture of Gaussians. For example, \cite{biroli2024dynamical} analyzed a mixture of two Gaussians with distinct means and identical variance, revealing that the reverse diffusion process exhibits distinct dynamical regimes. In addition, \cite{li2025dimension} demonstrated that diffusion models can efficiently sample from high-dimensional distributions that are well approximated by a mixture of Gaussians. In comparison, our work focuses on the training process of diffusion models rather than the sampling process. %  \qq{in comparison, we focus on learning rather than sampling}
    
    Additionally, a single Gaussian, as a special case of a mixture of  Gaussians, has been extensively studied owing to its analytical tractability, despite its limited expressive power. For example, \cite{wang2023hidden,wang2024the,li2024understanding} empirically demonstrated that the score function of a well-trained diffusion model at a high-noise scale is well approximated by the score of a single Gaussian. In addition, \cite{chen2025denoising} uses a single Gaussian model to show that denoising score distillation can identify the eigenspace of the covariance matrix of a Gaussian. 
    
\end{itemize}

% \begin{figure*}[t]
% \begin{center}
%     \includegraphics[width = 1\linewidth]{figure/MoLRG_verify1.jpg}\vspace{-0.1in}
%     \caption{\textbf{Comparison of images generated from the Gaussian, \MoLRG, and the distribution learned by diffusion models across different datasets.} Each row displays images generated from different distributions using the reverse-time ODE sampler, including the Gaussian, \MoLRG, and the distribution learned by U-Net. The columns represent images generated from the same initial noise. The results are shown for four datasets: FashionMNIST (top left), MNIST (top right), CIFAR-10 (bottom left), and FFHQ (bottom right ).\label{fig:MoLRG_data_model_verfication}}%\qq{I feel we should remove the closest point in the training data, the 2nd row}  
% \end{center}
% \end{figure*} 

\subsection{Network Parameterization Inspired by \MoLRG}

To analyze the distribution learning behavior of diffusion models, one natural approach is to study the training loss in \eqref{eq:em loss}. This approach critically depends on a suitable parameterization of the DAE $\bm x_{\bm \theta}(\cdot,t)$. In practice, $\bm x_{\bm \theta}(\cdot,t)$ is typically parameterized by a U-Net architecture \cite{ronneberger2015u}, which consists of deep nonlinear encoder and decoder networks with skip connections. %, or a Transformer-based architecture \cite{peebles2023scalable}, which leverages the self-attention mechanism \cite{vaswani2017attention}. 
However, the highly nonlinear structure of U-Net poses significant challenges for theoretical analysis.  

To enable analytical tractability while retaining structural similarity to U-Net, we consider a network architecture for $\bm x_{\bm \theta}(\cdot,t)$ that is a combination of multiple one-layer linear encoders and decoders, weighted by a softmax-like function, as follows:
\begin{align}\label{eq:DAE para}
    \bm x_{\bm \theta}(\bm x_t,t) = \sum_{k=1}^K w_{k,t}(\bm x_t) \left( \bm \mu_k + \bm U_k \bm D_{k,t}  \bm U_k^{ T}\left( \frac{\bm x_t}{s_t} - \bm \mu_k \right) \right),
\end{align}
where $\bm \theta = \{(\pi_k, \bm \mu_k,\bm U_k, \bm \Lambda_k)\}_{k=1}^K$ denotes the network parameters, $\bm U_k \in \mathcal{O}^{n\times d_k}$ has orthonormal columns, and $\bm \Lambda_k = \mathrm{diag}(\lambda_{k,1},\dots,\lambda_{k,d_k})$ is a diagonal matrix. Additionally, $\bm D_{k,t}$ and $w_{k,t}(\bm x_t)$ are defined as follows:  % with $\bm D_{k,t}=\mathrm{diag}\left( \frac{s_t^2\lambda_{k,1}}{\gamma_t^2 + s_t^2\lambda_{k,1}},\dots, \frac{s_t^2\lambda_{k,d_k}}{\gamma_t^2 + s_t^2\lambda_{k,d_k}} \right)$ and $w_{k,t}(\bm x) := \frac{\pi_k\mathcal{N}\left(\bm x; s_t\bm \mu_k, s_t^2\bm U_k\bm \Lambda_k\bm U_k^T + \gamma_t^2 \bm I_n \right)}{\sum_{l=1}^K \pi_l\mathcal{N}\left(\bm x; s_t\bm \mu_l, s_t^2\bm U_l\bm \Lambda_l\bm U_l^T  + \gamma_t^2 \bm I_n \right)}$. 
\begin{align*}
    \bm D_{k,t} = \mathrm{diag}\left( \frac{s_t^2\lambda_{k,1}}{\gamma_t^2 + s_t^2\lambda_{k,1}},\dots, \frac{s_t^2\lambda_{k,d_k}}{\gamma_t^2 + s_t^2\lambda_{k,d_k}} \right),\ w_{k,t}(\bm x) = \frac{\pi_k\mathcal{N}\left(\bm x; s_t\bm \mu_k, s_t^2\bm U_k\bm \Lambda_k\bm U_k^T + \gamma_t^2 \bm I_n \right)}{\sum_{l=1}^K \pi_l\mathcal{N}\left(\bm x; s_t\bm \mu_l, s_t^2\bm U_l\bm \Lambda_l\bm U_l^T  + \gamma_t^2 \bm I_n \right)},
\end{align*}
where $\sigma_t$ and $\gamma_t$ are defined in \Cref{subsec:prelim}. 
% \qq{the last two are too complicated, can we decouple and write it as softmax, and define what is softmax instead.}
Our network architecture in \eqref{eq:DAE para} can be viewed as a mixture-of-experts architecture \cite{shazeer2017outrageously}, where each expert network  consists of a linear encoder $\bm U_k^T$ and decoder $\bm U_k$.
These experts are then combined through a learnable weighted summation, allowing the model to adaptively assign weights among components. % On the other hand, \eqref{eq:DAE para} also bears similarity to the multi-head attention mechanism in Transformer architectures. Here, for an input $\bm x_t$, each head $k$ is to compute the dependency
In addition to the resemblance to U-Net, the parameterization in \eqref{eq:DAE para} is well motivated from the following perspectives: 

\begin{itemize}[leftmargin=*]
    \item \emph{Inspired by the posterior mean of \MoLRG.} As the DAE serves as an estimator of the posterior mean (i.e., $\bm x_{\bm \theta}(\bm x_t,t) \approx \mathbb E[ \bm x_0 | \bm x_t]$), our network parameterization is inspired by the analytical form of the posterior mean  $\mathbb E[ \bm x_0 | \bm x_t]$ of \MoLRG; see \Cref{lem:E[x_0]} in Appendix \ref{app:pf sec2}. Note that the network parameters $\bm \theta$ in \eqref{eq:DAE para} are learnable instead of being the ground-truth of the means and covariances of the \MoLRG. In \Cref{sec:results}, we will investigate how to learn the network parameters in simplified settings. % \qq{adding: in xx, we will investigate how to learn the parameters of the network under a simplified setting} 
%    In this sense, \eqref{eq:DAE para} is the optimal parameterization of the DAE for learning \MoLRG. 
    \item \emph{Meaningful image generation through the parameterization.} When the network parameters $\bm \theta$ are directly estimated from training data, the experimental results in \Cref{subsec:exp-real} demonstrate that our parameterization \eqref{eq:DAE para} generates images that are coarsely similar to those produced by a standard U-Net. This demonstrates the practical effectiveness of the proposed parameterization on real-world datasets and supports its ability to capture coarse image structure. Further details are provided below. 
%Second, the experimental results presented in \Cref{subsec:exp-real} demonstrate that our parameterization \eqref{eq:DAE para} produces images that are strikingly similar to those generated by a standard U-Net. This highlights the practical effectiveness of the model in real-world tasks, providing further evidence of its potential for capturing complex image distributions.
\end{itemize}

\subsection{Experimental Support for Data and Model Assumptions}\label{subsec:exp-real}

As illustrated in \Cref{fig:MoLRG_data_model_verfication} and \Cref{tab:noise_to_image_mapping_comparison}, we empirically validate the \MoLRG\;assumption and the corresponding network parameterization introduced in \eqref{eq:DAE para} for approximating real-world data distributions. In our experiments, we used the distribution learned by U-Net as a benchmark. To quantify the similarity between the images generated by \eqref{eq:DAE para} and those produced by U-Net, we computed the following metric: 
\begin{equation}\label{eq:dist}
    \frac{1}{M} \sum_{i = 1}^{M} \left\|\bm y_1^{(i)} - \bm y_2^{(i)}\right\|,
\end{equation}
where $M$ denotes the number of generated samples, and $\bm y^{(i)}_{1}$ and $\bm y^{(i)}_2$ denote the $i$-th samples generated from the distributions learned by U-Net and the parameterization in \eqref{eq:DAE para}, respectively. Here, both sets of samples are generated using the reverse-time ODE in \eqref{eq:reve}, initialized with the same noise. %Our experimental results in \Cref{fig:MoLRG_data_model_verfication} on real-world image datasets—including MNIST \cite{MNIST}, FashionMNIST \cite{FashionMNIST}, CIFAR-10 \cite{cifar10}, and FFHQ \cite{karras2019style}—demonstrate that (i) the \MoLRG~distribution captures the intrinsic low-dimensional structure of image data and (ii) our network parameterization generates images that preserve overall layout and semantics, albeit with less accuracy in fine-grained details. 
Detailed experimental setups are provided in Appendix \ref{app sec:expsetting_sec2}. 

\begin{table}[t]
    \centering
    \begin{tabular}{lcccc}
        \hline
         & FashionMNIST & MNIST & CIFAR-10 & FFHQ \\
        \hline
        Gaussian  & 72.62 & 69.12 & 33.55 & 36.75 \\
        $\MoLRG$  & \textbf{57.56} & \textbf{62.53} & \textbf{31.29} & \textbf{35.78} \\
        \hline
    \end{tabular}\smallskip 
    % \begin{tabular}{lcccccc}
    %     \hline
    %      & FashionMNIST & MNIST & CIFAR-10 & FFHQ & \rebuttal{CIFAR-10 LF} & \rebuttal{FFHQ LF}\\
    %     \hline
    %     Gaussian  & 72.62 & 69.12 & 33.55 & 36.75 & \rebuttal{8.03} & \rebuttal{8.91} \\
    %     $\MoLRG$  & \textbf{57.56} & \textbf{62.53} & \textbf{31.29} & \textbf{35.78} & \rebuttal{\textbf{7.63}} & \rebuttal{\textbf{9.43}}  \\
    %     \hline
    % \end{tabular}\smallskip 
    % \caption{\textbf{Quantitative comparison between theoretical distributions (MoLRG or Gaussian) and the real diffusion model.} Images are generated from both the theoretical distributions and the real diffusion model. We report the Euclidean distance between the images generated from the same initial noise.}
        \caption{\textbf{Distance (defined in \eqref{eq:dist}) between samples generated from the theoretical network parameterization (based on MoLRG or Gaussian) and those generated by U-Net.}} 
    \label{tab:noise_to_image_mapping_comparison}
\end{table}

Based on the above experimental setup, we conducted experiments on real-world image datasets, including MNIST \cite{MNIST}, FashionMNIST \cite{FashionMNIST}, CIFAR-10 \cite{cifar10}, and FFHQ \cite{ffhq}. % \rebuttal{We additionally considered the low-frequency components of CIFAR-10 and FFHQ.} 
Our proposed model and network architecture were then compared against existing approaches.
\begin{itemize}[leftmargin=*]
    \item \emph{Comparison between our model and U-Net.} First, we compare images generated by U-Net trained on the real-world dataset $\{\bm x^{(i)}\}_{i=1}^N$ with those generated by our parameterized network in \eqref{eq:DAE para}, which uses the means and covariances estimated from the same dataset. As illustrated in \Cref{fig:MoLRG_data_model_verfication}, images generated by the two network parameterizations using the same sampling procedure exhibit substantial visual similarity, especially on simpler datasets such as FashionMNIST and MNIST. This observation supports the validity of our data assumptions and confirms the effectiveness of our network parameterization in approximating real-world distributions. On more complex datasets, such as CIFAR-10 and FFHQ, although our parameterized network cannot capture fine-grained image details, it preserves the overall structural characteristics of the images generated by U-Net. The loss of fine details indicates a limitation of our model assumptions, which merits further investigation.
    
    \item \emph{Comparison between our model and the single full-rank Gaussian parameterization.} In addition, we compared our model with a network parameterized according to a single full-rank Gaussian model, as explored in prior studies \cite{wang2024the,li2024understanding}. % In particular, \cite{li2024understanding} employs a linear parameterization similar to \eqref{eq:DAE para}, but focuses on a single full-rank Gaussian and estimates its mean and covariance matrix from the training data. %Previous literature \cite{wang2024the,li2024understanding} supports the examination of the full-rank Gaussian model, emphasizing an inductive bias of diffusion models toward Gaussian structures. 
     As illustrated in \Cref{fig:MoLRG_data_model_verfication}, single Gaussian parameterization often results in high intra-class variance and blurred images, particularly on simpler datasets such as FashionMNIST and MNIST (second row of the figure). In contrast, our model based on \MoLRG~significantly improves generation quality (third row) by leveraging multiple mixture components to mitigate intra-class variance and employing low-rank covariance structures to suppress high-frequency noise. Moreover, as shown in \Cref{tab:noise_to_image_mapping_comparison}, despite employing fewer parameters, our model consistently outperforms the single Gaussian model in terms of the distance to images generated by U-Net. 
    % The primary distinction between the two approaches lies in the Gaussian assumption employed: the baseline method utilizes a single Gaussian distribution characterized by a full-rank covariance matrix, whereas \MoLRG~leverages multiple degenerate Gaussian distributions. Despite employing fewer parameters, \MoLRG~consistently outperforms the full-rank Gaussian model in terms of image generation quality. 
\end{itemize} 

\section{Sample Complexity Analysis for Learning \MoLRG}\label{sec:results} 

Building upon the setup introduced in \Cref{sec:setup}, we theoretically analyze the sample complexity of learning the \MoLRG\ distribution via diffusion models. Specifically, we show that 
\begin{tcolorbox}
    \begin{itemize}[leftmargin=*]
        \item The training loss of diffusion models in \eqref{eq:em loss} under our parameterization is equivalent to the canonical subspace clustering problem.
        \item The minimum number of samples required for learning \MoLRG~via diffusion models scales linearly with the intrinsic data dimension.
    \end{itemize}
\end{tcolorbox} 

To simplify our analysis, we assume that $\bm \mu^\star_k = \bm 0$ and $\bm \Lambda_k^\star = \bm I_{d_k}$ for each $k \in [K]$ in the \MoLRG~model (see Definition \ref{def:MoG}). 
% {\color{black}\begin{assum}\label{AS:1}
%     In Definition \ref{def:MoG}, the \MoLRG\ distribution satisfies $\bm \mu^\star_k = \bm 0$ and $\bm \Lambda_k^\star = \bm I_{d_k}$ for each $k \in [K]$. 
% \end{assum}}
Because real-world images often contain noise owing to sensor imperfections or environmental conditions, we additionally incorporate an additive noise term into the \MoLRG~model. {\color{black}Consequently, we formalize the noisy data generation process under the \MoLRG~model as follows: 
% \begin{assum}\label{AS:1}
% The training samples $\{\bm x^{(i)}\}_{i=1}^N$ are generated according to 
% \begin{align}\label{eq:MoG noise}
%     \bm x^{(i)} = \bm U_k^\star \bm a_i + \bm e_i\;\text{with probability}\;\pi_k,\;\forall i \in [N], 
% \end{align}
% where $\bm a_i \overset{i.i.d.}{\sim} \mathcal{N}(\bm 0, \bm I_{d_k})$ denotes the linear combination coefficients for the orthonormal basis $\bm U_k^\star \in \mathcal{O}^{n\times d_k}$ and $\bm e_i \in \R^n$ is noise for each $i \in [N]$.\footnote{The signal component of this model exactly satisfies Definition \ref{def:MoG} because of $\bm U_k^\star \bm a_i \sim \mathcal{N}(\bm 0, \bm U_k^\star\bm U_k^{\star T})$.}
% \end{assum}}
\begin{assum}\label{AS:1}
For each $i \in [N]$, let $z_i \in [K]$ be a latent component label with $\mathbb P(z_i=k)=\pi_k$. Conditional on $z_i=k$, the training sample is generated according to
\begin{align}\label{eq:MoG noise}
    \bm x^{(i)} = \bm U_k^\star \bm a_i + \bm e_i,
\end{align}
where $\bm a_i \sim \mathcal{N}(\bm 0, \bm I_{d_k})$ denotes the latent coefficient vector, $\bm U_k^\star \in \mathbb R^{n\times d_k}$ has orthonormal columns, and $\bm e_i \in \R^n$ denotes an additive noise vector.\footnote{The signal component satisfies Definition \ref{def:MoG}, since $\bm U_k^\star \bm a_i \sim \mathcal{N}(\bm 0, \bm U_k^\star\bm U_k^{\star T})$.}
\end{assum}

Notably, because the \MoLRG~distribution is fully characterized by the first- and second-order moments of each degenerate Gaussian component, learning this distribution reduces to estimating the bases $\{\bm U_k^\star\}_{k=1}^K$ according to our setup. In the following, we demonstrate that this estimation can be achieved by minimizing the DAE training loss in Problem \eqref{eq:em loss} with respect to optimization variables $\{\bm U_k\}_{k=1}^K$. 

\subsection{A Warm-Up Study: Learning a Single Low-Rank Gaussian}\label{subsec:low Gau} 

To build intuition, we begin by introducing our result in a simple setting, where the underlying distribution $p_{\rm data}$ is a \emph{single} low-rank Gaussian, i.e., $K=1$ in \eqref{eq:MoG noise}. Specifically, the training samples $\{\bm x^{(i)}\}_{i=1}^N$ are generated according to   
\begin{align}\label{eq:Ua+e} 
    \bm x^{(i)} = \bm U^\star \bm a_i + \bm e_i, 
\end{align}
where $\bm U^\star \in \mathcal{O}^{n\times d}$ denotes an orthonormal basis, $\bm a_i \overset{i.i.d.}{\sim} \mathcal{N}(\bm 0, \bm I_d)$ is the coefficient for each $i \in [N]$, and $\bm e_i \in \R^n$ is noise for all $i \in [N]$.
According to \eqref{eq:DAE para}, the parameterization of the DAE in this case reduces to
\begin{align}\label{eq:para Gau}
    \bm x_{\bm \theta}(\bm x_t, t) =  \frac{s_t}{s_t^2 + \gamma_t^2} \bm U\bm U^{T}  \bm x_t,
\end{align}
where $\bm \theta = \bm U \in \mathcal{O}^{n\times d}$. Equipped with the above setup, we obtain the following results.   
\begin{theorem}\label{thm:1}
Suppose that the DAE $\bm x_{\bm \theta}(\cdot,t)$ in Problem (\ref{eq:em loss}) is parameterized into \eqref{eq:para Gau} for each $t \in [0,1]$. Then, Problem (\ref{eq:em loss}) is equivalent to the following principal component analysis (PCA) problem:
    \begin{align}\label{eq:PCA}
        \max_{\bm U \in \R^{n\times d}} \sum_{i=1}^N \left\| \bm U^T \bm x^{(i)}\right\|^2\qquad \mathrm{s.t.}\quad \bm U^T\bm U = \bm I_d. 
    \end{align}
\end{theorem}
We defer the proof to \Cref{app:pf thm1}. In this case, \Cref{thm:1} shows that training diffusion models with the network parameterization \eqref{eq:para Gau} is equivalent to performing PCA on the training samples. Note that PCA is a classical and well-studied method for learning low-dimensional subspaces, whose optimal solution can be computed via singular value decomposition (SVD). This closed-form solution allows us to leverage existing results, such as Wedin's Theorem \cite{wedin1972perturbation}, to facilitate our analysis. Consequently, we can apply classical tools to analyze the sample complexity of learning the underlying distribution with diffusion models as follows.

\begin{figure}[t]
\begin{center}
    \begin{subfigure}{0.43\textwidth}
        \includegraphics[width=1\textwidth]{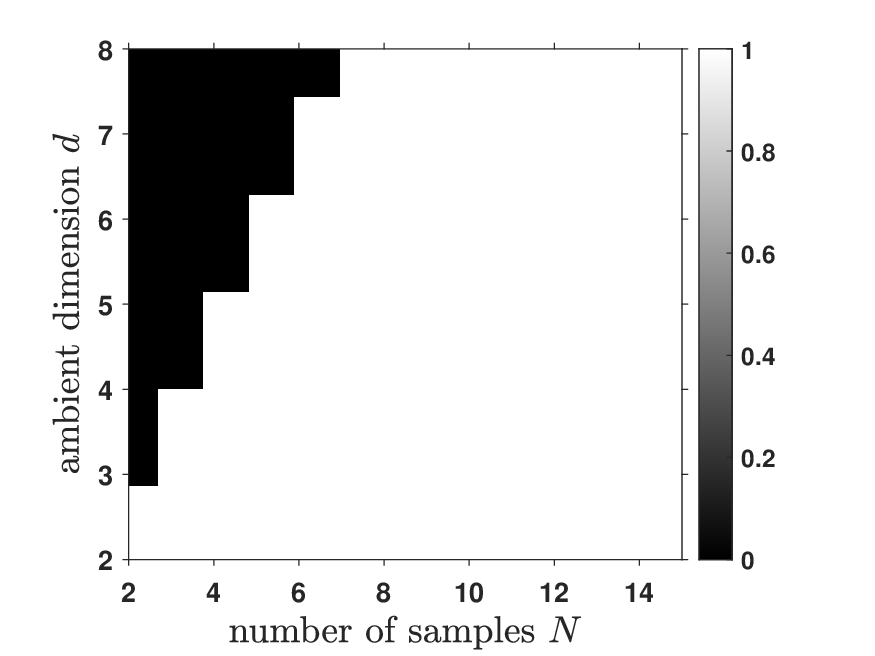} 
    \caption{\bf\footnotesize PCA} 
    \end{subfigure} 
    \begin{subfigure}{0.43\textwidth}
    	\includegraphics[width = 1\linewidth]{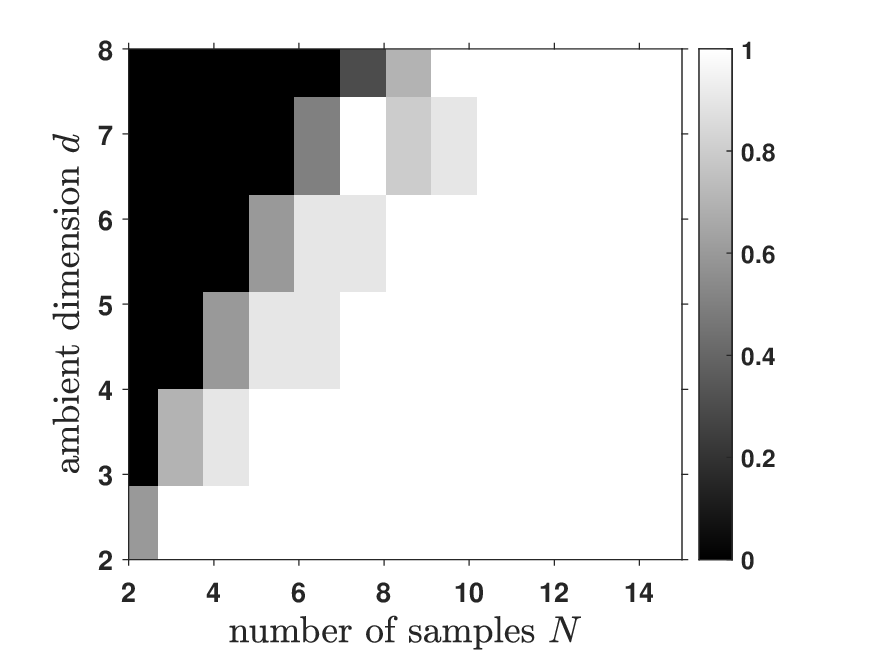} 
    \caption{\bf\footnotesize Diffusion Model} 
    \end{subfigure} \vspace{-0.1in}
    \caption{\textbf{Phase transition of learning the \MoLRG~distribution with $K = 1$.} The $x$-axis is the number of training samples and $y$-axis is the dimension of subspaces. Darker pixels represent a lower empirical probability of success. We apply SVD and stochastic gradient descent to solve Problems \eqref{eq:PCA} and \eqref{eq:em loss}, visualizing the results in (a) and (b), respectively.}  \label{fig:phase-transition-MoG-k=1} 
\end{center}
\end{figure}

\begin{theorem}\label{thm:2}
    Under the same setting of \Cref{thm:1}, {\color{black}suppose that \Cref{AS:1} holds}. Let $\hat{\bm U}$ denote an optimal solution of Problem (\ref{eq:em loss}). The following statements hold: 
    \begin{itemize} 
        \item[i)] If $N \ge d$, it holds with probability at least $1 - 1/2^{N-d+1} - \exp\left(-c_2N\right) $ that any optimal solution $\hat{\bm U}$ satisfies
   \begin{align}\label{rst1:thm 2}
       \left\|\hat{\bm U}\hat{\bm U}^T - \bm U^\star\bm U^{\star T}\right\|_F \le \frac{c_1\sqrt{\sum_{i=1}^N \|\bm e_i\|^2}}{\sqrt{N} - \sqrt{d-1}},
   \end{align}
   where $c_1, c_2 > 0$ are constants. 
   \item[ii)] If $N < d$, there exists an optimal solution $\hat{\bm U} \in \mathcal{O}^{n\times d}$ such that with probability at least $1 - 1/2^{d-N+1} - \exp\left(-c_2^\prime d\right)$, 
   \begin{align}\label{rst2:thm 2}
       \left\|\hat{\bm U}\hat{\bm U}^T - \bm U^\star\bm U^{\star T}\right\|_F \ge \alpha - 
       \frac{c_1^\prime\sqrt{\sum_{i=1}^N \|\bm e_i\|^2}}{\sqrt{d} - \sqrt{N-1}},
   \end{align}
   where $\alpha := \sqrt{2\min\{d-N, n-d\}}$ and $c_1^\prime, c_2^\prime > 0$ are constants . 
    \end{itemize}
\end{theorem} 
We defer the proof to Appendix \ref{app:pf thm2}. Next, we discuss the implications of our results. 
\begin{itemize}[leftmargin=*]
\item {\em Phase transition in learning the underlying distribution.} Building on the equivalence in \Cref{thm:1} and the data model in \eqref{eq:Ua+e}, \Cref{thm:2} clearly demonstrates a phase transition from failure to success of learning the underlying distribution via diffusion models as the number of training samples increases. More precisely, when the number of training samples is larger than the intrinsic dimension of the subspace, i.e., $N \geq d$, any optimal solution $\hat{\bm U}$ recovers the underlying subspace up to an approximation error determined by the noise level. Conversely, when $N < d$,  the training loss admits an optimal solution that fails to recover the underlying subspace. This phase transition is further corroborated by our experiments in \Cref{fig:phase-transition-MoG-k=1}.
Finally, because a Gaussian distribution can be fully characterized by its first two moments, our result rigorously shows that diffusion models can recover the underlying distribution when 
$N \geq d$, with the covariance estimation error bounded by the noise level

\item {\em The connections between PCA and semantic task vectors.} The correspondence between principal components and semantic meaning has been well studied in machine learning literature. For example, early work \cite{turk1991eigenfaces} demonstrated that PCA can reveal meaningful components of variation in natural image datasets,  such as facial expressions, lighting, or pose, implying a connection between directions of maximal variance and human-perceived semantics. Inspired by this insight, our empirical results in \Cref{sec:appli_2} reveal a similar phenomenon in diffusion models: task vectors can be identified through the leading singular vectors of the Jacobian of the DAE, which can effectively capture distinct semantic features of natural images for controlling the image generation. %\qq{we need to provide more description of our experimental findings here.} To further support this perspective, we present experimental evidence in \Cref{sec:appli_2}. 
% Our finding laid the groundwork for interpreting principal components as capturing dominant semantic directions in diffusion models.
\end{itemize} 

\subsection{Learning a Mixture of Low-Rank Gaussians}\label{subsec:thm MoG}  

In this subsection, we extend the above study to the \MoLRG~distribution with $K>1$. For the ease of analysis, {\color{black} we impose the following additional assumption on \Cref{AS:1}:
\begin{assum}\label{AS:2}
Under \Cref{AS:1}, the subspace bases satisfy $\bm U_k^{\star T} \bm U_l^\star = \bm 0$ for all $k \neq l$, the subspace dimensions are equal, i.e., $d_1=\dots=d_K=d$, and the mixing weights satisfy $\pi_1=\dots=\pi_K=1/K$.   
\end{assum}
This assumption enforces a symmetric and well-separated multi-subspace structure, which simplifies the analysis while preserving the essential geometric characteristics of the model.}
Moreover, we consider a hard-max counterpart of \eqref{eq:DAE para} for the DAE parameterization as follows:
{\color{black}
\begin{assum}\label{AS:3}
The DAE is parameterized as
\begin{align}\label{eq:para MoG}
     \bm x_{\bm \theta}(\bm x_t, t)  = \dfrac{s_t}{s_t^2 + \gamma_t^2} \sum_{k=1}^K \hat{w}_k(\bm \theta; \bm x_0) \bm U_k\bm U_k^T\bm x_t, 
    \end{align}
where $\bm \theta = \{\bm U_k\}_{k=1}^K$, $\bm U = [\bm U_1,\dots,\bm U_K] \in \mathcal{O}^{n\times dK }$, and $\{ \hat{w}_k(\bm \theta; \bm x_0) \}_{k=1}^K$ are defined as 
\begin{align}\label{eq:wk1}
\hat{w}_k(\bm \theta; \bm x_0) = \begin{cases}
1,\ & \text{if}\ k=k_0, \\
0,\ & \text{otherwise},
\end{cases}
\end{align}
where $k_0
= \min\left\{
j\in[K]:
j\in\arg\max_{l\in[K]}\|\bm U_l^T\bm x_0\|
\right\}.$  
\end{assum}
This assumption replaces the softmax weights in \eqref{eq:DAE para} with hardmax weights based on the clean training sample $\bm x_0$. This simplification facilitates the analysis while capturing the dominant-component behavior of the softmax weights. We refer the reader to Appendix \ref{app sec:para MoG} for a discussion of how these hard-max weights approximate the softmax weights in \eqref{eq:DAE para}. Now, we are ready to present the following theorem. 

\begin{figure}[t]
\begin{center}
    \begin{subfigure}{0.43\textwidth}
    	\includegraphics[width=1\textwidth]{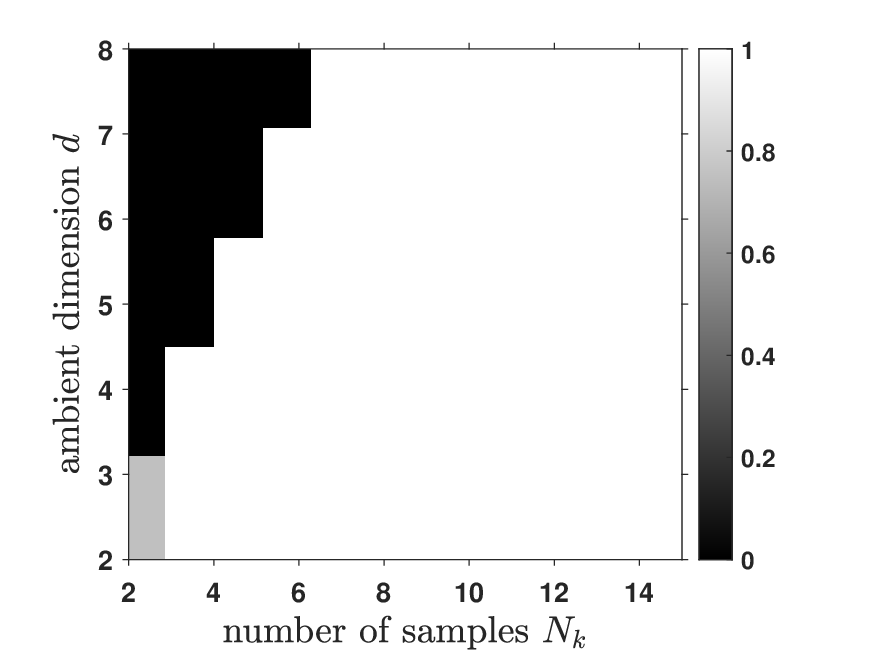} 
    \caption{\bf\footnotesize Subspace Clustering} 
    \end{subfigure}  
    \begin{subfigure}{0.43\textwidth}
    	\includegraphics[width = 1\linewidth]{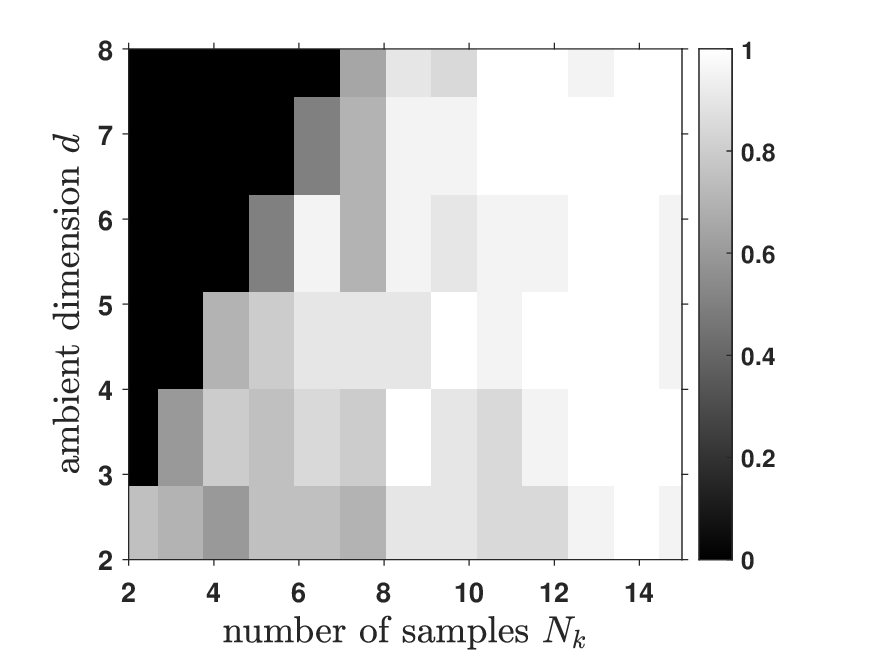} 
    \caption{\bf\footnotesize \rebuttal{Diffusion Model}} 
    \end{subfigure}\vspace{-0.1in} 
    \caption{\textbf{Phase transition of learning the \MoLRG~distribution with $K = 2$.} The $x$-axis is the number of training samples and $y$-axis is the dimension of subspaces. Darker pixels represent a lower empirical probability of success. We apply a subspace clustering method and stochastic gradient descent to solve Problems \eqref{eq:SC} and \eqref{eq:em loss}, visualizing the results in (a) and (b), respectively. Additional experiments for the case when $K = 3$ are presented in \Cref{fig:phase-transition-MoG-add-exp}.}  \label{fig:phase-transition-MoG-k=2} 
\end{center}
\vspace{-0.1in}
\end{figure}

\begin{theorem}\label{thm:3}
    %\begin{theorem}\label{thm:3}
   Suppose that \Cref{AS:3} holds.
    Then, Problem (\ref{eq:em loss}) is equivalent to the following subspace clustering problem: 
    \begin{align}\label{eq:SC}
        \max_{\bm \theta} \frac{1}{N} \sum_{k=1}^K \sum_{ i \in C_k(\bm \theta)} \|\bm U_k^T\bm x^{(i)}\|^2\qquad \mathrm{s.t.}\ \left[
            \bm U_1, \dots, \bm U_K \right] \in \mathcal{O}^{n\times dK}, 
    \end{align}
where 
\begin{align}\label{eq:Ck}
    C_k(\bm\theta)
:=
\left\{
i\in[N]:
k= \min\left\{
j\in[K]:
j\in\arg\max_{l\in[K]}
\|\bm U_l^T\bm x^{(i)}\|
\right\}
\right\},\ \forall k\in[K].
\end{align}
\end{theorem} 
We defer the proof of this theorem to Appendix \ref{app:pf thm3}. In this theorem, the set \(C_k(\bm \theta)\) contains the data points assigned to the \(k\)-th learned subspace, where each data point is assigned to the subspace with the largest projection norm; in the case of ties, it is assigned to the one with the smallest index. Problem (\ref{eq:SC}) seeks to maximize the sum of squared projection norms of data points onto their assigned subspaces.  With the network parameterization in \eqref{eq:para MoG}, \Cref{thm:3} shows that optimizing the training loss of diffusion models is equivalent to solving the subspace clustering problem \cite{vidal2011subspace,wang2022convergence}.  Notably, subspace clustering is a fundamental problem in unsupervised learning, which aims to identify and group data points that lie in a union of low-dimensional subspaces in a high-dimensional space \cite{vidal2011subspace,lerman2018overview}. By showing an equivalence between training diffusion models and subspace clustering in \Cref{thm:3}, we can characterize the minimum number of samples required for learning the underlying \MoLRG~distribution, similar to \Cref{thm:2}. 

Before proceeding, we clarify the transition from mixing weights to component sample sizes. Each sample has a latent label \(z_i\in[K]\) with \(\mathbb P(z_i=k)=\pi_k\). We define
\[
C_k^\star:=\{i\in[N]:z_i=k\},
\quad
N_k:=|C_k^\star|.
\]
Thus, \(\pi_k\) denotes the population-level mixing proportion, whereas \(N_k\) is the realized number of samples from the \(k\)-th component. Under \Cref{AS:2}, \((N_1,\dots,N_K)\) follows a multinomial distribution with equal probabilities \(1/K\), and \(N_{\min}:=\min_{k\in[K]}N_k\) is of order \(N/K\) with high probability. Therefore, the theorem below states the recovery guarantees in terms of the realized counts \(N_k\), which determine the recoverability of each subspace.
 
\begin{theorem}\label{thm:4}
Suppose that Assumptions \ref{AS:1}, \ref{AS:2}, and \ref{AS:3} hold, $d\gtrsim \log N$, and
\begin{align}\label{eq:noise thm4}
\max_{i\in[N]}\|\bm e_i\|
\le
c_\delta
\min\left\{
\frac{N_{\min}}{K^3N\sqrt d},
\frac{\sqrt d}{N}
\right\},
\end{align}
where $c_\delta>0$ is a sufficiently small absolute constant. Then the following statements hold:
\begin{itemize}
\item[(i)] Suppose in addition that $N_{\min}\gtrsim K^6\bigl(d\log K + \log N\bigr).$ Then, with probability at least \(1-N^{-\Omega(1)}\), for any optimal solution
$\{\hat{\bm U}_k\}_{k=1}^K$ of Problem \eqref{eq:em loss}, there exists a permutation
$\Pi:[K]\to[K]$ such that, for all $k\in[K]$,
\begin{align}\label{rst1:thm 4}
\left\|\hat{\bm U}_{\Pi(k)}\hat{\bm U}_{\Pi(k)}^T
-
\bm U_k^\star\bm U_k^{\star T}
\right\|_F
\le
\frac{c_1 \sqrt{\sum_{i \in C_k^\star} \|\bm e_i\|^2}}
{\sqrt{N_k} - \sqrt{d-1}},
\end{align}
where $c_1>0$ is a constant.

\item[(ii)]  If $N_{k_0}<d$ for some $k_0\in[K]$, then with probability at least $1-2^{-(d-N_{k_0}+1)}-\exp(-c_2^\prime d),$ the PCA problem over the true cluster \(C_{k_0}^\star\) admits an optimal solution \(\tilde{\bm U}_{k_0}\in\mathcal O^{n\times d}\) such that
\begin{align}\label{rst2:thm 4}
\left\|
\tilde{\bm U}_{k_0}\tilde{\bm U}_{k_0}^T
-
\bm U_{k_0}^{\star}\bm U_{k_0}^{\star T}
\right\|_F
\ge
\beta
-
\frac{c_1^\prime\sqrt{\sum_{i\in C_{k_0}^\star}\|\bm e_i\|^2}}
{\sqrt{d} - \sqrt{N_{k_0}-1}},
\end{align}
where \(\beta:=\sqrt{2\min\{d-N_{k_0},n-d\}}\) and
\(c_1^\prime,c_2^\prime>0\) are constants.
\end{itemize}
\end{theorem}
We defer the proof to Appendix~\ref{app:pf thm4}. This result extends \Cref{thm:2} from a single low-rank Gaussian component to the multi-component \MoLRG~setting. When each component has sufficiently many samples, the global optimizer recovers the true mixture assignments up to a permutation, after which the problem decomposes into \(K\) component-wise PCA problems. Conversely, \Cref{thm:4}(ii) shows that if one component has fewer than \(d\) samples, then even oracle PCA with true assignments cannot uniquely recover its subspace. Together, these results identify a sample-complexity threshold for learning the underlying subspaces, which we empirically validate in \Cref{fig:phase-transition-MoG-k=2}. We now discuss the implications of our results.

\begin{itemize}[leftmargin=*]
    % \item \emph{Phase transition in learning the underlying distribution.} This theorem demonstrates that when the number of samples in each subspace exceeds the dimension of the subspace and the noise is bounded, the optimal solution of the training loss \eqref{eq:em loss} under the parameterization \eqref{eq:para MoG} can recover the underlying subspaces up to the noise level. Conversely,  when the number of samples is insufficient, there exists an optimal solution that may recover wrong subspaces; see Figures \ref{fig:phase-transition-MoG}(c) and \ref{fig:phase-transition-MoG}(d). % This theorem provides insight into the empirically observed transition from memorization to generalization \emph{w.r.t.} the number of training samples \cite{zhang2023emergence,kadkhodaie2023generalization}. Similar to the single low-rank Gaussian model in \Cref{thm:2}, 
    \item {\em Understanding diffusion models via subspace clustering.} To the best of our knowledge, our work is the first to establish the equivalence between training diffusion models and subspace clustering under a tractable \MoLRG-inspired parameterization. This equivalence, together with the \MoLRG~model, allows us to show that the minimal number of samples for diffusion models to learn the underlying distribution scales linearly with the intrinsic dimension.  This finding stands in sharp contrast to existing results \cite{oko2023diffusion,wibisono2024optimal} in the literature, which show that diffusion models suffer from the curse of dimensionality when learning general classes of high-dimensional distributions.  Our results provide a more optimistic and practical perspective by demonstrating that diffusion models can effectively learn data distributions with intrinsic low-dimensional structures—a property commonly observed in image datasets, thereby avoiding the curse of dimensionality.

    \item \emph{Connections to the phase transition from memorization to generalization.}   \cite{kadkhodaie2023generalization,zhang2024emergence} have empirically revealed that diffusion models learn the score function across two distinct regimes--memorization (i.e., learning the empirical distribution of the training data)  and generalization (i.e., learning the underlying distribution of the data). Our work partially explains this intriguing experimental observation based on the \MoLRG~model in terms of generalization. We demonstrate that diffusion models learn the underlying data distribution, thereby enabling generalization, when the number of training samples scales linearly with the intrinsic dimension of the data distribution.\footnote{As discussed in \Cref{sec:prior-work}, we consider generalization in diffusion models as their ability to accurately capture the underlying data distribution.} Our theory reveals a phase transition from failure to success in learning the underlying distribution as the number of training samples increases, shedding light on the phase transition from memorization to generalization.  
    
    A recent work by \cite{merger2025generalization} also investigated the phase transition phenomenon by analyzing the gradient dynamics of linear models. Their findings are closely related to ours, highlighting how sample complexity governs the generalization behavior of diffusion models. However, there are key differences between our results and theirs. First, their analysis is limited to data drawn from a single Gaussian distribution, whereas our framework extends to a mixture of Gaussians. Second, their study focused on a linear neural network, whereas our model involves a mixture of two-layer neural networks, as defined in \eqref{eq:para MoG}.

    \item {\em Future directions based on our theory.} Several promising directions for future research are based on our theoretical framework. First, our current analysis assumes that the data lies on a union of mutually orthogonal subspaces. This facilitates theoretical tractability but does not fully capture the complexity of real-world data, which often reside on overlapping or nonlinear manifolds. Extending our framework to capture these complicated structures would be a meaningful extension of our results. Second, our analysis focuses on a simplified network parameterization for learning \MoLRG. In contrast, practical diffusion models typically rely on complex and over-parameterized architectures, such as U-Net and Transformers. A compelling direction for future research is to study the generalization behavior of diffusion models under over-parameterized nonlinear network architectures.
    % \qq{the discussion here should not use negative words. This is shooting ourself on our own feet. Instead we should point some interesting directions for future research, saying that "our work opened several direction for further investigation"} 
 
 \end{itemize} 

\subsection{Empirical Validation of Theoretical Findings}\label{subsec:thm verification} %\qq{we need to be consistent. either all capital for the title of subsection, or not?} 
Finally, we conclude this section by providing phase transition experiments for $K=1,2,3$, shown in \Cref{fig:phase-transition-MoG-k=1}, \Cref{fig:phase-transition-MoG-k=2}, and \Cref{fig:phase-transition-MoG-add-exp}.  Our experimental results show that training diffusion models consistently exhibits a phase transition from failure to success in learning the \MoLRG~distribution (or the subspaces) as the number of training samples increases, supporting our theoretical findings in Theorems \ref{thm:2} and \ref{thm:4}.  % These experiments empirically validate the equivalence between training diffusion models as defined in \eqref{eq:em loss} and the subspace clustering problem (or PCA when $K=1$).

Specifically, our experimental setup is as follows. For each plot, we fix the ambient dimension of the data to be $n = 48$ and vary the subspace dimension $d$ from $2$ to $8$ in increments of $1$. Similarly, we vary the number of training samples $N$ from $2$ to $15$ in steps of $1$. For each pair of $(n,d)$, we generate all training samples according to the \MoLRG~distribution in \eqref{eq:MoG noise} with $\bm e_i = 0$ for different $K=1,2,3$, independently repeating the experiment for $20$ times to empirically estimate the success probability of subspace recovery. 
For the case $K=1$, we apply SVD to solve the PCA problem in \eqref{eq:PCA}. To solve the subspace clustering problem in \eqref{eq:SC} when $K>1$, we apply the K-subspace method with spectral initialization as described in \cite{wang2022convergence}. 
To train the DAE with the theoretical parameterization \eqref{eq:para Gau} or \eqref{eq:para MoG}, we optimize the training loss \eqref{eq:em loss} via stochastic gradient descent (see \Cref{alg:1} for more details).

\section{Discussion on Related Results}\label{sec:prior-work} 

% \qq{I feel we can move this to Sec 1.2 with separated paragraph title}

In this section, we discuss the relationship between our results and closely related works on diffusion models and subspace clustering. 

\paragraph{Memorization and generalization in diffusion models.} Many interesting studies have been conducted to investigate memorization and generalization of diffusion models. \cite{zhang2024emergence,kadkhodaie2023generalization} demonstrated that diffusion models tend to memorize the training data in the memorization regime and generate new samples in the generalization regime.  \cite{gu2023memorization,zhang2024emergence} showed that diffusion models learn the empirical optimal score function in the memorization regime. \cite{yoon2023diffusion} argued that diffusion models tend to generalize when they fail to memorize. Recently, \cite{lyu2025resolving} showed that the memorization problem can be resolved by a simple inertia update step. 
In the generalization regime, a popular line of research \cite{chen2023sampling,chen2023improved,debortoli2022convergence,dupuis2025algorithm,lee2022convergence,lee2023convergence} has established error bounds on the distance between the true data distribution and the learned data distribution under different metrics, including KL divergence and Wasserstein distance.

\paragraph{Diffusion models for learning low-dimensional distributions.} Recently, a growing body of work has studied how diffusion models learn distributions with different low-dimensional structures. An important line of research focuses on data supported on low-dimensional subspaces. For example, a seminal work by \cite{chen2023score} theoretically studied score approximation, estimation, and distribution recovery of diffusion models for learning from data supported on a low-dimensional linear subspace, with general latent variables beyond Gaussian distributions. Recently, \cite{chen2025diffusion} proposed a diffusion factor model to exploit the low-dimensional structure in asset returns and established an error bound for score estimation. In contrast to these studies, our work focuses on a union of subspaces simultaneously instead of a single subspace. In addition, while the analysis in \cite{chen2023score,chen2025diffusion}  establishes a polynomial sample complexity bound in terms of the intrinsic dimension, our results yield a sharper bound that scales linearly with the intrinsic dimension under the \MoLRG~model.  In addition, \cite{chen2025interpolation} assumed that the data lies in a one-dimensional linear subspace and demonstrated that the generalization ability of diffusion models stems from a smoothing-induced interpolation effect. 

Another important line of research investigates more general low-dimensional manifold structures. For example, \cite{yakovlev2025generalization} studied generalization and approximation errors of the score matching estimator under a nonparametric Gaussian mixture. \cite{gottwald2025localized} studied locality structure, a form of low-dimensional structure characterized by sparse dependencies among components of the data distribution, to reduce sample complexity for training diffusion models. \cite{cui2025precise} studied the training dynamics of diffusion models when the underlying distribution is an infinite Gaussian mixture supported on a latent low-dimensional manifold. A promising direction inspired by these works is to extend the \MoLRG~model to mixtures of low-dimensional manifolds and analyze the training loss of diffusion models.  % They proposed a localized diffusion model and proved that it achieves lower sample complexity, thus circumventing the curse of dimensionality. 

% \qq{this is about sampling not learning, does not fit the paragraph title well, needs a separate title}  

\paragraph{Sampling rate of diffusion models with low-dimensional structures.} In a complementary direction, recent works leverage low-dimensional structures in data to improve the sampling convergence analysis in diffusion models. For example, \cite{huang2024denoising,li2024adapting,liang2025low} showed that the sampling rate of denoising diffusion probabilistic models scales with the intrinsic dimension of the data distribution. \cite{azangulov2024convergence,tang2024adaptivity,debortoli2022convergence} established sharp convergence rates for score-based diffusion models when the data distribution lies on or near low-dimensional manifolds.  

\paragraph{Subspace clustering.} Subspace clustering is a fundamental problem in unsupervised learning, which aims to identify and group data points that lie in a union of low-dimensional subspaces in a high-dimensional space \cite{agarwal2004k,vidal2011subspace,lerman2018overview}. Over the past years, a substantial body of literature has explored various approaches to the algorithmic development and theoretical analysis of subspace clustering. These include techniques such as sparse representation \cite{elhamifar2013sparse,wang2013noisy,soltanolkotabi2012geometric,soltanolkotabi2014robust}, low-rank representation \cite{wang2022convergence,liu2013efficient,liu2017subspace,maunu2019well,lerman2025global}, and spectral clustering  \cite{pmlr-v139-li21f,vidal2005generalized}. In this work, we present a new interpretation of diffusion models from the perspective of subspace clustering. This is the first time that diffusion models have been analyzed through this lens, offering new insights into how these models can effectively learn complex data distributions by leveraging the intrinsic low-dimensional subspaces within the data.

\section{Practical Implications of Our Theoretical Results} \label{sec:experiments} 

Building on the results in \Cref{sec:results}, we study the practical value of our theoretical investigation by showing that: (i) our study of low-dimensional distribution learning offers key insights into the generalization behavior of real-world diffusion models (see \Cref{sec:appli_1}), and (ii) the basis vectors of the identified low-dimensional subspaces correspond to different semantic task vectors in practice, enabling controlled editing of specific attributes in content generation (see \Cref{sec:appli_2}).

\subsection{Phase Transition of Generalization in Real-World Diffusion Models}\label{sec:appli_1}
In this subsection, we conduct experiments on both synthetic \MoLRG~data and real image datasets to train U-Net-based diffusion models. Consistent with the predictions of \Cref{thm:2} and \Cref{thm:4}, we observe a similar phase transition in generalization from failure to success, on both synthetic datasets and real image datasets.  %As discussed in the end of \Cref{subsec:cont}, achieving good generalization is closely tied to accurately learning the underlying distribution in diffusion models. 
More specifically, we empirically show that the minimum number of training samples, denoted by $N_{\mathrm{min}}$, required for  generalization scales linearly with the intrinsic dimension, denoted by $\mathrm{ID}$, on both synthetic and real datasets, 
\begin{equation} \label{eq:linear_relation}
    N_{\mathrm{min}} \approx c \cdot \mathrm{ID}, 
\end{equation}
where $c > 0$ is a constant. As discussed at the end of \Cref{subsec:cont}, achieving good generalization is closely tied to accurately learning the underlying distribution in diffusion models. Therefore, our theoretical framework not only explains distribution learning in the \MoLRG~model but also offers valuable insights into the generalization of real-world diffusion models. Now, we introduce the experimental setup. 

\paragraph{Measuring generalization in diffusion models.} Recent studies \cite{zhang2024emergence,kadkhodaie2023generalization} have shown that diffusion models trained under different settings can reproduce each other’s outputs. This reproducibility provides strong evidence of generalization \cite{kadkhodaie2023generalization}. Furthermore, \cite{zhang2024emergence} empirically demonstrates that this phenomenon co-emerges with the models’ ability to generate novel samples distinct from their training data. Together, these findings suggest that the ability of diffusion models to generate new samples can serve as an indicator of good generalization. Let $\{\bm y^{(j)}\}_{j=1}^M$ denote $M$ samples generated by a diffusion model trained on the dataset $\{\bm x^{(i)}\}_{i=1}^N$. Then, we adopt a variant of the generalization (GL) score proposed in \cite{zhang2024emergence}, defined as follows:
\begin{align}\label{def:gl_score}
    \text{GL} := \frac{1}{M}\sum_{j=1}^M \mathbb{I}\left(\min_{i\in[N]} \left\|\bm \Psi \left(\bm x^{(i)}\right) -  \bm \Psi \left(\bm y^{(j)}\right)\right\| \geq \delta\right). 
\end{align}
Here, $\delta$ is a pre-defined threshold, $\bm \Psi(\bm x)$ denotes a descriptor function applied to $\bm x$, and $\mathbb{I}(\cdot)$ is the indicator function, where $\mathbb{I}(x \ge \delta) = 1$ if $x \ge \delta$ and $0$ otherwise. 
% Following \cite{zhang2023emergence}, we consider a diffusion model to successfully generalize if $\text{GL} \geq 0.95$. 
For the \MoLRG\ distribution, we set $\bm \Psi(\bm x)$ as the identity function and $\delta$ is defined in \eqref{eq:def_delta} in Appendix \ref{app:exp_setting_MoLRG_unet}. For real-world datasets, we set $\bm \Psi(\bm x)$ as the self-supervised copy detection descriptor introduced in \cite{pizzi2022self}, a neural feature extractor tailored for copy detection tasks, and set $\delta = 0.8$ according to \cite{pizzi2022self, somepalli2023understanding}. Additional details are provided in Appendix \ref{app sec:expset_appli_1}.

\begin{figure*}[t]
% \captionsetup[subfigure]{justification=centering}
\begin{center}
    \includegraphics[width = .9\linewidth]{figure/Phase-transition-real-MoG-dataset.jpg} 
    \caption{\textbf{Phase transition of generalization using U-Net.} Diffusion models with a U-Net architecture are trained on synthetic data sampled from the \MoLRG~distribution (left column; $K = 2$, $n = 48$, varying intrinsic dimensions) and on real image datasets: CIFAR-10, CelebA, FFHQ, and AFHQ (right column). The GL score is plotted against the ratio of training samples to the intrinsic dimension (top row) and to the square of the intrinsic dimension (bottom row). A black dashed line fits the data across different intrinsic dimensions (datasets) for each figure. A GL score above 0.95 (within the dark grey region) indicates good generalization, while a score below 0.95 ( within the light grey region) indicates poor generalization.
    % \qq{the title needs to be consistent, either all capitalized, or only the first word. Currently, "MoLRG \textbf{d}istribution" and "Real Datasets" are not consistent. Please be careful on all the figures} 
    } \label{fig:phase-transition-UNet}
\end{center}
\end{figure*} 

Intuitively, the GL score measures the dissimilarity between the generated samples $\{\bm y^{(j)}\}_{j=1}^M$ and the training samples $\{\bm x^{(i)} \}_{i=1}^N$ in the feature space. A higher GL score indicates that the generated samples are less similar to the training data, reflecting better generalization. In this work, we consider a diffusion model to generalize well when $\text{GL} > 0.95$, i.e.,  at least 95\% of the generated samples are distinct from the training set.

\paragraph{Experiments on synthetic data.} First, we demonstrate a phase transition in generalization when training U-Net on synthetic data generated from the \MoLRG~distribution. We use the \MoLRG\ distribution defined in \eqref{eq:MoG noise} and set the data dimension $n=48$, the number of components $K=2$, the noise level $\bm e_i = \bm 0$, and mixing proportion $\pi_k=1/2$. Then, we set each cluster to contain an equal number of samples and the total number is $N$. The intrinsic dimension of each subspace is set to $d$, identical across clusters. Because the subspace bases are orthogonal, the total intrinsic dimension of the distribution $\mathrm{ID} = Kd$. We optimize the training loss in \eqref{eq:em loss} using a DAE $\bm x_{\bm \theta}(\cdot,t)$, parameterized by U-Net. The same U-Net architecture is used across all experiments. The detailed experimental settings are provided in Appendix \ref{app:exp_setting_MoLRG_unet}.

In \Cref{fig:phase-transition-UNet} (top-left), we plot the GL scores against the ratio $\log_2\left(N/\mathrm{ID}\right)$ by varying both $N$ and $\mathrm{ID}$. Here, different scatter colors correspond to different choices of $\mathrm{ID} = 8, 10, 12$. The GL scores plotted against $\log_2\left(N/\mathrm{ID}\right)$ consistently exhibit a sigmoid-shaped curve across different $\mathrm{ID}$ values. This suggests that, for a fixed model architecture, the generalization ability depends primarily on the ratio $N/\mathrm{ID}$ rather than on the values of $N$ or $\mathrm{ID}$ individually. Specifically, we fit all points using a sigmoid function (the black dashed curve shown in \Cref{fig:phase-transition-UNet} top-left), denoted by $f_{\MoLRG} \left(N/\mathrm{ID}\right)$, with details provided in Appendix \ref{app:exp_setting_MoLRG_unet}. For comparison, we plot the GL scores against $\log_2\left(N/\mathrm{ID}^2\right)$ in \Cref{fig:phase-transition-UNet} (bottom-left) and fit them with a sigmoid function. The data points deviate more from the fitted curve than the curve in the top plot. The stronger alignment between the data points and the fitted curve in the top plot confirms that $N/\mathrm{ID}$ is a better indicator for GL score.

Recall that $\text{GL} > 0.95$ indicates successful generalization. To identify $N_{\mathrm{min}}$, we solve $\text{GL}({N_{\mathrm{min}}}/\mathrm{ID}) \approx f_{\MoLRG} \left(N_{\mathrm{min}}/\mathrm{ID}\right)= 0.95$, which implies $c = f_{\MoLRG}^{-1}(0.95)$ in \eqref{eq:linear_relation}.\footnote{It is worth noting that the linear relationship between $N_{\mathrm{min}}$ and $\mathrm{ID}$ holds as long as the data points (plotting GL score against $N/\mathrm{ID}$) can be well-fitted by a function. Changing the threshold for successful generalization affects only the slope $c$ of the linear relationship.} 
% This yields a linear relationship between $N_{\mathrm{min}}$ and $\mathrm{ID}$, as expressed in \eqref{eq:linear_relation}, where $c = f_{\MoLRG}^{-1}(0.95)$. 
This demonstrates that achieving successful generalization requires $N_{\mathrm{min}}$ to scale linearly with $\mathrm{ID}$, thereby corroborating our theoretical findings in \Cref{thm:4}.

\begin{table}[t]
    \centering
    \begin{tabular}{lcccc}
        \hline
         & CIFAR-10 & CelebA & FFHQ & AFHQ \\
        \hline
        $\mathrm{ID}$  & 10.8 & 11.5 & 15.8 & 16.7 \\
        \hline
    \end{tabular}\smallskip 
    \caption{\textbf{Intrinsic dimensions $\mathrm{ID}$ for different real world datasets.}}
    \label{tab:id}
    \vspace{-0.1in}
\end{table}

\paragraph{Experiments on real image datasets.} Next, our results in \Cref{fig:phase-transition-UNet} (top-right) reveal a similar phase transition in generalization across several real-world image datasets, including CIFAR-10 \cite{cifar10}, CelebA \cite{liu2015faceattributes}, FFHQ \cite{ffhq}, and AFHQ \cite{choi2020stargan}. Following a similar experimental setup for \MoLRG, we use the same U-Net architecture for different datasets with extra experimental details provided in Appendix \ref{app:exp_setting_real_unet}. Then, we respectively plot the GL score against $\log_2(N/\mathrm{ID})$ and $\log_2(N/\mathrm{ID}^2)$ by varying the number of training samples $N$. However, because the intrinsic dimension of image datasets here is not known, we estimate it using the method described in Appendix \ref{app:exp_setting_real_dataset_rank}, with the resulting estimates summarized in \Cref{tab:id}.

Our results across different real image distributions also suggest that the GL score is primarily determined by the ratio $N/\mathrm{ID}$.
As shown in \Cref{fig:phase-transition-UNet} (top-right), the plot of the GL score against $\log_2\left(
N/\mathrm{ID}\right)$ yields nearly identical sigmoid-shaped curves across different datasets. This behavior is consistent with our observations for the \MoLRG~distribution. In contrast, \Cref{fig:phase-transition-UNet} (bottom-right) shows the GL scores plotted against $N/\mathrm{ID^2}$, where data points cannot be captured by a single function across datasets. This comparison further supports that $N/\mathrm{ID}$ is a more appropriate indicator for GL score.

Accordingly, we fit all the points of GL scores using the same function
$f_{\texttt{real}}\left(N/\mathrm{ID}\right)$ (the black dashed curve shown in \Cref{fig:phase-transition-UNet} top-right), with more details provided in Appendix \ref{app:exp_setting_real_unet}. In this case, training diffusion models on real image datasets requires at least $N_{\mathrm{min}} =  f^{-1}_{\texttt{real}}\left(0.95\right) \mathrm{ID}$ number of training samples to achieve good generalization. This aligns with the linear relationship in \eqref{eq:linear_relation} with constant $c = f^{-1}_{\texttt{real}}\left(0.95\right)$.

\subsection{Correspondence between Basis Vectors of Low-Dimensional Subspaces and Semantic Attributes}\label{sec:appli_2} 
% \qq{change meanings to task vectors, and the following paragraph needs to be updated}
In this subsection, we demonstrate that our theoretical insights provide valuable guidance for improving the controllability of image generation.
% Since real-world image data distributions can be approximated by the \MoLRG~distribution, 
We begin by outlining a method for identifying low-rank subspaces in diffusion models trained on real-world image datasets. Next, we demonstrate how to verify that the orthogonal basis vectors of the identified subspaces are semantic task vectors. As illustrated in \Cref{fig:meta_1}, these vectors can be leveraged to steer diffusion models to edit image attributes, such as gender, hairstyle, and color. While previous studies \cite{manor2023posterior,manor2024zero,tinaz2025emergence} have explored similar methods for image editing, our study offers a new perspective for understanding the method through the lens of a low-dimensional subspace. Building on this study, our concurrent work \cite{Chen2024} proposes a training-free method that enables controllable image editing.  

\paragraph{Identifying low-rank subspaces in diffusion models.}
Although the DAE $\bm x_{\bm \theta}(\cdot,t)$ of real-world diffusion models cannot be exactly written as the theoretical parameterization introduced in \eqref{eq:para MoG}, it is still possible to locally identify a low-rank subspace. 
Recent studies \cite{li2024understanding, Chen2024} have shown that $\bm x_{\bm \theta}(\cdot,t)$ can be well approximated using a first-order Taylor expansion: 
\begin{align}
    \bm x_{\bm \theta}(\bm x_t + \bm \delta_t,t) \approx \bm x_{\bm \theta}(\bm x_t,t) + \bm J_t \bm \delta_t,
\end{align}
where $\bm \delta_t$ is the steering direction and $\bm J_t = \nabla_{\bm x_t} \bm x_{\bm \theta}(\bm x_t,t) $ denotes the Jacobian of the DAE $\bm x_{\bm \theta}(\cdot,t)$ at $\bm x_t$.
As we empirically verify in Appendix \ref{app:exp_setting_real_dataset_rank}, $\bm J_t$ is often a low-rank matrix at certain timesteps $t$, indicating that its range spans a low-dimensional subspace around $\bm x_t$. To identify an orthonormal basis for this subspace, we apply an SVD to $\bm J_t$ and obtain $\bm J_t = \bm P \bm \Sigma \bm Q^T$, %Therefore, if we treat $\bm x_{\bm \theta}(\bm x_t,t)$ as a constant with respect to $\bm \delta_t$, 
%Here, an eigenvalue decomposition of the Jacobian $\bm J_t = \bm Q \bm \Sigma \bm Q^T$ reveals a structural resemblance to the parameterization in \eqref{eq:para MoG}, albeit in an asymmetric form. 
where $r:=\rank(\bm J_t)$,  $\bm P = [\bm p_{1}, \cdots, \bm p_{r} ] \in \mathcal{O}^{n \times r}$, $\bm Q = [\bm q_{1}, \cdots, \bm q_{r} ] \in \mathcal{O}^{n \times r}$, and $\bm \Sigma = \mathrm{diag}(\sigma_{1},\dots,\sigma_{r})$ with $\sigma_{1}\ge\dots\ge \sigma_{r}\ge 0$. Each $\bm p_i$ serves as an orthogonal basis vector for the subspace.

\begin{figure*}[t]
\begin{center}
    \begin{subfigure}{0.45\textwidth}
        \includegraphics[width=1\textwidth]{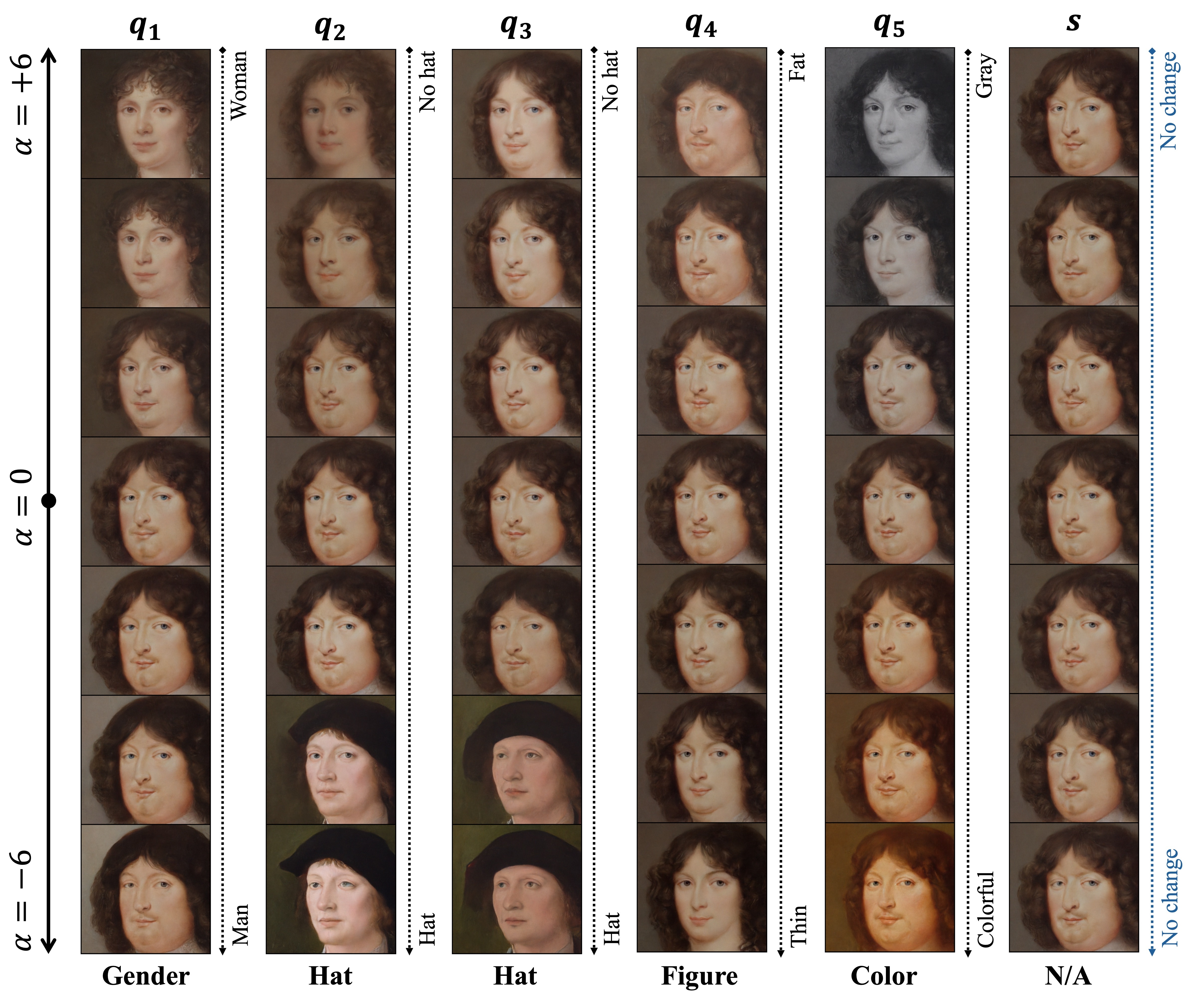}
    \caption{}
    \end{subfigure}  
    \begin{subfigure}{0.45\textwidth}
        \includegraphics[width=1\textwidth]{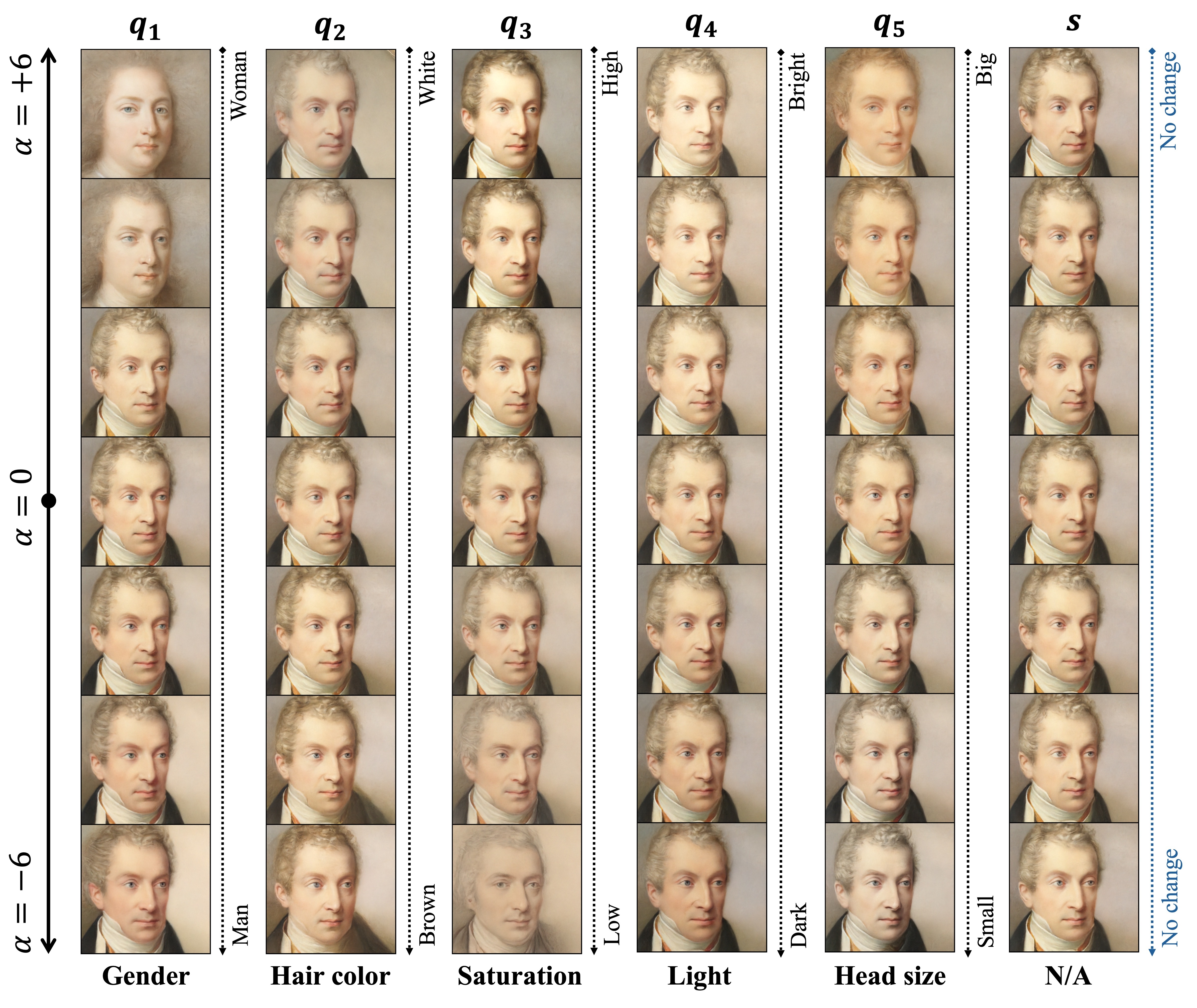} 
    \caption{}
    \end{subfigure}  
    \caption{\textbf{Correspondence between the singular vectors of the Jacobian of the DAE and semantic image attributes.} We use a pre-trained DDPM with U-Net on the MetFaces dataset \cite{metafaces}. We edit the original image $\bm x_0$ by changing $\bm x_t$ into $\bm x_t + \alpha \bm q_i$, where $\bm q_i$ is a singular vector of the Jacobian of the DAE $\bm x_{\bm \theta}(\bm x_t,t)$. 
    %(a, b) Some examples when $t = 0.7T$. (b) Ablation studies when $t = 0.1T$ and $0.9T$.    
    } 
    \label{fig:meta_1}
\end{center}
\end{figure*}

As such, if we choose $\bm \delta_t = \alpha \bm q_i $ to be one of the singular vectors $i \in [r]$, then we obtain  
\begin{align*}
    \bm x_{\bm \theta}(\bm x_t + \alpha \bm q_i,t) &\approx \bm x_{\bm \theta}(\bm x_t,t) + \sum_{j=1}^r \sigma_j \bm p_j \langle \bm q_j, \alpha \bm q_i \rangle = \bm x_{\bm \theta}(\bm x_t,t) + \alpha \sigma_i \bm p_i.
\end{align*}
Given that $\bm x_{\bm \theta}(\bm x_t,t) \approx \mathbb{E}[ \bm x_0 \mid \bm x_t ]$ serves as an estimate of the clean image, applying a perturbation $\bm \delta = \alpha \bm q_i$ modifies the original image $\bm x_0$ along the direction of the corresponding orthogonal basis vector $\bm p_i$ with a strength of $\alpha \sigma_i$. In the following, we empirically show that each $\bm p_i$ is often associated with semantic attributes.

\paragraph{Experimental implementation and results.} We use a pre-trained diffusion denoising probabilistic model (DDPM) \cite{ho2020denoising} on the MetFaces dataset \cite{metafaces}. We randomly select an image $\bm x_0$ from the dataset and use DDIM citep{song2020denoising} inversion to obtain the corresponding latent $\bm x_t$ at $t=0.7$. We choose the steering direction as the leading right singular vectors $\bm q_{1}, \ldots, \bm q_5$ and use $\tilde{\bm x}_t = \bm x_t + \alpha \bm q_i$ to generate new images with editing strength $\alpha \in [-6, 6]$. \Cref{fig:meta_1} shows that these singular vectors enable different semantic edits in terms of gender, hairstyle, and color of the image. For comparison, steering the image along a direction $\bm s$ drawn uniformly at random from the unit sphere results in almost \emph{no} perceptible change in the edited images. This implies that the low-dimensional subspace spanned by $\bm P$ is nontrivial, with its leading basis vectors corresponding to semantic task vectors. The experimental results for more images and ablation studies for $t=0.1$ and $0.9$ are shown in \Cref{fig:meta_more}.

\section{Conclusion \& Future Directions}\label{sec:conclusion}

In this work, we studied the training loss of diffusion models to investigate when and why they can learn the underlying distribution without suffering from the curse of dimensionality. Assuming that the data follow a \MoLRG~distribution—an assumption supported by extensive empirical evidence—we showed that, under an appropriate network parameterization, minimizing the training loss of diffusion models is equivalent to solving a subspace clustering problem. Based on this equivalence, we further showed that the optimal solutions to the training loss can recover the underlying subspaces when the minimum number of samples scales linearly with the intrinsic dimensionality of the data distribution. Moreover, we established a correspondence between the subspace basis and the semantic attributes of image data. 

Our work opens several new directions for advancing the theoretical understanding of diffusion models. First, as noted in the remarks of \Cref{thm:4}, while our work explains the generalization ability of diffusion models, it does not fully address the phenomenon of memorization or the phase transition from memorization to generalization. Future work should extend the current analysis to consider over-parameterized models and explore how these models contribute to memorization and generalization. Second, our study focuses on leveraging low-dimensional structures to understand the training process of diffusion models. However, the sampling process is also critical in diffusion models, as it influences the efficiency of generated samples. Third, our current analysis relies on simplifying assumptions, including the orthogonality and symmetry of subspaces as well as the hard-max parameterization; see Assumptions \ref{AS:2} and \ref{AS:3}. An important direction for future work is to relax these assumptions by considering more general subspace configurations (e.g., non-orthogonal or heterogeneous subspaces) and softmax parameterizations, and to understand how these more realistic settings affect the optimization landscape and generalization behavior of diffusion models. As discussed in \Cref{sec:prior-work}, many studies have exploited low-dimensional structures to improve the sampling rate. A key direction for future research is to analyze the sampling behavior of diffusion models in the \MoLRG~model.   

\acks{P. Wang is supported in part by the University of Macau under grants SRG2025-00043-FST and UMDF-TISF-I/2026/013/FST, and in part by the Macau Science and Technology Development Fund (FDCT) 0091/2025/ITP2. H. Zhang, Z. Zhang, S. Chen, and Q. Qu acknowledge support from NSF CAREER CCF-2143904, NSF CCF-2212066, NSF CCF-2212326, NSF IIS 2312842, NSF IIS 2402950, a gift grant from KLA, the MICDE Catalyst Grant, and a Google TPU Research Award. Y. Ma acknowledges support from the joint Simons Foundation-NSF DMS Grant 2031899, NSF IIS 2402951, and the ONR Grant N00014-22-1-2102. Y. Ma also acknowledges support from the startup fund from the University of Hong Kong and the JC STEM Lab fund by The Hong Kong Jockey Club Charities Trust.

The authors would also like to thank Laura Balzano (U. Michigan), Jeff Fessler (U. Michigan), Huikang Liu (SJTU), Dogyoon Song (UC Davis), Liyue Shen (U. Michigan), Rene Vidal (UPenn), and Zhihui Zhu (OSU) for stimulating discussions.  
}

% \qq{please double check all the references, and cite the journal/conf version, and update arxiv citation to published versions. For example, Binxu Wang and John J Vastola 2023 paper is published at TMLR, which we should cite that one}

% \acks{We would like to acknowledge support for this project
% from the National Science Foundation (NSF grant IIS-9988642)
% and the Multidisciplinary Research Program of the Department
% of Defense (MURI N00014-00-1-0637). }


% \newpage

% \appendix
% \section*{Appendix A.}
% \label{app:theorem}

% Note: in this sample, the section number is hard-coded in. Following
% proper LaTeX conventions, it should properly be coded as a reference:

%In this appendix we prove the following theorem from
%Section~\ref{sec:textree-generalization}:

% {\small 
% \bibliographystyle{unsrt}
% \bibliography{biblio/diffusion,biblio/mixture_model}
% }
\printbibliography

@inproceedings{kadkhodaie2023generalization,
  author    = {Kadkhodaie, Zahra and Guth, Florentin and Simoncelli, Eero P. and Mallat, Stephane},
  title     = {Generalization in diffusion models arises from geometry-adaptive harmonic representations},
  booktitle = {International Conference on Learning Representations},
  year      = {2024}
}

@inproceedings{han2024neural,
  title={Neural network-based score estimation in diffusion models: Optimization and generalization},
  author={Han, Yinbin and Razaviyayn, Meisam and Xu, Renyuan},
  booktitle={International Conference on Learning Representations},
  year={2024}
}

@inproceedings{bentonnearly,
  title={Nearly $d$-Linear Convergence Bounds for Diffusion Models via Stochastic Localization},
  author={Benton, Joe and De Bortoli, Valentin and Doucet, Arnaud and Deligiannidis, George},
  booktitle={International Conference on Learning Representations},
  year = {2024}
}

@article{li2025convergence,
  title={A Convergence Theory for Diffusion Language Models: An Information-Theoretic Perspective},
  author={Li, Gen and Cai, Changxiao},
  journal={arXiv preprint arXiv:2505.21400},
  year={2025}
}

@article{wang2024the,
title={The Unreasonable Effectiveness of Gaussian Score Approximation for Diffusion Models and its Applications},
author={Binxu Wang and John Vastola},
journal={Transactions on Machine Learning Research},
issn={2835-8856},
year={2024},
}

@inproceedings{wibisono2024optimal,
  title={Optimal score estimation via empirical bayes smoothing},
  author={Wibisono, Andre and Wu, Yihong and Yang, Kaylee Yingxi},
  booktitle={The Thirty Seventh Annual Conference on Learning Theory},
  pages={4958--4991},
  year={2024},
  organization={PMLR}
}

@inproceedings{wang2023dr2,
  title={Dr2: Diffusion-based robust degradation remover for blind face restoration},
  author={Wang, Zhixin and Zhang, Ziying and Zhang, Xiaoyun and Zheng, Huangjie and Zhou, Mingyuan and Zhang, Ya and Wang, Yanfeng},
  booktitle={Proceedings of the IEEE/CVF Conference on Computer Vision and Pattern Recognition},
  pages={1704--1713},
  year={2023}
}

@article{fabian2024adapt,
  title={Adapt and diffuse: Sample-adaptive reconstruction via latent diffusion models},
  author={Fabian, Zalan and Tinaz, Berk and Soltanolkotabi, Mahdi},
  journal={Proceedings of Machine Learning Research},
  volume={235},
  pages={12723},
  year={2024}
}

@article{xing2024survey,
  title={A survey on video diffusion models},
  author={Xing, Zhen and Feng, Qijun and Chen, Haoran and Dai, Qi and Hu, Han and Xu, Hang and Wu, Zuxuan and Jiang, Yu-Gang},
  journal={ACM Computing Surveys},
  volume={57},
  number={2},
  pages={1--42},
  year={2024},
  publisher={ACM New York, NY}
}

@Article{hornik1989multilayer,
  title={Multilayer feedforward networks are universal approximators},
  author={Hornik, Kurt and Stinchcombe, Maxwell and White, Halbert},
  journal={Neural Networks},
  volume={2},
  number={5},
  pages={359--366},
  year={1989},
  publisher={Elsevier}
}

@inproceedings{arora2017generalization,
  title={Generalization and equilibrium in generative adversarial nets ({GAN}s)},
  author={Arora, Sanjeev and Ge, Rong and Liang, Yingyu and Ma, Tengyu and Zhang, Yi},
  booktitle={International Conference on Machine Learning},
  pages={224--232},
  year={2017},
  organization={PMLR}
}

@Article{wang2023hidden,
  author  = {Wang, Binxu and Vastola, John J},
  journal = {arXiv preprint arXiv:2311.10892},
  title   = {The Hidden Linear Structure in Score-Based Models and its Application},
  year    = {2023},
}

@InProceedings{ho2020denoising,
  title={Denoising diffusion probabilistic models},
  author={Ho, Jonathan and Jain, Ajay and Abbeel, Pieter},
  booktitle={Advances in Neural Information Processing Systems},
  volume={33},
  pages={6840--6851},
  year={2020}
}

@article{maunu2019well,
  title={A well-tempered landscape for non-convex robust subspace recovery},
  author={Maunu, Tyler and Zhang, Teng and Lerman, Gilad},
  journal={Journal of Machine Learning Research},
  volume={20},
  number={37},
  pages={1--59},
  year={2019}
}

@misc{kingma2013auto,
  title={Auto-encoding variational bayes},
  author={Kingma, Diederik P and Welling, Max and others},
  year={2013},
  publisher={Banff, Canada}
}

@inproceedings{bar2024lumiere,
  title={Lumiere: A space-time diffusion model for video generation},
  author={Bar-Tal, Omer and Chefer, Hila and Tov, Omer and Herrmann, Charles and Paiss, Roni and Zada, Shiran and Ephrat, Ariel and Hur, Junhwa and Liu, Guanghui and Raj, Amit and others},
  booktitle={SIGGRAPH Asia 2024 Conference Papers},
  pages={1--11},
  year={2024}
}

@InProceedings{chen2023score,
  author       = {Chen, Minshuo and Huang, Kaixuan and Zhao, Tuo and Wang, Mengdi},
  booktitle    = {International Conference on Machine Learning},
  title        = {Score approximation, estimation and distribution recovery of diffusion models on low-dimensional data},
  year         = {2023},
  organization = {PMLR},
  pages        = {4672--4712},
}

@InProceedings{Chen2024,
  author    = {Chen, Siyi and Zhang, Huijie and Guo, Minzhe and Lu, Yifu and Wang, Peng and Qu, Qing},
  booktitle = {Advances in Neural Information Processing Systems},
  title     = {Exploring Low-Dimensional Subspace in Diffusion Models for Controllable Image Editing},
  year      = {2024},
  pages     = {27340--27371},
  volume    = {37},
}

@article{efron2011tweedie,
  title={Tweedie’s formula and selection bias},
  author={Efron, Bradley},
  journal={Journal of the American Statistical Association},
  volume={106},
  number={496},
  pages={1602--1614},
  year={2011},
  publisher={Taylor \& Francis}
}

@Article{li2024generalization,
  author  = {Li, Puheng and Li, Zhong and Zhang, Huishuai and Bian, Jiang},
  journal = {Advances in Neural Information Processing Systems},
  title   = {On the generalization properties of diffusion models},
  year    = {2024},
  volume  = {36},
  pages={2097--2127},
}

@inproceedings{song2021scorebased,
title={Score-Based Generative Modeling through Stochastic Differential Equations},
author={Yang Song and Jascha Sohl-Dickstein and Diederik P Kingma and Abhishek Kumar and Stefano Ermon and Ben Poole},
booktitle={International Conference on Learning Representations},
year={2021}
}

@article{turk1991eigenfaces,
  title={Eigenfaces for recognition},
  author={Turk, Matthew and Pentland, Alex},
  journal={Journal of Cognitive Neuroscience},
  volume={3},
  number={1},
  pages={71--86},
  year={1991},
  publisher={MIT Press One Rogers Street, Cambridge, MA 02142-1209, USA journals-info~…}
}

@InProceedings{oko2023diffusion,
  author       = {Oko, Kazusato and Akiyama, Shunta and Suzuki, Taiji},
  booktitle    = {International Conference on Machine Learning},
  title        = {Diffusion models are minimax optimal distribution estimators},
  year         = {2023},
  organization = {PMLR},
  pages        = {26517--26582},
}

@InProceedings{shah2023learning,
  author  = {Shah, Kulin and Chen, Sitan and Klivans, Adam},
  booktitle = {Advances in Neural Information Processing Systems},
  title   = {Learning mixtures of Gaussians using the {DDPM} objective},
  year    = {2023},
  pages   = {19636--19649},
  volume  = {36},
}

@InProceedings{karras2022elucidating,
  author  = {Karras, Tero and Aittala, Miika and Aila, Timo and Laine, Samuli},
  booktitle = {Advances in Neural Information Processing Systems},
  title   = {Elucidating the design space of diffusion-based generative models},
  year    = {2022},
  pages   = {26565--26577},
  volume  = {35},
}

@Article{chen2024deconstructing,
  author  = {Chen, Xinlei and Liu, Zhuang and Xie, Saining and He, Kaiming},
  journal = {arXiv preprint arXiv:2401.14404},
  title   = {Deconstructing Denoising Diffusion Models for Self-Supervised Learning},
  year    = {2024},
}

@InProceedings{zhang2024emergence,
  title = 	 {The Emergence of Reproducibility and Consistency in Diffusion Models},
  author =       {Zhang, Huijie and Zhou, Jinfan and Lu, Yifu and Guo, Minzhe and Wang, Peng and Shen, Liyue and Qu, Qing},
  booktitle = 	 {International Conference on Machine Learning},
  pages = 	 {60558--60590},
  year = 	 {2024},
  volume = 	 {235},
  publisher =    {PMLR},
}

@article{loaiza-ganem2024deep,
title={Deep Generative Models through the Lens of the Manifold Hypothesis: A Survey and New Connections},
author={Gabriel Loaiza-Ganem and Brendan Leigh Ross and Rasa Hosseinzadeh and Anthony L. Caterini and Jesse C. Cresswell},
journal={Transactions on Machine Learning Research},
issn={2835-8856},
year={2024},
}

@Article{gu2023memorization,
  author  = {Gu, Xiangming and Du, Chao and Pang, Tianyu and Li, Chongxuan and Lin, Min and Wang, Ye},
  journal = {arXiv preprint arXiv:2310.02664},
  title   = {On memorization in diffusion models},
  year    = {2023},
}

@InProceedings{cole2024score,
  author    = {Cole, Frank and Lu, Yulong},
  booktitle = {International Conference on Learning Representations},
  title     = {Score-based generative models break the curse of dimensionality in learning a family of sub-Gaussian distributions},
  year      = {2024},
}

@article{lyu2025resolving,
  title={Resolving memorization in empirical diffusion model for manifold data in high-dimensional spaces},
  author={Lyu, Yang and Qian, Yuchun and Nguyen, Tan Minh and Tong, Xin T},
  journal={arXiv preprint arXiv:2505.02508},
  year={2025}
}

@InProceedings{yoon2023diffusion,
  author    = {Yoon, TaeHo and Choi, Joo Young and Kwon, Sehyun and Ryu, Ernest K},
  booktitle = {ICML 2023 Workshop on Structured Probabilistic Inference $\&$ Generative Modeling},
  title     = {Diffusion probabilistic models generalize when they fail to memorize},
  year      = {2023},
}

@InProceedings{wu2024theoretical,
  author    = {Wu, Yuchen and Chen, Minshuo and Li, Zihao and Wang, Mengdi and Wei, Yuting},
  booktitle = {International Conference on Machine Learning},
  title     = {Theoretical insights for diffusion guidance: A case study for Gaussian mixture models},
  pages={53291--53327},
  year={2024},
  organization={PMLR}
}

@Article{vincent2011connection,
  author    = {Vincent, Pascal},
  journal   = {Neural Computation},
  title     = {A connection between score matching and denoising autoencoders},
  year      = {2011},
  number    = {7},
  pages     = {1661--1674},
  volume    = {23},
  publisher = {MIT Press},
}

@Article{merger2025generalization,
  author  = {Merger, Claudia and Goldt, Sebastian},
  journal = {arXiv preprint arXiv:2505.24769},
  title   = {Generalization Dynamics of Linear Diffusion Models},
  year    = {2025},
}

@article{tinaz2025emergence,
  title={Emergence and Evolution of Interpretable Concepts in Diffusion Models},
  author={Tinaz, Berk and Fabian, Zalan and Soltanolkotabi, Mahdi},
  journal={arXiv preprint arXiv:2504.15473},
  year={2025}
}

@article{wang2023patch,
  title={Patch diffusion: Faster and more data-efficient training of diffusion models},
  author={Wang, Zhendong and Jiang, Yifan and Zheng, Huangjie and Wang, Peihao and He, Pengcheng and Wang, Zhangyang and Chen, Weizhu and Zhou, Mingyuan and others},
  journal={Advances in Neural Information Processing Systems},
  volume={36},
  pages={72137--72154},
  year={2023}
}

@InProceedings{pope2020intrinsic,
  author    = {Pope, Phil and Zhu, Chen and Abdelkader, Ahmed and Goldblum, Micah and Goldstein, Tom},
  booktitle = {International Conference on Learning Representations},
  title     = {The Intrinsic Dimension of Images and Its Impact on Learning},
  year      = {2020},
}

@InProceedings{song2020denoising,
  author    = {Song, Jiaming and Meng, Chenlin and Ermon, Stefano},
  booktitle = {International Conference on Learning Representations},
  title     = {Denoising Diffusion Implicit Models},
  year      = {2020},
}

@InProceedings{song2020score,
  author  = {Song, Yang and Sohl-Dickstein, Jascha and Kingma, Diederik P and Kumar, Abhishek and Ermon, Stefano and Poole, Ben},
  booktitle = {International Conference on Learning Representations},
  title   = {Score-based generative modeling through stochastic differential equations},
  year    = {2021},
}

@Article{rudelson2009smallest,
  author    = {Rudelson, Mark and Vershynin, Roman},
  journal   = {Communications on Pure and Applied Mathematics: A Journal Issued by the Courant Institute of Mathematical Sciences},
  title     = {Smallest singular value of a random rectangular matrix},
  year      = {2009},
  number    = {12},
  pages     = {1707--1739},
  volume    = {62},
  publisher = {Wiley Online Library},
}

@InProceedings{chen2023sampling,
  author    = {Chen, Sitan and Chewi, Sinho and Li, Jerry and Li, Yuanzhi and Salim, Adil and Zhang, Anru R},
  booktitle = {International Conference on Learning Representations},
  title     = {Sampling is as easy as learning the score: theory for diffusion models with minimal data assumptions},
  year      = {2023},
}

@Article{gatmiry2024learning,
  author  = {Gatmiry, Khashayar and Kelner, Jonathan and Lee, Holden},
  journal = {arXiv preprint arXiv:2404.18869},
  title   = {Learning Mixtures of Gaussians Using Diffusion Models},
  year    = {2024},
}

@Article{chen2024learning,
  author  = {Chen, Sitan and Kontonis, Vasilis and Shah, Kulin},
  journal = {arXiv preprint arXiv:2404.18893},
  title   = {Learning general Gaussian mixtures with efficient score matching},
  year    = {2024},
}

@InProceedings{sohl2015deep,
  author       = {Sohl-Dickstein, Jascha and Weiss, Eric and Maheswaranathan, Niru and Ganguli, Surya},
  booktitle    = {International Conference on Machine Learning},
  title        = {Deep unsupervised learning using nonequilibrium thermodynamics},
  year         = {2015},
  organization = {PMLR},
  pages        = {2256--2265},
}

@InProceedings{wen2023detecting,
  author    = {Wen, Yuxin and Liu, Yuchen and Chen, Chen and Lyu, Lingjuan},
  booktitle = {International Conference on Learning Representations},
  title     = {Detecting, Explaining, and Mitigating Memorization in Diffusion Models},
  year      = {2023},
}

@InProceedings{somepalli2023understanding,
  author  = {Somepalli, Gowthami and Singla, Vasu and Goldblum, Micah and Geiping, Jonas and Goldstein, Tom},
  booktitle = {Advances in Neural Information Processing Systems},
  title   = {Understanding and mitigating copying in diffusion models},
  year    = {2023},
  pages   = {47783--47803},
  volume  = {36},
}

@InProceedings{somepalli2023diffusion,
  author    = {Somepalli, Gowthami and Singla, Vasu and Goldblum, Micah and Geiping, Jonas and Goldstein, Tom},
  booktitle = {Proceedings of the IEEE/CVF Conference on Computer Vision and Pattern Recognition},
  title     = {Diffusion art or digital forgery? investigating data replication in diffusion models},
  year      = {2023},
  pages     = {6048--6058},
}

@Article{biroli2024dynamical,
  author  = {Biroli, Giulio and Bonnaire, Tony and de Bortoli, Valentin and M{\'e}zard, Marc},
  journal = {arXiv preprint arXiv:2402.18491},
  title   = {Dynamical Regimes of Diffusion Models},
  year    = {2024},
}

@InProceedings{brownverifying,
  author    = {Brown, Bradley CA and Caterini, Anthony L and Ross, Brendan Leigh and Cresswell, Jesse C and Loaiza-Ganem, Gabriel},
  booktitle = {International Conference on Learning Representations},
  title     = {Verifying the Union of Manifolds Hypothesis for Image Data},
  year      = {2023},
}

@InProceedings{kamkari2024geometric,
  title={A geometric view of data complexity: Efficient local intrinsic dimension estimation with diffusion models},
  author={Kamkari, Hamid and Ross, Brendan and Hosseinzadeh, Rasa and Cresswell, Jesse and Loaiza-Ganem, Gabriel},
  booktitle={Advances in Neural Information Processing Systems},
  volume={37},
  pages={38307--38354},
  year={2024}
}

@InProceedings{rombach2022high,
  author    = {Rombach, Robin and Blattmann, Andreas and Lorenz, Dominik and Esser, Patrick and Ommer, Bj{\"o}rn},
  booktitle = {Proceedings of the IEEE/CVF Conference on Computer Vision and Pattern Recognition},
  title     = {High-resolution image synthesis with latent diffusion models},
  year      = {2022},
  pages     = {10684--10695},
}

@InProceedings{kongdiffwave2021,
  author    = {Kong, Zhifeng and Ping, Wei and Huang, Jiaji and Zhao, Kexin and Catanzaro, Bryan},
  booktitle = {International Conference on Learning Representations},
  title     = {{DIFFWAVE}: A Versatile Diffusion Model for Audio Synthesis},
  year      = {2021},
}

@InProceedings{kong2020hifi,
  author  = {Kong, Jungil and Kim, Jaehyeon and Bae, Jaekyoung},
  booktitle = {Advances in Neural Information Processing Systems},
  title   = {Hi{F}i-{GAN}: Generative adversarial networks for efficient and high fidelity speech synthesis},
  year    = {2020},
  pages   = {17022--17033},
  volume  = {33},
}

@inproceedings{chung2023diffusion,
  title={Diffusion posterior sampling for general noisy inverse problems},
  author={Chung, Hyungjin and Kim, Jeongsol and McCann, Michael T and Klasky, Marc L and Ye, Jong Chul},
  booktitle={International Conference on Learning Representations},
  year={2023}
}

@inproceedings{li2025unifed,
  title={Unified convergence analysis for score-based diffusion models with deterministic samplers},
  author={Li, Runjia and Di, Qiwei and Gu, Quanquan},
  booktitle={International Conference on Learning Representations},
  year={2025}
}

@InProceedings{lee2022convergence,
  author  = {Lee, Holden and Lu, Jianfeng and Tan, Yixin},
  booktitle = {Advances in Neural Information Processing Systems},
  title   = {Convergence for score-based generative modeling with polynomial complexity},
  year    = {2022},
  pages   = {22870--22882},
  volume  = {35},
}

@inproceedings{lee2023convergence,
  title={Convergence of score-based generative modeling for general data distributions},
  author={Lee, Holden and Lu, Jianfeng and Tan, Yixin},
  booktitle={International Conference on Algorithmic Learning Theory},
  pages={946--985},
  year={2023},
  organization={PMLR}
}

@inproceedings{chen2023improved,
  title={Improved analysis of score-based generative modeling: User-friendly bounds under minimal smoothness assumptions},
  author={Chen, Hongrui and Lee, Holden and Lu, Jianfeng},
  booktitle={International Conference on Machine Learning},
  pages={4735--4763},
  year={2023},
  organization={PMLR}
}

@article{dupuis2025algorithm,
  title={Algorithm-and Data-Dependent Generalization Bounds for Score-Based Generative Models},
  author={Dupuis, Benjamin and Shariatian, Dario and Haddouche, Maxime and Durmus, Alain and Simsekli, Umut},
  journal={arXiv preprint arXiv:2506.03849},
  year={2025}
}

@article{li2025dimension,
  title={Dimension-free convergence of diffusion models for approximate Gaussian mixtures},
  author={Li, Gen and Cai, Changxiao and Wei, Yuting},
  journal={arXiv preprint arXiv:2504.05300},
  year={2025}
}

@InProceedings{gong2019intrinsic,
  author    = {Gong, Sixue and Boddeti, Vishnu Naresh and Jain, Anil K},
  booktitle = {Proceedings of the IEEE/CVF Conference on Computer Vision and Pattern Recognition},
  title     = {On the intrinsic dimensionality of image representations},
  year      = {2019},
  pages     = {3987--3996},
}

@InProceedings{stanczukdiffusion,
  author    = {Stanczuk, Jan Pawel and Batzolis, Georgios and Deveney, Teo and Sch{\"o}nlieb, Carola-Bibiane},
  booktitle = {International Conference on Machine Learning},
  title     = {Diffusion Models Encode the Intrinsic Dimension of Data Manifolds},
  year      = {2024},
}

@Article{lecun1998gradient,
  author    = {LeCun, Yann and Bottou, L{\'e}on and Bengio, Yoshua and Haffner, Patrick},
  journal   = {Proceedings of the IEEE},
  title     = {Gradient-based learning applied to document recognition},
  year      = {1998},
  number    = {11},
  pages     = {2278--2324},
  volume    = {86},
  publisher = {Ieee},
}

@Article{cifar10,
  author    = {Krizhevsky, Alex and Hinton, Geoffrey and others},
  title     = {Learning multiple layers of features from tiny images},
  year      = {2009},
  publisher = {Toronto, ON, Canada},
  journal = {Technical Report},
}

@Article{russakovsky2015imagenet,
  author    = {Russakovsky, Olga and Deng, Jia and Su, Hao and Krause, Jonathan and Satheesh, Sanjeev and Ma, Sean and Huang, Zhiheng and Karpathy, Andrej and Khosla, Aditya and Bernstein, Michael and others},
  journal   = {International Journal of Computer Vision},
  title     = {Imagenet large scale visual recognition challenge},
  year      = {2015},
  pages     = {211--252},
  volume    = {115},
  publisher = {Springer},
}

@InProceedings{ronneberger2015u,
  author       = {Ronneberger, Olaf and Fischer, Philipp and Brox, Thomas},
  booktitle    = {Medical Image Computing and Computer-Assisted Intervention--MICCAI 2015: 18th International Conference, Munich, Germany, October 5-9, 2015, Proceedings, part III 18},
  title        = {U-net: Convolutional networks for biomedical image segmentation},
  year         = {2015},
  organization = {Springer},
  pages        = {234--241},
}

@article{chen2025diffusion,
  title={Diffusion factor models: Generating high-dimensional returns with factor structure},
  author={Chen, Minshuo and Xu, Renyuan and Xu, Yumin and Zhang, Ruixun},
  journal={arXiv preprint arXiv:2504.06566},
  year={2025}
}

@inproceedings{li2024adapting,
title={Adapting to Unknown Low-Dimensional Structures in Score-Based Diffusion Models},
author={Gen Li and Yuling Yan},
booktitle={Advances in Neural Information Processing Systems},
year={2024}
}

@article{debortoli2022convergence,
title={Convergence of denoising diffusion models under the manifold hypothesis},
author={Valentin De Bortoli},
journal={Transactions on Machine Learning Research},
issn={2835-8856},
year={2022},
}

@article{liang2025low,
  title={Low-dimensional adaptation of diffusion models: Convergence in total variation},
  author={Liang, Jiadong and Huang, Zhihan and Chen, Yuxin},
  journal={arXiv preprint arXiv:2501.12982},
  year={2025}
}

@article{huang2024denoising,
  title={Denoising diffusion probabilistic models are optimally adaptive to unknown low dimensionality},
  author={Huang, Zhihan and Wei, Yuting and Chen, Yuxin},
  journal={arXiv preprint arXiv:2410.18784},
  year={2024}
}

@article{azangulov2024convergence,
  title={Convergence of diffusion models under the manifold hypothesis in high-dimensions},
  author={Azangulov, Iskander and Deligiannidis, George and Rousseau, Judith},
  journal={arXiv preprint arXiv:2409.18804},
  year={2024}
}

@inproceedings{tang2024adaptivity,
  title={Adaptivity of diffusion models to manifold structures},
  author={Tang, Rong and Yang, Yun},
  booktitle={International Conference on Artificial Intelligence and Statistics},
  pages={1648--1656},
  year={2024},
  organization={PMLR}
}

@Article{cui2025precise,
  author  = {Cui, Hugo and Pehlevan, Cengiz and Lu, Yue M},
  journal = {arXiv preprint arXiv:2501.03937},
  title   = {A precise asymptotic analysis of learning diffusion models: theory and insights},
  year    = {2025},
}

@Article{chen2025interpolation,
  author  = {Chen, Zhengdao},
  journal = {arXiv preprint arXiv:2502.19499},
  title   = {On the interpolation effect of score smoothing},
  year    = {2025},
}

@Article{chen2025denoising,
  author  = {Chen, Tianyu and Zhang, Yasi and Wang, Zhendong and Wu, Ying Nian and Leong, Oscar and Zhou, Mingyuan},
  journal = {arXiv preprint arXiv:2503.07578},
  title   = {Denoising score distillation: From noisy diffusion pretraining to one-step high-quality generation},
  year    = {2025},
}

@Article{gottwald2025localized,
  author  = {Gottwald, Georg A and Liu, Shuigen and Marzouk, Youssef and Reich, Sebastian and Tong, Xin T},
  journal = {arXiv preprint arXiv:2505.04417},
  title   = {Localized diffusion models for high-dimensional distributions generation},
  year    = {2025},
}

@Article{yakovlev2025generalization,
  author  = {Yakovlev, Konstantin and Puchkin, Nikita},
  journal = {arXiv preprint arXiv:2502.13662},
  title   = {Generalization error bound for denoising score matching under relaxed manifold assumption},
  year    = {2025}
}

@InProceedings{celeba,
  author    = {Liu, Ziwei and Luo, Ping and Wang, Xiaogang and Tang, Xiaoou},
  booktitle = {Proceedings of the IEEE International Conference on Computer Vision},
  title     = {Deep learning face attributes in the wild},
  year      = {2015},
  pages     = {3730--3738},
}

@Article{kingma2014adam,
  author  = {Kingma, Diederik P and Ba, Jimmy},
  journal = {arXiv preprint arXiv:1412.6980},
  title   = {Adam: A method for stochastic optimization},
  year    = {2014},
}

@InProceedings{ffhq,
  author    = {Kazemi, Vahid and Sullivan, Josephine},
  booktitle = {Proceedings of the IEEE Conference on Computer Vision and Pattern Recognition},
  title     = {One millisecond face alignment with an ensemble of regression trees},
  year      = {2014},
  pages     = {1867--1874},
}

@Article{wedin1972perturbation,
  author    = {Wedin, Per-{\AA}ke},
  journal   = {BIT Numerical Mathematics},
  title     = {Perturbation bounds in connection with singular value decomposition},
  year      = {1972},
  pages     = {99--111},
  volume    = {12},
  publisher = {Springer},
}

@Book{vershynin2018high,
  author    = {Vershynin, Roman},
  publisher = {Cambridge university press},
  title     = {High-dimensional probability: An introduction with applications in data science},
  year      = {2018},
  volume    = {47},
}

@InProceedings{metafaces,
  title={Training generative adversarial networks with limited data},
  author={Karras, Tero and Aittala, Miika and Hellsten, Janne and Laine, Samuli and Lehtinen, Jaakko and Aila, Timo},
  booktitle={Advances in Neural Information Processing Systems},
  volume={33},
  pages={12104--12114},
  year={2020}
}

@InProceedings{pizzi2022self,
  author    = {Pizzi, Ed and Roy, Sreya Dutta and Ravindra, Sugosh Nagavara and Goyal, Priya and Douze, Matthijs},
  booktitle = {Proceedings of the IEEE/CVF Conference on Computer Vision and Pattern Recognition},
  title     = {A self-supervised descriptor for image copy detection},
  year      = {2022},
  pages     = {14532--14542},
}

@article{liu2017subspace,
  title={SUBSPACE CLUSTERING BY (k, k)-SPARSE MATRIX FACTORIZATION.},
  author={Liu, Haixia and Cai, Jian-Feng and Wang, Yang},
  journal={Inverse Problems \& Imaging},
  volume={11},
  number={3},
  year={2017}
}

@article{lerman2018overview,
  title={An overview of robust subspace recovery},
  author={Lerman, Gilad and Maunu, Tyler},
  journal={Proceedings of the IEEE},
  volume={106},
  number={8},
  pages={1380--1410},
  year={2018},
  publisher={IEEE}
}

@InProceedings{pmlr-v139-li21f,
  author    = {Li, Gen and Gu, Yuantao},
  booktitle = {International Conference on Machine Learning},
  title     = {Theory of Spectral Method for Union of Subspaces-Based Random Geometry Graph},
  year      = {2021},
  pages     = {6337--6345},
  publisher = {PMLR},
  volume    = {139},
}

@InProceedings{agarwal2004k,
  author    = {Agarwal, Pankaj K and Mustafa, Nabil H},
  booktitle = {Proceedings of the 23rd ACM SIGMOD-SIGACT-SIGART Symposium on Principles of Database Systems},
  title     = {K-means projective clustering},
  year      = {2004},
  pages     = {155--165},
}

@Article{vidal2011subspace,
  author    = {Vidal, Ren{\'e}},
  journal   = {IEEE Signal Processing Magazine},
  title     = {Subspace clustering},
  year      = {2011},
  number    = {2},
  pages     = {52--68},
  volume    = {28},
  publisher = {IEEE},
}

@Article{elhamifar2013sparse,
  author    = {Elhamifar, Ehsan and Vidal, Ren{\'e}},
  journal   = {IEEE Transactions on Pattern Analysis and Machine Intelligence},
  title     = {Sparse subspace clustering: Algorithm, theory, and applications},
  year      = {2013},
  number    = {11},
  pages     = {2765--2781},
  volume    = {35},
  publisher = {IEEE},
}

@InProceedings{wang2022convergence,
  author       = {Wang, Peng and Liu, Huikang and So, Anthony Man-Cho and Balzano, Laura},
  booktitle    = {International Conference on Machine Learning},
  title        = {Convergence and recovery guarantees of the {K}-subspaces method for subspace clustering},
  year         = {2022},
  organization = {PMLR},
  pages        = {22884--22918},
}

@InProceedings{wang2013noisy,
  author       = {Wang, Yu-Xiang and Xu, Huan},
  booktitle    = {International Conference on Machine Learning},
  title        = {Noisy sparse subspace clustering},
  year         = {2013},
  organization = {PMLR},
  pages        = {89--97},
}

@misc{liu2026concentrationinequalitycovariancematrix,
      title={A Concentration Inequality for the Covariance Matrix of an Arbitrary Subset of Random Vectors}, 
      author={Huikang Liu and Peng Wang and Laura Balzano},
      year={2026},
      eprint={2606.24766},
      archivePrefix={arXiv},
}

@Article{soltanolkotabi2012geometric,
  author    = {Soltanolkotabi, Mahdi and Candes, Emmanuel J},
  journal   = {The Annals of Statistics},
  title     = {A geometric analysis of subspace clustering with outliers},
  year      = {2012},
  number    = {4},
  pages     = {2195--2238},
  volume    = {40},
  publisher = {Institute of Mathematical Statistics},
}

@article{lerman2025global,
  title={Global Convergence of Iteratively Reweighted Least Squares for Robust Subspace Recovery},
  author={Lerman, Gilad and Li, Kang and Maunu, Tyler and Zhang, Teng},
  journal={arXiv preprint arXiv:2506.20533},
  year={2025}
}

@Article{soltanolkotabi2014robust,
  author    = {Soltanolkotabi, Mahdi and Elhamifar, Ehsan and Candes, Emmanuel J},
  journal   = {Annals of Statistics},
  title     = {Robust subspace clustering},
  year      = {2014},
  number    = {2},
  pages     = {669--699},
  volume    = {42},
  publisher = {Institute of Mathematical Statistics},
}

@Article{liu2013efficient,
  author    = {Liu, Yuanyuan and Jiao, LC and Shang, Fanhua},
  journal   = {Pattern Recognition},
  title     = {An efficient matrix factorization based low-rank representation for subspace clustering},
  year      = {2013},
  number    = {1},
  pages     = {284--292},
  volume    = {46},
  publisher = {Elsevier},
}

@Article{vidal2005generalized,
  author    = {Vidal, Ren{\'e} and Ma, Yi and Sastry, Shankar},
  journal   = {IEEE Transactions on Pattern Analysis and Machine Intelligence},
  title     = {Generalized principal component analysis ({GPCA})},
  year      = {2005},
  number    = {12},
  pages     = {1945--1959},
  volume    = {27},
  publisher = {IEEE},
}

@article{MNIST,
  title={The {MNIST} Database of Handwritten Digit Images for Machine Learning Research},
  author={Li Deng},
  journal={IEEE Signal Processing Magazine},
  year={2012},
  volume={29},
  pages={141-142}
}

@online{FashionMNIST,
  author       = {Han Xiao and Kashif Rasul and Roland Vollgraf},
  title        = {Fashion-{MNIST}: a Novel Image Dataset for Benchmarking Machine Learning Algorithms},
  date         = {2017-08-28},
  year         = {2017},
  eprintclass  = {cs.LG},
  eprinttype   = {arXiv},
  eprint       = {cs.LG/1708.07747},
}

@article{li2024sharp,
  title={A sharp convergence theory for the probability flow {ODEs} of diffusion models},
  author={Li, Gen and Wei, Yuting and Chi, Yuejie and Chen, Yuxin},
  journal={arXiv preprint arXiv:2408.02320},
  year={2024}
}

@inproceedings{li2024understanding,
  title={Understanding generalizability of diffusion models requires rethinking the hidden {Gaussian} structure},
  author={Li, Xiang and Dai, Yixiang and Qu, Qing},
  booktitle={Advances in Neural Information Processing Systems},
  volume={37},
  pages={57499--57538},
  year={2024}
}

@inproceedings{choi2020stargan,
  title={Stargan v2: Diverse image synthesis for multiple domains},
  author={Choi, Yunjey and Uh, Youngjung and Yoo, Jaejun and Ha, Jung-Woo},
  booktitle={Proceedings of the IEEE/CVF Conference on Computer Vision and Pattern Recognition},
  pages={8188--8197},
  year={2020}
}

@inproceedings{
manor2023posterior,
title={On the Posterior Distribution in Denoising: Application to Uncertainty Quantification},
author={Hila Manor and Tomer Michaeli},
booktitle={International Conference on Learning Representations},
year={2024},
}

@InProceedings{manor2024zero,
  title = 	 {Zero-Shot Unsupervised and Text-Based Audio Editing Using {DDPM} Inversion},
  author =       {Manor, Hila and Michaeli, Tomer},
  booktitle = 	{International Conference on Machine Learning},
  pages = 	 {34603--34629},
  year = 	 {2024},
  volume = 	 {235},
  publisher =    {PMLR},
}

@inproceedings{shazeer2017outrageously,
  title={Outrageously Large Neural Networks: The Sparsely-Gated Mixture-of-Experts Layer},
  author={Shazeer, Noam and Mirhoseini, Azalia and Maziarz, Krzysztof and Davis, Andy and Le, Quoc and Hinton, Geoffrey and Dean, Jeff},
  booktitle={International Conference on Learning Representations},
  year={2017}
}

@inproceedings{liu2015faceattributes,
  title = {Deep Learning Face Attributes in the Wild},
  author = {Liu, Ziwei and Luo, Ping and Wang, Xiaogang and Tang, Xiaoou},
  booktitle = {Proceedings of International Conference on Computer Vision (ICCV)},
  year = {2015} 
}
\newpage 
\appendix
% \begin{appendix}
% \begin{center}
% {\Huge \bf Appendices}
% \end{center} 
% \setcounter{section}{0}
% \renewcommand\thesection{\Alph{section}}

In the appendix, the organization is as follows. We first provide proof details for the results in Sections \ref{sec:setup} and \ref{sec:results} in Appendices \ref{app:pf sec2} and \ref{app sec:pf low Gau}, respectively. Then, we present our experimental setups for \Cref{sec:setup} in Appendix \ref{app sec:expsetting_sec2}, for \Cref{sec:results} in Appendix \ref {app sec:expsetting_sec3}, and for \Cref{sec:experiments} in Appendix \ref{app sec:expset_appli_1}. Finally, additional auxiliary results for proving the main theorems are provided in Appendix \ref{app sec:auxi}. For ease of exposition, we introduce additional notations. Given a Gaussian random vector $\bm x \sim {\cal N}(\bm \mu, \bm \Sigma)$,  if $\bm \Sigma \succ \bm 0$, with abuse of notation, we write its pdf as
\begin{align}\label{eq:pdf}
{\cal N}(\bm x; \bm \mu, \bm \Sigma) := \frac{1}{(2\pi)^{n/2}\det^{1/2}(\bm \Sigma)} \exp\left( -\frac{1}{2}(\bm x - \bm \mu)^T\bm \Sigma^{-1}(\bm x - \bm \mu) \right).     
\end{align}  

\section{Proofs in Section \ref{sec:setup}}\label{app:pf sec2} 

When the data $\bm{x}_0$ is drawn from the \MoLRG~distribution (see Definition \ref{def:MoG}), the simplicity of Gaussian components allows us to derive a closed-form expression for the ground-truth posterior mean $\mathbb{E}\left[ \bm{x}_0 \mid \bm{x}_t \right]$ for all $t \in (0,1]$ as follows. To proceed, for each $\bm \Sigma_k^\star$,  we write its eigen-decomposition as follows: 
\begin{align}\label{eq:eig}
\bm \Sigma_k^\star = \bm U_k^\star \bm \Lambda_k^\star \bm U_k^{\star T},
\end{align}
where $\bm \Lambda_k^\star  = \mathrm{diag}(\lambda_{k,1}^\star,\dots,\lambda_{k,d_k}^\star)$ is a diagonal matrix with $\lambda_{k,1}^\star \ge \dots \ge \lambda_{k,d_k}^\star > 0$ being its positive eigenvalues and $\bm U_k^\star \in \mathcal{O}^{n\times d_k}$ is an orthonormal matrix whose columns are the corresponding eigenvectors.   

\begin{prop}\label{prop:score MoG}
    Suppose that the underlying data distribution $p_{\rm data}$ is a mixture of low-rank Gaussian distributions in Definition \ref{def:MoG}. In the forward process of diffusion models, the pdf of $\bm x_t$ for each $t > 0$ is 
    \begin{align}\label{eq:pt(x)}
        p_t(\bm x) =  \sum_{k=1}^K \pi_k \mathcal{N}\left(\bm x; s_t\bm \mu_k^\star, s_t^2\bm \Sigma_k^\star + \gamma_t^2 \bm I_n \right),
    \end{align}
    where $\gamma_t := s_t\sigma_t$. Moreover, the score function of $p_t(\bm x)$ is 
    \begin{align}\label{eq:score pt(x)}
        \nabla \log p_t(\bm x) =  \frac{1}{\gamma_t^2}\frac{\sum_{k=1}^K \pi_k \mathcal{N}\left(\bm x; s_t\bm \mu_k^\star, s_t^2\bm \Sigma_k^\star + \gamma_t^2 \bm I_n \right)\left(\bm I_n - \bm U_k^\star\bm D_{k,t}^\star\bm U_k^{\star T}
 \right) \left( s_t\bm \mu_k^\star - \bm x \right)}{ \sum_{k=1}^K \pi_k \mathcal{N}\left(\bm x; s_t\bm \mu_k^\star, s_t^2\bm \Sigma_k^\star + \gamma_t^2 \bm I_n \right)},
    \end{align}
where $\bm D_{k,t}^\star = \mathrm{diag}\left( \frac{s_t^2\lambda_{k,1}^\star}{\gamma_t^2 + s_t^2\lambda_{k,1}^\star},\dots, \frac{s_t^2\lambda_{k,d_k}^\star}{\gamma_t^2 + s_t^2\lambda_{k,d_k}^\star} \right)$. 
\end{prop} 
\begin{proof}
    Let $Y \in \{1,\dots,K\}$ be a discrete random variable that denotes the value of components of the mixture model. Note that $\gamma_t = s_t\sigma_t$. It follows from \Cref{def:MoG} that $\mathbb{P}(Y=k) = \pi_k$ for each $k \in [K]$. Conditioned on $Y=k$, we have $\bm x_0 \sim \mathcal{N}(\bm \mu_k^\star,\bm \Sigma_k^\star)$. This, together with \eqref{eq:trans}, implies $\bm x_t \sim \mathcal{N}\left(s_t\bm \mu_k^\star, s_t^2\bm \Sigma_k^\star + \gamma_t^2 \bm I_n  \right)$. Therefore, we have
\begin{align*}
    p_t(\bm x) = \sum_{k=1}^K p_t(\bm x \mid Y=k) \mathbb{P}\left(Y=k\right) = \sum_{k=1}^K \pi_k \mathcal{N}\left(\bm x; s_t\bm \mu_k^\star, s_t^2\bm \Sigma_k^\star + \gamma_t^2 \bm I_n \right). 
\end{align*}
Next, we directly compute
\begin{align*}
    \nabla \log p_t(\bm x) & = \frac{\nabla p_t(\bm x)}{p_t(\bm x)} =  \frac{\sum_{k=1}^K \pi_k \mathcal{N}\left(\bm x; s_t\bm \mu_k^\star, s_t^2\bm \Sigma_k^\star + \gamma_t^2 \bm I_n \right)\left( s_t^2\bm \Sigma_k^\star + \gamma_t^2 \bm I_n 
 \right)^{-1}\left( s_t\bm \mu_k^\star - \bm x \right)}{ \sum_{k=1}^K \pi_k \mathcal{N}\left(\bm x; s_t\bm \mu_k^\star, s_t^2\bm \Sigma_k^\star + \gamma_t^2 \bm I_n \right)}. 
\end{align*}
Using \eqref{eq:eig} and the matrix inversion lemma, we compute
\begin{align}\label{eq:matrix inversion}
    \left( s_t^2\bm \Sigma^\star_k + \gamma_t^2 \bm I_n 
 \right)^{-1} & =  \left( s_t^2\bm U_k^\star \bm \Lambda_k^\star\bm U_k^{\star T}  + \gamma_t^2 \bm I_n 
 \right)^{-1} = \frac{1}{\gamma_t^2} \left( \bm I_n - \bm U_k^\star \bm D_{k,t}^\star \bm U_k^{\star T}\right), 
\end{align}
where $\bm D_{k,t}^\star = \mathrm{diag}\left( \frac{s_t^2\lambda_{k,1}^\star}{\gamma_t^2 + s_t^2\lambda_{k,1}^\star},\dots, \frac{s_t^2\lambda_{k,d_k}^\star}{\gamma_t^2 + s_t^2\lambda_{k,d_k}^\star} \right)$. This, together with the above equation, implies \eqref{eq:score pt(x)}. 
\end{proof} 

Using the above result and \eqref{eq:Tweedie}, we compute $\mathbb{E}\left[ \bm x_0\mid \bm x_t\right]$ when $\bm{x}_0$ is drawn from the \MoLRG~distribution as follows: 
\begin{lemma}\label{lem:E[x_0]}
    Suppose $\bm x_0$ is drawn from the \MoLRG~distribution with parameters $\{\pi_k\}_{k=1}^K$, $\{\bm \mu_k^\star\}_{k=1}^K$, and $\{\bm \Sigma^\star_k \}_{k=1}^K$. For each time $t \in (0,1]$, it holds that  
    \begin{align}\label{eq:E MoG}
    &\mathbb{E} \left[ \bm x_0\mid \bm x_t\right] = \sum_{k=1}^K w_{k,t}^\star(\bm x_t) \left( \bm \mu_k^\star + \bm U_k^\star \bm D_{k,t}^\star \bm U_k^{\star T}\left( \frac{\bm x_t}{s_t} - \bm \mu_k^\star \right) \right), 
    % &\text{where}\ \bm D_{k,t}^\star = \mathrm{diag}\left( \frac{s_t^2\lambda_{k,1}^\star}{\gamma_t^2 + s_t^2\lambda_{k,1}^\star},\dots, \frac{s_t^2\lambda_{k,d_k}^\star}{\gamma_t^2 + s_t^2\lambda_{k,d_k}^\star} \right)\ \text{and}\ \notag \\
    % & w_{k,t}^\star(\bm x) := \frac{\pi_k\mathcal{N}\left(\bm x; s_t\bm \mu_k^\star, s_t^2\bm \Sigma_k^\star + \gamma_t^2 \bm I_n \right)}{\sum_{l=1}^K \pi_l\mathcal{N}\left(\bm x; s_t\bm \mu_l^\star, s_t^2\bm \Sigma_l^\star  + \gamma_t^2 \bm I_n \right)}. 
    \end{align} 
where 
\begin{align*}
    \bm D_{k,t}^\star = \mathrm{diag}\left( \frac{s_t^2\lambda_{k,1}^\star}{\gamma_t^2 + s_t^2\lambda_{k,1}^\star},\dots, \frac{s_t^2\lambda_{k,d_k}^\star}{\gamma_t^2 + s_t^2\lambda_{k,d_k}^\star} \right),\  w_{k,t}^\star(\bm x) := \frac{\pi_k\mathcal{N}\left(\bm x; s_t\bm \mu_k^\star, s_t^2\bm \Sigma_k^\star + \gamma_t^2 \bm I_n \right)}{\sum_{l=1}^K \pi_l\mathcal{N}\left(\bm x; s_t\bm \mu_l^\star, s_t^2\bm \Sigma_l^\star  + \gamma_t^2 \bm I_n \right)}. 
\end{align*}
\end{lemma}
\begin{proof} 
    According to \eqref{eq:Tweedie} and \eqref{eq:score pt(x)} in \Cref{prop:score MoG}, we compute 
    \begin{align*}
        \mathbb{E}\left[\bm x_0 \mid \bm x_t \right] & = \frac{\bm x_t + \gamma_t^2 \nabla \log p_t(\bm x_t)}{s_t} = \sum_{k=1}^K w_{k,t}^\star(\bm x_t) \left(\bm \mu_k^\star + \bm U_k^\star \bm D_{k,t}^\star \bm U_k^{\star T}\left( \frac{\bm x_t}{s_t} - \bm \mu_k^\star \right) \right). 
\end{align*}
\end{proof}
This lemma implies that the ground-truth posterior mean is a convex combination of the terms $\bm \mu_k^\star + \bm U_k^\star \bm D_{k,t}^\star \bm U_k^{\star T}\left( {\bm x_t}/{s_t} - \bm \mu_k^\star \right)$, where the weights are $w_{k,t}^\star(\bm x_t)$ for each $k \in [K]$. 

\section{Proofs in Section \ref{sec:results}}\label{app sec:pf low Gau}

When $\bm \mu_k^\star = \bm 0$ and $\bm \Lambda_k^\star = \bm I_{d_k}$ for each $k \in [K]$, we focus on a special instance of the \MoLRG~distribution in Definition \ref{def:MoG} as follows:
\begin{align}\label{eq:MoG dist}
   \bm x_0 \sim \sum_{k=1}^K \pi_k \mathcal{N}\left(\bm 0, \bm U_k^\star \bm U_k^{\star T} \right). 
\end{align} 
This, together with \Cref{lem:E[x_0]}, yields that the optimal parameterization for the DAE to learn the above distribution is  
\begin{align}\label{eq:MoG1}
   \bm x_{\bm \theta}(\bm x_t, t) = \dfrac{s_t}{s_t^2 + \gamma_t^2}  \sum_{k=1}^K w_{k,t}(\bm \theta; \bm x_t) \bm U_k\bm U_k^T \bm x_t,
\end{align}
where $\bm U_k \in \mathcal{O}^{n\times d_k}$ for each $k \in [K]$ and 
\begin{align}\label{eq:weight}
   w_{k,t}(\bm \theta; \bm x_t) = \frac{\pi_k \mathcal{N}(\bm x_t; \bm 0, s_t^2\bm U_k\bm U_k^T + \gamma_t^2 \bm I)}{\sum_{l=1}^K \pi_l \mathcal{N}(\bm x_t; \bm 0, s_t^2\bm U_l\bm U_l^T + \gamma_t^2 \bm I)}. 
\end{align} 

\subsection{Proof of \Cref{thm:1}}\label{app:pf thm1} 

\begin{proof} 
    Plugging \eqref{eq:para Gau} and $\bm x_t = s_t\bm x^{(i)} + \gamma_t\bm \epsilon$  into the integrand of \eqref{eq:em loss} yields 
    \begin{align*}
        &\ \mathbb{E}_{\bm \epsilon \sim \mathcal{N}(\bm 0, \bm I_n)} \left[\left\| \frac{s_t}{s_t^2 + \gamma_t^2} \bm U\bm U^T\left( s_t\bm x^{(i)} + \gamma_t \bm \epsilon\right) -  \bm x^{(i)} \right\|^2\right] \\
        = &\  \left\| \frac{s_t^2}{s_t^2 + \gamma_t^2} \bm U\bm U^T \bm x^{(i)} -  \bm x^{(i)}  \right\|^2 + {\color{black}\frac{(s_t\gamma_t)^2}{(s_t^2 + \gamma_t^2)^2}} \mathbb{E}_{\bm \epsilon \sim \mathcal{N}(\bm 0, \bm I_n)} \left[  \|\bm U\bm U^T \bm \epsilon \|^2 \right] \\
        = & \left\| \frac{s_t^2}{s_t^2 + \gamma_t^2} \bm U\bm U^T \bm x^{(i)} -  \bm x^{(i)}  \right\|^2 + {\color{black}\frac{(s_t\gamma_t)^2d}{(s_t^2 + \gamma_t^2)^2}},
    \end{align*}
 	where the first equality follows from $\mathbb{E}_{\bm \epsilon}[\langle \bm x, \bm \epsilon \rangle] = 0$ for any given $\bm x \in \R^n$ due to $\bm \epsilon \sim \mathcal{N}(\bm 0, \bm I_n)$, and the second equality uses $\mathbb{E}_{\bm \epsilon} \left[  \|\bm U\bm U^T \bm \epsilon \|^2 \right] = \mathbb{E}_{\bm \epsilon} \left[  \|\bm U^T \bm \epsilon \|^2 \right] = \sum_{i=1}^d \mathbb{E}_{\bm \epsilon}\left[  \|\bm u_i^T \bm \epsilon \|^2 \right] = d$ due to $\bm U \in \mathcal{O}^{n\times d}$ and $\bm \epsilon \sim \mathcal{N}(\bm 0, \bm I_n)$. 
    This, together with $\gamma_t = s_t\sigma_t$ and \eqref{eq:em loss}, yields 
\begin{align*}
    \ell(\bm U) = \frac{1}{N}\sum_{i=1}^N \int_{0}^1 \lambda_t \left( \|\bm x^{(i)}\|^2 - \frac{1+2\sigma_t^2}{(1+\sigma_t^2)^2} \|\bm U^T\bm x^{(i)}\|^2 + \frac{\sigma_t^2d}{(1+\sigma_t^2)^2} \right) {\rm d}t. 
\end{align*}
Therefore, minimizing the above function in terms of $\bm U$ amounts to
\begin{align*}
    \min_{\bm U^T\bm U = \bm I_d} -\int_{0}^1\frac{(1+2\sigma_t^2)\lambda_t}{(1+\sigma_t^2)^2}   {\rm d}t  \frac{1}{N}\sum_{i=1}^N \| \bm U^T\bm x^{(i)}\|^2,
\end{align*}
which is equivalent to Problem (\ref{eq:PCA}). 
\end{proof}

\subsection{Proof of Theorem \ref{thm:2}}\label{app:pf thm2} 

\begin{proof}
	For ease of exposition, let 
    \begin{align*}
        \bm X = \begin{bmatrix}
            \bm x^{(1)} & \dots & \bm x^{(N)} 
        \end{bmatrix} \in \R^{n\times N},\ \bm A = \begin{bmatrix}
            \bm a_1 & \dots & \bm a_N 
        \end{bmatrix} \in \R^{d \times N},\ \bm E = \begin{bmatrix}
            \bm e_1 & \dots & \bm e_N 
        \end{bmatrix} \in \R^{n \times N}. 
    \end{align*}
Using this and \eqref{eq:Ua+e}, we obtain
    \begin{align}\label{eq1:thm2}
        \bm X = \bm U^\star \bm A + \bm E. 
    \end{align} 
Let $r_A := \mathrm{rank}(\bm A) \le \min\{d,N\}$ and $\bm A = \bm U_A \bm \Sigma_A \bm V_A^T$ be a singular value decomposition (SVD) of $\bm A$, where $\bm U_A \in \mathcal{O}^{d \times r_A}$, $\bm V_A \in \mathcal{O}^{N \times r_A}$, and $\bm \Sigma_A \in \R^{r_A\times r_A}$.
    It follows from \Cref{thm:1} that Problem (\ref{eq:em loss}) with the parameterization \eqref{eq:para Gau} is equivalent to Problem (\ref{eq:PCA}).
    
    (i) Suppose that $N \ge d$. Applying \Cref{lem:Gau} with $\varepsilon =   {1}/{(2c_1)}$ to $\bm A \in \R^{d\times N}$, it holds with probability at least $1 - 1/2^{N-d+1} - \exp\left(-c_2N\right)$ that 
    \begin{align}
     \sigma_{\min}(\bm A) = \sigma_d(\bm A) \ge \frac{\sqrt{N}-\sqrt{d-1}}{2c_1},
    \end{align}
    where $c_1,c_2> 0$ are constants depending polynomially only on the Gaussian moment. This implies $r_A = d$ and $\bm U_A \in \mathcal{O}^d$. Since Problem (\ref{eq:PCA}) is a PCA problem, the columns of any optimal solution $\hat{\bm U} \in \mathcal{O}^{n\times d}$ consist of left singular vectors associated with the top $d$ singular values of $\bm X$. This, together with Wedin's Theorem \cite{wedin1972perturbation} and \eqref{eq1:thm2}, yields 
    \begin{align*}
        \left\|\hat{\bm U}\hat{\bm U}^T - \bm U^\star\bm U^{\star T}\right\|_F = \left\|\hat{\bm U}\hat{\bm U}^T - (\bm U^\star\bm U_A)(\bm U^{\star}\bm U_A)^T\right\|_F \le \frac{2\|\bm E\|_F}{\sigma_{\min}(\bm A)} {  \le } \frac{4c_1\|\bm E\|_F}{\sqrt{N} - \sqrt{d-1}}. 
    \end{align*}
This, together with absorbing $4$ into $c_1$, yields \eqref{rst1:thm 2}.

(ii) Suppose that $N < d$. According to \Cref{lem:Gau} with $\varepsilon =   {1}/{(2c_1)}$, it holds with probability at least $1 - 1/2^{d-N+1} - \exp\left(-c_2d\right) $ that 
    \begin{align}\label{eq2:thm2}
     \sigma_{\min}(\bm A) = \sigma_N(\bm A) \ge \frac{\sqrt{d}-\sqrt{N-1}}{2c_1},
    \end{align}
    where $c_1,c_2> 0$ are constants depending polynomially only on the Gaussian moment. This implies $r_A = N$ and $\bm U_A \in \mathcal{O}^{d\times N}$. This, together with the fact that $\bm A = \bm U_A \bm \Sigma_A \bm V_A^T$ is an SVD of $\bm A$, yields that $\bm U^\star \bm A = \left(\bm U^\star\bm U_A\right) \bm \Sigma_A \bm V_A^T$ is an SVD of $\bm U^\star \bm A$ with $\bm U^\star\bm U_A \in \mO^{n\times N}$. Note that $\mathrm{rank}(\bm X) \le N$. Let \(\bm U_X\in\mathcal O^{n\times N}\) be an orthonormal basis containing the left singular vectors of \(\bm X\) associated with its nonzero singular values. This, together with Wedin's Theorem \cite{wedin1972perturbation} and \eqref{eq2:thm2}, yields 
    \begin{align}\label{eq3:thm2}
        \left\|\bm U_X\bm U_X^T - \bm U^\star\bm U_A\bm U_A^T\bm U^{\star T}\right\|_F  \le \frac{2\|\bm E\|_F}{\sigma_{\min}(\bm A)} {  \le } \frac{4c_1\|\bm E\|_F}{\sqrt{d} - \sqrt{N-1}}. 
    \end{align}  
Let \(\bm S := \bm U^\star(\bm I-\bm U_A\bm U_A^T)\). Then
\(\mathrm{rank}(\bm S)=d-N\), and $\mathrm{rank}\big([\bm U_X,\bm S]\big)\le d.$ Since \(N<d\), one can verify that for any feasible
\(\bar{\bm U}_X\in\mathcal{O}^{n\times(d-N)}\) satisfying
\(\bm U_X^T\bar{\bm U}_X=\bm 0\), the matrix
\[
\hat{\bm U}=\begin{bmatrix}
    \bm U_X & \bar{\bm U}_X
\end{bmatrix} \in \mathcal{O}^{n\times d}
\]
is an optimal solution of Problem \eqref{eq:PCA}. 

We choose \(\bar{\bm U}_X\) as an optimal solution of the following optimization problem: 
\begin{align}\label{eq:P}
\bar{\bm U}_X \in \arg\min_{\bm V\in\mathcal{O}^{n\times(d-N)},\ \bm U_X^T\bm V=\bm 0}
\|\bm V^T\bm S\|_F^2.    
\end{align}
\begin{itemize}[leftmargin=*]
    \item If \(n\ge 2d-N\), then the orthogonal complement of
\(\mathrm{span}([\bm U_X,\bm S])\) has dimension at least \(n -d \ge d-N\). Therefore, there exists a feasible \(\bm V\) such that \(\bm V^T\bm S=\bm 0\), and the optimal value is \(0\).

    \item If \(n<2d-N\), the orthogonal complement of
\(\mathrm{span}([\bm U_X,\bm S])\) has dimension at least \(n-d\). Thus, we can choose
\(\bm V_1\in\mathbb{R}^{n\times(n-d)}\) such that \(\bm V_1^T[\bm U_X,\bm S]=\bm 0\), and complete it with
\(\bm V_2\in\mathbb{R}^{n\times(2d-N-n)}\) so that
\(\bm V=[\bm V_1,\bm V_2]\) is feasible. Then
\[
\|\bm V^T\bm S\|_F^2
=
\|\bm V_2^T\bm S\|_F^2
\le 2d-N-n.
\]
\end{itemize}
These, together with $\bm S = \bm U^\star(\bm I-\bm U_A\bm U_A^T)$ and \eqref{eq:P}, yield
\begin{align}\label{eq5:thm2}
    \left\| \bar{\bm U}_X^T \bm U^\star(\bm I-\bm U_A\bm U_A^T) \right\|_F^2 \le \max\{0,2d-N-n\}.
\end{align}

Then, we obtain that 
\begin{align*}
     \left\|\hat{\bm U}\hat{\bm U}^T - \bm U^\star\bm U^{\star T}\right\|_F
     &=
     \left\|
     \bm U_X \bm U_X^T + \bar{\bm U}_X \bar{\bm U}_X^T
     - \bm U^\star\bm U_A\bm U_A^T\bm U^{\star T}
     - \bm U^\star(\bm I - \bm U_A\bm U_A^T)\bm U^{\star T}
     \right\|_F \\
     &\ge
     \left\|
     \bar{\bm U}_X \bar{\bm U}_X^T
     - \bm U^\star(\bm I - \bm U_A\bm U_A^T)\bm U^{\star T}
     \right\|_F
     -
     \left\|
     \bm U_X\bm U_X^T
     - \bm U^\star\bm U_A\bm U_A^T\bm U^{\star T}
     \right\|_F \\
     &\ge
     \sqrt{2(d-N)-2\max\left\{0,2d-(n+N)\right\}}
     -
     \frac{4c_1\|\bm E\|_F}{\sqrt{d}-\sqrt{N-1}} \\
     &=
     \sqrt{2\min\{d-N,n-d\}}
     -
     \frac{4c_1\|\bm E\|_F}{\sqrt{d}-\sqrt{N-1}},
\end{align*}
where the second inequality follows from \eqref{eq3:thm2}, \eqref{eq5:thm2}, and the identity
\[
\left\|
\bar{\bm U}_X \bar{\bm U}_X^T
-
\bm U^\star(\bm I-\bm U_A\bm U_A^T)\bm U^{\star T}
\right\|_F^2
=
2(d-N)
-
2
\left\|
\bar{\bm U}_X^T\bm U^\star(\bm I-\bm U_A\bm U_A^T)
\right\|_F^2 .
\]
\end{proof}

\subsection{Theoretical Justification of the DAE in \eqref{eq:para MoG}}\label{app sec:para MoG}  

Substituting $\pi_1=\dots=\pi_K$ into \eqref{eq:weight} yields 
\begin{align*}
       w_{k,t}(\bm \theta; \bm x_t) & = \frac{ \mathcal{N}(\bm x_t; \bm 0, s_t^2\bm U_k\bm U_k^T + \gamma_t^2 \bm I)}{\sum_{l=1}^K \mathcal{N}(\bm x_t; \bm 0, s_t^2\bm U_l\bm U_l^T + \gamma_t^2 \bm I)} = \frac{\exp\left(  \phi_t\left\| \bm U_k^T \bm x_t \right\|^2 \right)}{\sum_{l=1}^K \exp\left(  \phi_t\left\| \bm U_l^T \bm x_t \right\|^2 \right)},
\end{align*}
where the second equality follows from \eqref{eq:pdf}, \eqref{eq:matrix inversion}, $d_1=\dots=d_K$, and $\phi_t := s_t^2/(2\gamma_t^2(s_t^2+\gamma_t^2))$. Noting $\bm x_t = s_t\bm x_0 + \gamma_t \bm \epsilon$, we compute
\begin{align*}
    \mathbb{E}_{\bm \epsilon}\left[ \|\bm U_k^T(s_t\bm x_0 + \gamma_t \bm \epsilon)\|^2 \right] = s_t^2 \|\bm U_k^T\bm x_0\|^2 + \gamma_t^2\mathbb{E}_{\bm \epsilon}[\|\bm U_k^T\bm \epsilon\|^2] = s_t^2 \|\bm U_k^T\bm x_0\|^2 + \gamma_t^2 d,
\end{align*} 
where the first equality is due to $\bm \epsilon \sim \mathcal{N}(\bm 0, \bm I_n)$ and $\mathbb{E}_{\bm \epsilon}[\langle \bm U_k^T\bm x_0, \bm U_k^T\bm \epsilon \rangle] =  0$ for each $k \in [K]$. 
This implies that when the random fluctuation of
\(\|\bm U_k^T(s_t\bm x_0+\gamma_t\bm\epsilon)\|^2\) is small compared to its mean, we can approximate $w_{k,t}(\bm \theta; \bm x_t)$ in  \eqref{eq:MoG1} well by
\begin{align*}
   w_{k,t}(\bm \theta; \bm x_t) & \approx \frac{\exp\left( \phi_t \left( s_t^2\|\bm U_k^T\bm x_0\|^2 + \gamma_t^2 d \right) \right)}{\sum_{l=1}^K \exp\left(  \phi_t  \left( s_t^2\|\bm U_l^T\bm x_0\|^2 + \gamma_t^2 d \right) \right)}. 
\end{align*} 
This softmax function can be further approximated by the hardmax function. Therefore, we obtain the parameterization \eqref{eq:wk1}. 

\subsection{Proof of Theorem \ref{thm:3}}\label{app:pf thm3} 

\begin{proof} 
    Plugging \eqref{eq:para MoG} into the integrand of \eqref{eq:em loss} yields 
    \begin{align*} 
         &\  \mathbb{E}_{\bm \epsilon} \left[\left\| \frac{s_t}{s_t^2 + \gamma_t^2} \sum_{k=1}^K \hat{w}_k(\bm \theta; \bm x^{(i)}) \bm U_k \bm U_k^{T} (s_t\bm x^{(i)} + \gamma_t \bm \epsilon)  -  \bm x^{(i)} \right\|^2\right] \notag  \\ 
         = &\  \left\| \frac{s_t^2}{s_t^2 + \gamma_t^2} \sum_{k=1}^K \hat{w}_k(\bm \theta; \bm x^{(i)}) \bm U_k \bm U_k^{T} \bm x^{(i)} - \bm x^{(i)} \right\|^2 + \frac{(s_t\gamma_t)^2}{(s_t^2 + \gamma_t^2)^2} \mathbb{E}_{\bm \epsilon} \left[ \left\|\sum_{k=1}^K\hat{w}_k(\bm \theta; \bm x^{(i)}) \bm U_k \bm U_k^{T} \bm \epsilon \right\|^2 \right] \\
         = &\   \frac{s_t^2}{s_t^2 + \gamma_t^2} \sum_{k=1}^K \left( \frac{s_t^2}{s_t^2 + \gamma_t^2} \hat{w}_k^2(\bm \theta; \bm x^{(i)}) - 2\hat{w}_k(\bm \theta; \bm x^{(i)}) \right) \|\bm U_k^{T} \bm x^{(i)}\|^2 + \|\bm x^{(i)}\|^2 + \frac{(s_t\gamma_t)^2d}{(s_t^2 + \gamma_t^2)^2} \sum_{k=1}^K\hat{w}_k^2(\bm \theta; \bm x^{(i)}),
    \end{align*}
   where the first equality follows from $\mathbb{E}_{\bm \epsilon} [\langle \bm x, \bm \epsilon \rangle] = 0$ for any fixed $\bm x\in \R^n$ due to $\bm \epsilon \sim \mathcal{N}(\bm 0, \bm I_n)$, and the last equality uses $\bm U_k \in \mathcal{O}^{n\times d}$ and $\bm U_k^T\bm U_l = \bm 0$ for all $k \neq l$.  This, together with \eqref{eq:em loss} and $\gamma_t = s_t\sigma_t$, yields
    \begin{align*}
    \ell(\bm \theta) & = \frac{1}{N}\sum_{i=1}^N\sum_{k=1}^K  \int_0^1 \frac{\lambda_t}{1 + \sigma_t^2} \left(  \frac{1}{1 + \sigma_t^2} \hat{w}_k^2(\bm \theta; \bm x^{(i)}) - 2\hat{w}_k(\bm \theta; \bm x^{(i)})\right){\rm d}t  \|\bm U_k^T\bm x^{(i)}\|^2 + \\
    & \frac{1}{N} \int_0^1 \lambda_t {\rm d}t \sum_{i=1}^N \|\bm x^{(i)}\|^2  +  \left(\int_0^1 \frac{\sigma_t^2\lambda_t}{(1 + \sigma_t^2)^2} {\rm d}t\right) \frac{d}{N} \sum_{i=1}^N \sum_{k=1}^K \hat{w}_k^2(\bm \theta; \bm x^{(i)}). 
    \end{align*} 
    According to \eqref{eq:para MoG}, we partition \([N]\) into
\(\{C_k(\bm \theta)\}_{k=1}^K\) defined in \eqref{eq:Ck}.  Then, by the definition of the hardmax weights in \eqref{eq:para MoG}, we have
\[
\hat w_k(\bm\theta;\bm x^{(i)})=
\begin{cases}
1, & i\in C_k(\bm\theta),\\
0, & i\notin C_k(\bm\theta).
\end{cases}
\]
Then, we obtain 
\begin{align*}
\sum_{i=1}^N \sum_{k=1}^K  \hat{w}_k^2(\bm \theta; \bm x^{(i)})
= \sum_{k=1}^K \sum_{i \in C_k(\bm \theta)} 1 = N.      
\end{align*}
    Using this partition in the above loss function, minimizing \(\ell(\bm\theta)\) is equivalent to minimizing  
    \begin{align*}
        & \frac{1}{N}\sum_{i=1}^N\sum_{k=1}^K  \int_0^1 \frac{\lambda_t}{1 + \sigma_t^2} \left(  \frac{1}{1 + \sigma_t^2} \hat{w}_k^2(\bm \theta; \bm x^{(i)}) - 2 \hat{w}_k(\bm \theta; \bm x^{(i)})\right){\rm d}t  \|\bm U_k^T\bm x^{(i)}\|^2 \\
        = & \left(\int_0^1 \frac{\lambda_t}{1+\sigma_t^2}  \left(\frac{1}{1+\sigma_t^2} - 2\right) {\rm d}t\right) \frac{1}{N} \sum_{k=1}^K \sum_{ i \in C_k(\bm \theta)}  \|\bm U_k^T\bm x^{(i)}\|^2.
    \end{align*}
    Since $\frac{\lambda_t}{1+\sigma_t^2}  \left(\frac{1}{1+\sigma_t^2} - 2\right) < 0$ for all $t \in [0,1]$, minimizing the above function is equivalent to 
    \begin{align*}
        \max_{\bm \theta} \frac{1}{N} \sum_{k=1}^K \sum_{ i \in C_k(\bm \theta)} \|\bm U_k^T\bm x^{(i)}\|^2 \qquad \mathrm{s.t.}\ \left[
            \bm U_1\ \dots\ \bm U_K
        \right] \in \mathcal{O}^{n\times dK}. 
    \end{align*}
\end{proof}

\subsection{Proof of \Cref{thm:4}}\label{app:pf thm4}

We first establish two basic consequences of the hard assignment rule. The first shows that the ground-truth subspace bases recover the true mixture assignments. 
The second compares the objective value at a global optimizer with that at the ground-truth subspaces. For ease of exposition, let
\[
\bm\theta^\star:=\{\bm U_k^\star\}_{k=1}^K,
\quad
C_k^\star:=\{i\in[N]:z_i=k\},
\]
where \(z_i\) denotes the latent component label in \Cref{AS:1}. 
\begin{prop}\label{prop:true-assignment-optimality}
Suppose that Assumptions~\ref{AS:1}, \ref{AS:2}, and \ref{AS:3} hold,
$d \gtrsim \log N$, and 
\begin{align}\label{eq:noise 1}
\|\bm e_i\|
<
\frac{\sqrt d-2\sqrt{\log N}-2}{2},
\quad \forall i\in[N].
\end{align}
Then, with probability at least \(1-2N^{-1}\), the following statements hold:\\
(i) It holds that 
\begin{align}\label{eq:true Ck}
C_k(\bm\theta^\star)=C_k^\star,\quad \forall k\in[K].
\end{align}
(ii) For any optimal solution
\(\hat{\bm\theta}:=\{\hat{\bm U}_k\}_{k=1}^K\) of Problem~\eqref{eq:SC}, we have
\begin{align}\label{eq:optimality gap}
\sum_{i=1}^N \|\bm a_i\|^2 - \sum_{l=1}^K\sum_{k=1}^K
\sum_{i\in C_k(\hat{\bm\theta})\cap C_l^\star}
\|\hat{\bm U}_k^T\bm U_l^\star\bm a_i\|^2 \le
6\delta N\sqrt d+N\delta^2,
\end{align}
where \(\delta:=\max_{i\in[N]}\|\bm e_i\|\).
\end{prop}
\begin{proof}
Suppose that \eqref{eq:norm Gau} holds for all \(i\in[N]\), which happens with probability at least \(1 - 2N^{-1}\) according to \Cref{lem:norm Gau}. Then, for all \(i\in[N]\),
\begin{align}\label{eq0:prop true assignment}
\sqrt d-(2\sqrt{\log N}+2)
\le
\|\bm a_i\|
\le
\sqrt d+(2\sqrt{\log N}+2).
\end{align}

We first prove part (i). For any \(i\in C_k^\star\), \Cref{AS:1} gives
\(\bm x^{(i)}=\bm U_k^\star\bm a_i+\bm e_i\). Therefore,
\begin{align}
\|\bm U_k^{\star T}\bm x^{(i)}\|
&=
\|\bm a_i+\bm U_k^{\star T}\bm e_i\|
\ge
\|\bm a_i\|-\|\bm e_i\|, \label{eq:true-proj}\\
\|\bm U_l^{\star T}\bm x^{(i)}\|
&=
\|\bm U_l^{\star T}\bm e_i\|
\le
\|\bm e_i\|,
\quad \forall l\neq k, \label{eq:false-proj}
\end{align}
where the second line uses \(\bm U_l^{\star T}\bm U_k^\star=\bm 0\) for \(l\neq k\). These, together with \eqref{eq:noise 1} and \eqref{eq0:prop true assignment}, yield
\[
\|\bm U_k^{\star T}\bm x^{(i)}\|
>
\|\bm U_l^{\star T}\bm x^{(i)}\|,
\quad \forall l\neq k.
\]
Hence, \(k\) is the unique maximizer in the definition of \(C_k(\bm\theta^\star)\), and thus
\(i\in C_k(\bm\theta^\star)\). This implies
\(C_k^\star\subseteq C_k(\bm\theta^\star)\) for all \(k\in[K]\).
Since both \(\{C_k^\star\}_{k=1}^K\) and
\(\{C_k(\bm\theta^\star)\}_{k=1}^K\) form partitions of \([N]\), we obtain
\eqref{eq:true Ck}.

We next prove part (ii). For ease of exposition, let 
\[
f(\bm\theta) := \sum_{k=1}^K
\sum_{i\in C_k(\bm\theta)}
\|\bm U_k^T\bm x^{(i)}\|^2.
\]
By part (i), we have \(C_k(\bm\theta^\star)=C_k^\star\) for all \(k\in[K]\). Hence,
\begin{align}\label{eq:f theta star}
f(\bm\theta^\star)
&=
\sum_{k=1}^K
\sum_{i\in C_k^\star}
\|\bm U_k^{\star T}\bm x^{(i)}\|^2
=
\sum_{k=1}^K
\sum_{i\in C_k^\star}
\|\bm a_i+\bm U_k^{\star T}\bm e_i\|^2 \notag \\
&=
\sum_{i=1}^N\|\bm a_i\|^2
+
2\sum_{k=1}^K
\sum_{i\in C_k^\star}
\langle \bm a_i,\bm U_k^{\star T}\bm e_i\rangle
+
\sum_{k=1}^K
\sum_{i\in C_k^\star}
\|\bm U_k^{\star T}\bm e_i\|^2.
\end{align}
On the other hand,
\begin{align}\label{eq:f theta hat}
f(\hat{\bm\theta})
&=
\sum_{k=1}^K
\sum_{i\in C_k(\hat{\bm\theta})}
\|\hat{\bm U}_k^T\bm x^{(i)}\|^2 =
\sum_{l=1}^K\sum_{k=1}^K
\sum_{i\in C_k(\hat{\bm\theta})\cap C_l^\star}
\|\hat{\bm U}_k^T(\bm U_l^\star\bm a_i+\bm e_i)\|^2 \notag \\
&=
\sum_{l=1}^K\sum_{k=1}^K
\sum_{i\in C_k(\hat{\bm\theta})\cap C_l^\star}
\left(
\|\hat{\bm U}_k^T\bm U_l^\star\bm a_i\|^2
+
2\langle
\bm a_i,
\bm U_l^{\star T}\hat{\bm U}_k\hat{\bm U}_k^T\bm e_i
\rangle
\right) \notag \\
&\quad+
\sum_{k=1}^K
\sum_{i\in C_k(\hat{\bm\theta})}
\|\hat{\bm U}_k^T\bm e_i\|^2.
\end{align}
Since \(\hat{\bm\theta}\) is an optimal solution of Problem~\eqref{eq:SC}, we have
\(f(\hat{\bm\theta})\ge f(\bm\theta^\star)\). Combining this with
\eqref{eq:f theta star} and \eqref{eq:f theta hat} gives
\begin{align*}
\sum_{i=1}^N \|\bm a_i\|^2
-
\sum_{l=1}^K\sum_{k=1}^K
\sum_{i\in C_k(\hat{\bm\theta})\cap C_l^\star}
\|\hat{\bm U}_k^T\bm U_l^\star\bm a_i\|^2 
&\le
2\sum_{l=1}^K\sum_{k=1}^K
\sum_{i\in C_k(\hat{\bm\theta})\cap C_l^\star}
\left|
\langle
\bm a_i,
\bm U_l^{\star T}\hat{\bm U}_k\hat{\bm U}_k^T\bm e_i
\rangle
\right| \\
& +
\sum_{k=1}^K
\sum_{i\in C_k(\hat{\bm\theta})}
\|\hat{\bm U}_k^T\bm e_i\|^2
+
2\sum_{k=1}^K
\sum_{i\in C_k^\star}
\left|
\langle \bm a_i,\bm U_k^{\star T}\bm e_i\rangle
\right|.
\end{align*}
Using \(\|\bm e_i\|\le\delta\) and
\(\bm U_k^\star,\hat{\bm U}_k\in\mathcal O^{n\times d}\), we obtain
\[
\sum_{i=1}^N \|\bm a_i\|^2
-
\sum_{l=1}^K\sum_{k=1}^K
\sum_{i\in C_k(\hat{\bm\theta})\cap C_l^\star}
\|\hat{\bm U}_k^T\bm U_l^\star\bm a_i\|^2
\le
4\delta\sum_{i=1}^N\|\bm a_i\|+N\delta^2.
\]
Finally, by \eqref{eq0:prop true assignment} and \(d\gtrsim\log N\), we have $\|\bm a_i\|\le \sqrt d+2\sqrt{\log N}+2\le 3\sqrt d/2,$
and hence
\[
4\delta\sum_{i=1}^N\|\bm a_i\|+N\delta^2
\le
6\delta N\sqrt d+N\delta^2.
\]
This proves \eqref{eq:optimality gap}.
\end{proof}

We now show that, under a sufficiently small noise level, any global solution of the subspace clustering problem recovers the true assignments up to a permutation. The key challenge is that the overlaps \(C_r(\hat{\bm\theta})\cap C_k^\star\) are data-dependent subsets. To address this, we use \Cref{lem:cov arbitrary subset}, which holds uniformly over arbitrary subsets.

\begin{prop}\label{prop:assignment recovery}
Suppose that Assumptions~\ref{AS:1}, \ref{AS:2}, and \ref{AS:3} hold. Let
\(\hat{\bm\theta}=\{\hat{\bm U}_k\}_{k=1}^K\) be an optimal solution of Problem~\eqref{eq:SC} and
\(N_{\min}:=\min_{k\in[K]}N_k\). Suppose that \(d\gtrsim \log N\) and
\begin{align}\label{eq:delta}
\|\bm e_i\| \le c_\delta
\min\left\{\frac{N_{\min}}{K^3N\sqrt d},
\frac{\sqrt d}{N}
\right\},\quad \forall i \in [N], 
\end{align}
where \(c_\delta>0\) is a sufficiently small absolute constant. Then, with probability at least
\[
1
-
2N^{-1}
-
CK\exp\left(
Cd\log(CK)
-
c\frac{N_{\min}}{K^6}
\right),
\]
there exists a permutation \(\Pi:[K]\to[K]\) such that
\[
C_{\Pi(k)}(\hat{\bm\theta})=C_k^\star,\qquad \forall k\in[K],
\]
where \(c,C>0\) are absolute constants.
\end{prop}
 \begin{proof}
Recall that \(\delta:=\max_{i\in[N]}\|\bm e_i\|\). 
By \eqref{eq:delta}, \(d\gtrsim\log N\), and a sufficiently small choice of \(c_\delta>0\), the noise condition in \Cref{prop:true-assignment-optimality}
holds. Therefore, according to \Cref{prop:true-assignment-optimality},  it holds with probability at least \(1-2N^{-1}\) that 
\begin{align}\label{eq:optimality gap use}
\sum_{i=1}^N \|\bm a_i\|^2 -
\sum_{l=1}^K\sum_{r=1}^K
\sum_{i\in C_r(\hat{\bm\theta})\cap C_l^\star}
\|\hat{\bm U}_r^T\bm U_l^\star\bm a_i\|^2
\le B :=6\delta N\sqrt d+N\delta^2. 
\end{align}
Using \eqref{eq:delta}, we obtain for sufficiently small \(c_\delta>0\) that
\begin{align}\label{eq:B-small}
B\le c_0\frac{N_{\min}}{K^3},
\quad
8B<d,
\end{align}
where \(c_0>0\) is a sufficiently small absolute constant. For ease of exposition, let
\[
N_{rk}:=|C_r(\hat{\bm\theta})\cap C_k^\star|,
\quad \forall r,k\in[K].
\]
For each \(k\in[K]\), choose
\(\Pi(k)\in\arg\max_{r\in[K]}N_{rk}\). Since
\(\sum_{r=1}^K N_{rk}=N_k\), we have
\begin{align}\label{eq2:prop assignment}
N_{\Pi(k)k}\ge \frac{N_k}{K}.    
\end{align}

Fix any \(k\in[K]\). From \eqref{eq:optimality gap use}, and using
\(\|\hat{\bm U}_r^T\bm U_l^\star\bm a_i\|\le \|\bm a_i\|\), we obtain
\begin{align}\label{eq:main-overlap-k}
B
&\ge
\sum_{i\in C_{\Pi(k)}(\hat{\bm\theta})\cap C_k^\star}
\left(
\|\bm a_i\|^2
-
\|\hat{\bm U}_{\Pi(k)}^T\bm U_k^\star\bm a_i\|^2
\right) \notag \\
&=
\left\langle
\bm I-\bm U_k^{\star T}\hat{\bm U}_{\Pi(k)}
\hat{\bm U}_{\Pi(k)}^T\bm U_k^\star,
\sum_{i\in C_{\Pi(k)}(\hat{\bm\theta})\cap C_k^\star}
\bm a_i\bm a_i^T
\right\rangle .
\end{align}
Conditioning on the latent labels, the vectors
\(\{\bm a_i:i\in C_k^\star\}\) are i.i.d. standard Gaussian vectors. We apply \Cref{lem:cov arbitrary subset} with \(\eta=1/4\), \(\kappa=1/K\),
\(M=N_k\), and \(m=d\). In this case,
\[
\gamma_K
=
\Phi^{-1}\left(\frac12+\frac{3}{8K}\right),
\]
and a Taylor expansion of \(\Phi\) around zero gives
\[
\gamma_K\asymp K^{-1},
\quad
\mu_K
:=
\frac{3}{4K}
-
\sqrt{\frac{2\gamma_K^2}{\pi}}
\exp\left(-\frac{\gamma_K^2}{2}\right)
\asymp K^{-3}.
\]
Therefore, with probability at least $1-C_1\exp\left(
C_1d\log(C_1K) - c_1{N_k}/{K^6}\right),$
the following bound holds simultaneously for every subset
\(\Omega\subseteq C_k^\star\) satisfying \(|\Omega|\ge N_k/K\):
\[
\lambda_{\min}\left(
\sum_{i\in\Omega}\bm a_i\bm a_i^T
\right)
\ge
c_2\frac{N_k}{K^3},
\]
where \(c_1,c_2,C_1>0\) are absolute constants. By
\eqref{eq2:prop assignment}, \(C_{\Pi(k)}(\hat{\bm\theta})\cap C_k^\star\)
is such a subset. Hence,
\[
\lambda_{\min}\left(
\sum_{i\in C_{\Pi(k)}(\hat{\bm\theta})\cap C_k^\star}
\bm a_i\bm a_i^T
\right)
\ge
c_2\frac{N_k}{K^3}.
\]
This, together with \Cref{lem:trace} and \eqref{eq:main-overlap-k}, yields
\[
\mathrm{Tr}\left(
\bm I-\bm U_k^{\star T}\hat{\bm U}_{\Pi(k)}
\hat{\bm U}_{\Pi(k)}^T\bm U_k^\star
\right) \le \frac{K^3B}{c_2N_k}.
\]
Since \([\bm U_1^\star,\dots,\bm U_K^\star]\in\mathcal O^{n\times dK}\), it follows that
\begin{align}\label{eq:cross-subspace-bound}
\sum_{l\neq k}
\|\hat{\bm U}_{\Pi(k)}^T\bm U_l^\star\|_F^2
&\le
\mathrm{Tr}\left(
\bm I-\hat{\bm U}_{\Pi(k)}^T
\bm U_k^\star\bm U_k^{\star T}
\hat{\bm U}_{\Pi(k)}
\right) \le \frac{K^3B}{c_2N_k} \le \frac12,
\end{align}
where the last inequality follows from \eqref{eq:B-small} by taking
\(c_0>0\) sufficiently small.

Next, again by \eqref{eq:optimality gap use},
\[
B
\ge
\sum_{l\neq k}
\sum_{i\in C_{\Pi(k)}(\hat{\bm\theta})\cap C_l^\star}
\left(
\|\bm a_i\|^2
-
\|\hat{\bm U}_{\Pi(k)}^T\bm U_l^\star\bm a_i\|^2
\right).
\]
For each \(l\neq k\), \eqref{eq:cross-subspace-bound} gives $\|\hat{\bm U}_{\Pi(k)}^T\bm U_l^\star\|^2
\le \|\hat{\bm U}_{\Pi(k)}^T\bm U_l^\star\|_F^2
\le 1/2$. Thus, for every \(i\in C_{\Pi(k)}(\hat{\bm\theta})\cap C_l^\star\) with \(l\neq k\),
\[
\|\bm a_i\|^2
-
\|\hat{\bm U}_{\Pi(k)}^T\bm U_l^\star\bm a_i\|^2
\ge
\frac12\|\bm a_i\|^2
\ge
\frac12(\sqrt d-2\sqrt{\log N}-2)^2
\ge
\frac d8,
\]
where the last inequality follows from \(d\gtrsim\log N\). Hence,
\[
B
\ge
\frac d8
\sum_{l\neq k}N_{\Pi(k)l}.
\]
Using \(8B<d\), we obtain
\(\sum_{l\neq k}N_{\Pi(k)l}<1\). Since the left-hand side is an integer,
\(N_{\Pi(k)l}=0\) for all \(l\neq k\). Therefore,
\[
C_{\Pi(k)}(\hat{\bm\theta})\subseteq C_k^\star,
\quad \forall k\in[K].
\]

It remains to show that \(\Pi\) is a permutation. If \(\Pi(k)=\Pi(l)\) for some
\(k\neq l\), then
\[
C_{\Pi(k)}(\hat{\bm\theta})
=
C_{\Pi(l)}(\hat{\bm\theta})
\subseteq
C_k^\star\cap C_l^\star
=
\varnothing,
\]
which contradicts \(N_{\Pi(k)k}\ge N_k/K>0\). Hence \(\Pi\) is injective and
therefore a permutation of \([K]\). Since both
\(\{C_{\Pi(k)}(\hat{\bm\theta})\}_{k=1}^K\) and
\(\{C_k^\star\}_{k=1}^K\) are partitions of \([N]\), we conclude that
\[
C_{\Pi(k)}(\hat{\bm\theta})=C_k^\star,
\qquad \forall k\in[K].
\]

Finally, applying a union bound over \(k\in[K]\) for the covariance events and
combining it with the event in \Cref{prop:true-assignment-optimality}
gives probability at least
\[
1
-
2N^{-1}
-
C_1K\exp\left(
C_1d\log(C_1K)
-
c_1\frac{N_{\min}}{K^6}
\right).
\]
This proves the desired result.
\end{proof}

\begin{proof}{\bf of \Cref{thm:4}.}
By \Cref{thm:3}, Problem~\eqref{eq:em loss} is equivalent to the subspace clustering problem in \eqref{eq:SC}. Hence, for part (i), it suffices to analyze the optimal solutions of Problem~\eqref{eq:SC}.

We first prove part (i). Let \(\mathcal E\) denote the event in  \Cref{prop:assignment recovery}. Under the assumptions in part (i),  \Cref{prop:assignment recovery} implies that \(\mathcal E\) holds with probability at least \(1-N^{-\Omega(1)}\). On this event, for any optimal solution
\(\hat{\bm\theta}=\{\hat{\bm U}_k\}_{k=1}^K\) of Problem~\eqref{eq:SC}, there exists a permutation \(\Pi:[K]\to[K]\) such that
\begin{align}\label{eq:assign recovered thm4}
    C_{\Pi(k)}(\hat{\bm\theta})=C_k^\star,
    \quad \forall k\in[K].
\end{align}
After relabeling the estimated subspaces according to \(\Pi\), the objective over the recovered clusters decomposes into the component-wise PCA problems
\begin{align}\label{eq:component PCA thm4}
    \max_{\bm U\in\mathcal O^{n\times d}}
    \frac{1}{N_k}
    \sum_{i\in C_k^\star}
    \|\bm U^T\bm x^{(i)}\|^2,
    \qquad k\in[K].
\end{align}
For each \(k\in[K]\), the samples in \(C_k^\star\) satisfy $\bm x^{(i)}=\bm U_k^\star\bm a_i+\bm e_i$ for all $i\in C_k^\star$, where \(\bm a_i\overset{i.i.d.}{\sim}\mathcal N(\bm0,\bm I_d)\). Since the additional sample-size condition in part (i) implies \(N_k\ge d+\Omega(\log N)\) for all \(k\in[K]\), applying \Cref{thm:2} to \eqref{eq:component PCA thm4} yields that, with probability at least \(1-N^{-\Omega(1)}\), for all \(k\in[K]\),
\[
\left\|\hat{\bm U}_{\Pi(k)}\hat{\bm U}_{\Pi(k)}^T
-
\bm U_k^\star\bm U_k^{\star T}
\right\|_F
\le
\frac{
c_1\sqrt{\sum_{i\in C_k^\star}\|\bm e_i\|^2}
}{
\sqrt{N_k}-\sqrt{d-1}
}.
\]
Combining this event with \(\mathcal E\) proves part (i).

We next prove part (ii). This part follows directly from \Cref{thm:2}(ii) applied to the PCA problem over \(C_{k_0}^\star\), since for \(i\in C_{k_0}^\star\), \Cref{AS:1} gives $\bm x^{(i)}=\bm U_{k_0}^\star\bm a_i+\bm e_i$ for all $\bm a_i\overset{i.i.d.}{\sim}\mathcal N(\bm0,\bm I_d).$  This proves part (ii).
\end{proof}

\section{Experimental Setups in \Cref{sec:setup}}\label{app sec:expsetting_sec2} 

In this section, we provide the detailed experimental setup for \Cref{subsec:exp-real}.  Given a real-world dataset $\{\bm x^{(i)}\}_{i=1}^N$ with $K$ classes, we outline the procedure to estimate a \MoLRG~distribution from the data. First, we set $\pi_k = 1/K$ and compute $\bm \mu_k$ as the mean of all images in class $k$. We then estimate the $\bm U_k$ and $\bm \Lambda_k$ by computing a rank-$d_k$ truncated SVD of the covariance matrix for the samples in class $k$. We plug these parameters into $\E[\bm x_0\mid \bm x_t]$ in $\eqref{eq:DAE para}$ and compute the score function $\nabla\log p_t\left(\bm x_t\right)$ using \eqref{eq:Tweedie}. Finally, we use the estimated score function to generate images by numerically solving \eqref{eq:reve}. 

\begin{figure*}[t]
\begin{center}
	\begin{subfigure}{0.43\textwidth}
    	\includegraphics[width=1\textwidth]{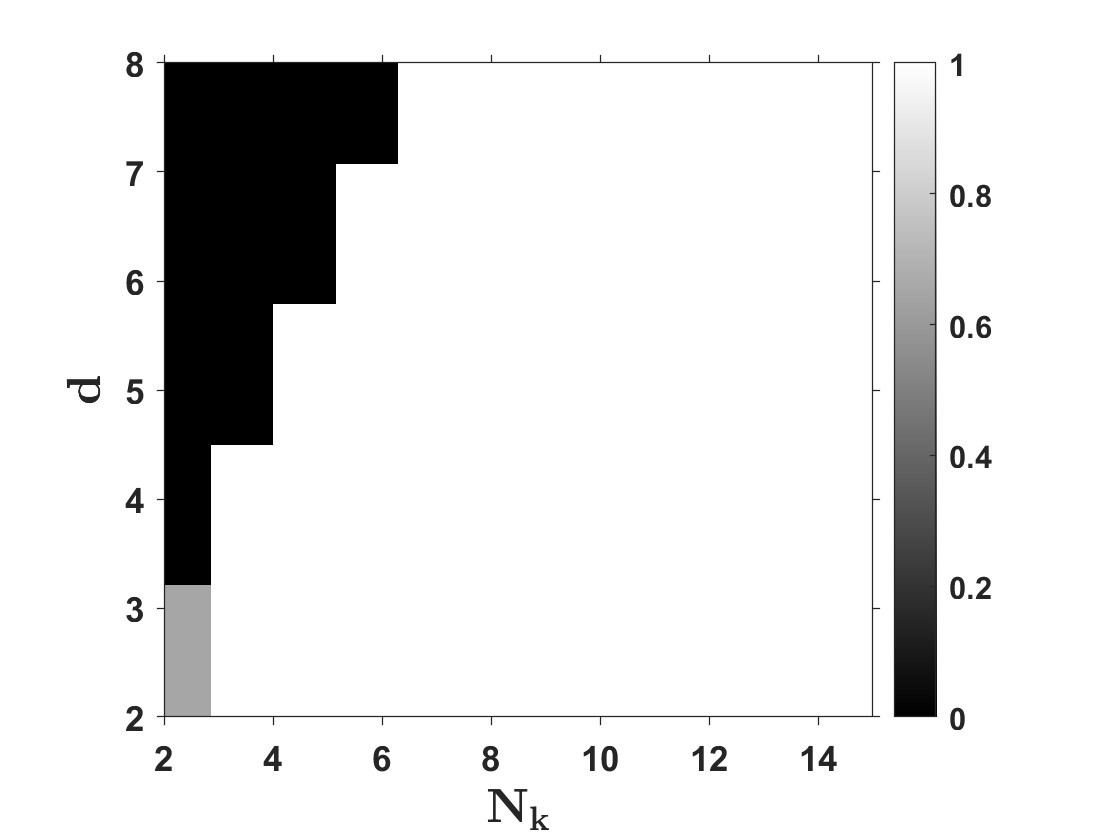}
    \caption{\bf PCA} 
    \end{subfigure}  
    \begin{subfigure}{0.43\textwidth}
    	\includegraphics[width =1\linewidth]{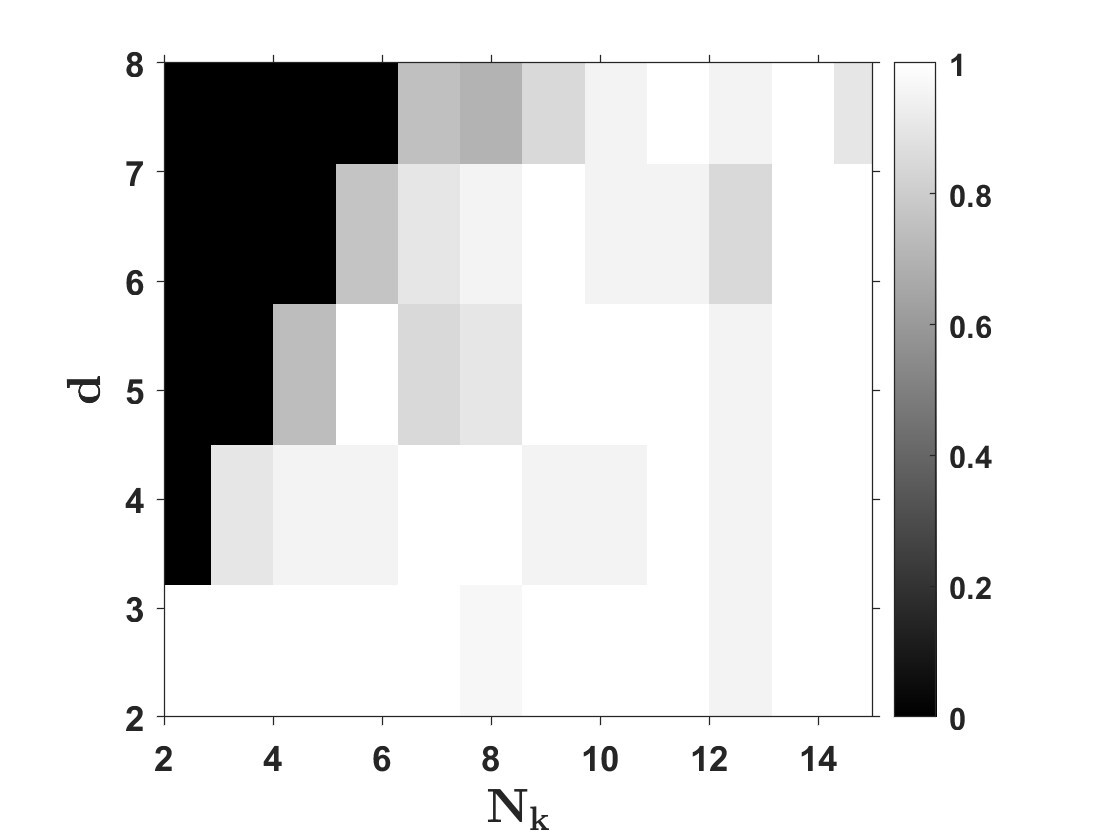}
     % \vspace{-0.05in}
    \caption{\bf Diffusion model} 
    \end{subfigure}
    \caption{\textbf{Phase transition of learning the \MoLRG~distribution when $K=3$.} The $x$-axis is the number of training samples and $y$-axis is the dimension of subspaces. We apply a subspace clustering method and train diffusion models for solving Problems \eqref{eq:SC} and \eqref{eq:em loss}, visualizing the results in (a) and (b), respectively.}  \label{fig:phase-transition-MoG-add-exp}
\end{center}
\end{figure*}

\rebuttal{For the CIFAR-10 Low-Frequency and FFHQ Low-Frequency datasets, we apply Gaussian blurring with a kernel size of 13 and a variance of 2.0 to extract the low-frequency components from CIFAR-10 and FFHQ.}

We set $K = 10$, $d_k = 20$ for MNIST and FashionMNIST, $K = 10$, $d_k = 200$ for CIFAR-10, and $K = 5$, $d_k = 500$ for FFHQ. Since FFHQ lacks annotated labels, we apply the expectation-maximization algorithm for clustering and label generation. For comparison, we use a Gaussian distribution with its mean and covariance set to the mean and covariance of all training samples. In addition, we train an EDM-based diffusion model \cite{karras2022elucidating} on each dataset as a comparison. We employ the second-order Heun Solver \cite{karras2022elucidating} with 35 steps as the diffusion sampler to numerically solve \eqref{eq:reve} and generate samples from learned distributions. For qualitative evaluation, we visualize samples from 12 initial noise inputs per dataset for both theoretical (\MoLRG\ and Gaussian) and real diffusion models. For quantitative evaluation, we generate 10K noise samples and compute the Euclidean distance between the theoretical and real model outputs (defined in \eqref{eq:dist}).

\section{Experimental Setups in Section \ref{sec:results}} \label{app sec:expsetting_sec3} 
 
In this section, we provide detailed setups for the experiment in \Cref{subsec:thm verification}. This experiment aims to validate the \Cref{thm:2} and \Cref{thm:4}. Here, we present the stochastic gradient descent (SGD) algorithm for solving Problem \eqref{eq:em loss} in \Cref{alg:1}.

Now, we specify how to choose the parameters of the SGD in our implementation. We divide the time interval $[0,1]$ into $64$ time steps. When $K = 1$, we set the learning rate $\eta =  10^{-4}$, batch size $M = 128 N_k$, and number of iterations $J = 10^4$. When $K = 2$, we set the learning rate $\eta = 2 \times 10^{-5}$, batch size $M = 1024$, number of iterations $J = 10^5$. % In practice, we find although theoretically there exists only one optimal solution $\bm U^*$ when $\bm N_k \geq d_k$, diffusion model with theoretically parameterized networks is hard to converge to this solution when $\bm N_k$ is close to $d_k$. Therefore, we use a special initialization $\bm U = \bm U_0$ when $K = 2$:
In particular, when $K=2$, we use the inner-product based spectral method (TIPS) initialization \cite{wang2022convergence} to improve the convergence of SGD. Specifically, TIPS generates the initial clusters based on the pairwise similarities of $\{\vx^{(i)}\}_{i = 1}^{N}$ and then estimates the initialization $\bm \theta^0 = \{\bm U_k^0\}$ according to the clustering results. Further details can be found in Section 2.3 of \cite{wang2022convergence}.
% \begin{align}
%     \label{eq:initialization}
%     \bm U_k^0 = \bm U_k^\star + 0.2 \bm \Delta,\ k \in \{1,2\},
% \end{align}
% where $\bm \Delta \sim \mathcal{N}(\bm 0, \bm I_n)$. 
% Notably, we argue that this initialization is still valid to verify our theorem. When $\bm N_k < d_k$, there exists infinite optimal solution so even under the initialization, the diffusion model will not converge to $\bm U^*$.
We calculate the success rate as follows. If the returned subspace basis matrices $\{\bm U_k\}_{k=1}^K$ satisfy 
\begin{align*}
    \frac{1}{K} \sum_{k=1}^K \left\|\bm U_{\Pi(k)} \bm U_{\Pi(k)}^T - \bm U_k^{\star} \bm U_k^{\star T}\right\| \le 0.5
\end{align*}
for some permutation $\Pi:[K]\to [K]$, it is considered successful.

\begin{algorithm}[t]
   \caption{SGD for optimizing the training loss \eqref{eq:em loss}}
   \label{alg:1}
   {\bfseries Input:} Training samples $\{\bm x^{(i)}\}_{i=1}^N$ \\
    \textbf{for} $j=0,1,2,\dots,J$ \textbf{do} \\
    1. Randomly select $\{(i_m,t_m)\}_{m=1}^M$, where $i_m \in [N]$ and $t_m \in (0,1)$   and a noise $\bm \epsilon \sim \mathcal{N}(\bm 0, \bm I)$ \\
   2. Take a gradient step 
   \begin{align*}
       \bm \theta^{j+1} \leftarrow \bm \theta^{j} - \frac{\eta}{M} \sum_{m \in [M]} \nabla_{\bm \theta} \left\| \bm x_{\bm \theta^j}(s_{t_m} \bm x^{(i_m)} + \gamma_{t_m}\bm \epsilon, t_m) -  \bm x^{(i_m)} \right\|^2
   \end{align*}
    \textbf{end for}
\end{algorithm}

\section{Experimental Setups in Section \ref{sec:experiments}} \label{app sec:expset_appli_1}

In this section, we provide detailed setups for the experiments in \Cref{sec:experiments}. Specifically, we describe the settings for using a U-Net-based diffusion model to (1) learn \MoLRG\ distribution (Appendix \ref{app:exp_setting_MoLRG_unet}), (2) learn real-world image distribution (Appendix \ref{app:exp_setting_real_unet}), and (3) estimate the intrinsic dimension of real-world image distribution (Appendix \ref{app:exp_setting_real_dataset_rank}).

\subsection{Learning the \MoLRG~distribution with U-Net}
\label{app:exp_setting_MoLRG_unet}

In our implementation, we set $\mathrm{ID} \in \{8, 10, 12\}$.
\begin{itemize}
    \item When $\mathrm{ID} = 8$, $N_k \in\{20, 50, 70,  100, 200, 300, 1000\}$; \item When $\mathrm{ID} = 10$, $N_k \in\{100, 150, 200, 250, 300, 1000\}$;
    \item When $\mathrm{ID} = 12$, $N_k \in\{100, 150, 200, 250, 300, 350, 400, 1000\}$.
\end{itemize}
To train U-Net, we use the stochastic gradient descent in \Cref{alg:1}. We use DDPM++ architecture \cite{song2020score} for the U-Net and EDM \cite{karras2022elucidating} noise scheduler. We set the learning rate $10^{-3}$, batch size $64$, and number of iterations $ J = 10^{4}$.  

For a specific \MoLRG\ distribution $p_{\text{data}}$ with $N$ pre-selected training data $\bm x^{(i)} \sim p_{\text{data}}$, the threshold $\delta$ is chosen such that the following inequality holds: 
\begin{equation} \label{eq:def_delta}
    \frac{1}{M_z}\sum_{j=1}^{M_z} \mathbb{I}\left(\min_{i\in[N]} \|\bm \Psi \left(\bm x^{(i)}\right) -  \bm \Psi \left(\bm z^{(j)}\right)\| \geq \delta\right)  = 0.95.    
\end{equation}
Intuitively, this definition ensures that with $95\%$ probability, a newly drawn sample $\bm z^{(j)} \sim p_{\text{data}}$ will be at least $\delta$ away (in the $\bm \Psi$-transformation space) from its nearest neighbor among the training samples $\bm x^{(i)}$. While \eqref{eq:def_delta} has a theoretical analytical solution, we approximate $\delta$ numerically in practice. Specifically, we set $M_z = 10^3$, compute the minimum distance $\min_{i\in[N]} ||\bm \Psi \left(\bm x\right) -  \bm \Psi \left(\bm y_i\right)||_2$ for each, and set $\delta$ as the 5th percentile (i.e., the 0.05-quantile) of the resulting distance distribution. To empirically estimate \eqref{def:gl_score}, we set $M = 10^3$. 

To quantitatively estimate the transition in \Cref{fig:phase-transition-UNet} (top-left), we fit the curve using the following sigmoid-parameterized function:
\begin{equation}\label{eq:sigmoid-shaped function}
    \mathrm{GL}\left(\frac{N}{\mathrm{ID}}\right) \approx f_{\MoLRG}\left(\frac{N}{\mathrm{ID}}\right) = \frac{1}{1 + \exp\left(-a(\log_2\left(N/\mathrm{ID}\right) - b)\right)},
\end{equation}
where the fitted parameters are $a = 6.22$ and $b = 5.20$. Solving this equation numerically gives $f_{\MoLRG}^{-1}\left(0.95\right) = 50.2$, indicating that U-Net architectures trained on the \MoLRG\ distribution generalize when $N_k \geq 50.2 d_k$. We use the same parameterized function \eqref{eq:sigmoid-shaped function} for the fitted curve in \Cref{fig:phase-transition-UNet} (bottom-left), by changing the input variable to $N/\mathrm{ID}^2$.

\subsection{Learning real-world image data distributions with U-Net}
\label{app:exp_setting_real_unet}
To train diffusion models for real-world image datasets, we use the DDPM++ architecture \cite{song2020score} for U-Net and variance preserving (VP) \cite{song2020score} noise scheduler. The U-Net is trained using the Adam optimizer \cite{kingma2014adam}, a variant of SGD in \Cref{alg:1}. We set the learning rate $\eta = 10^{-3}$, batch size $M = 512$, and the total number of iterations $10^5$. To empirically estimate \eqref{def:gl_score}, we set $M = 10^4$. 

The curve $f_{\texttt{real}}$ is parameterized the same way as $f_{\MoLRG}$ in \eqref{eq:sigmoid-shaped function}, with $a = 1.88$ and $b=7.74$. Then, we numerically solve that $f^{-1}_{\texttt{real}}\left(0.95\right) = 630.3$, indicating that U-Net architectures trained on real data distribution generalize when $N \geq 630.3 \mathrm{ID}$. We use the same parameterized function \eqref{eq:sigmoid-shaped function} for the fitted curve in \Cref{fig:phase-transition-UNet} (bottom-right), by changing the input variable to $N/\mathrm{ID}^2$.

\subsection{Estimating the intrinsic dimension of real-world dataset} \label{app:exp_setting_real_dataset_rank}

In this subsection, we conduct numerical experiments to estimate the intrinsic dimension of real-world image data distribution. Following from \Cref{lem:E[x_0]}, as $t \rightarrow 1$,  we have
\begin{align}
    {\color{black}\nabla_{\bm x_t} \mathbb{E}[\bm x_0\mid \bm x_t]} \approx \frac{1}{s_t}\sum_{k = 1}^{K}\bm U_k^\star\bm D_{k,t}^\star\bm U_k^{\star T},  
\end{align}
given $w_{k,t}^*(\bm x_t) \approx 1$ and $\nabla_{\bm x_t} w_{k,t}^*(\bm x_t) \approx \bm 0$. This relationship allows us to estimate the intrinsic dimension $\mathrm{ID}$ of a \MoLRG\ distribution as: 
\begin{equation}
    \mathrm{ID} \coloneqq \rank\left(\sum_{k = 1}^{K}\bm U_k^\star\bm D_{k,t}^\star\bm U_k^{\star T}\right) \approx \rank\left(\nabla_{\bm x_t} \mathbb{E}[\bm x_0\mid \bm x_t]\right),
\end{equation}
as $t \rightarrow 1$. Note that the DAE $\bm x_{\bm \theta}(\cdot,t)$ of the trained diffusion models satisfies $\bm x_{\bm \theta}(\bm x_t,t) \approx \mathbb{E}[\bm x_0 \mid \bm x_t]$. Combining this with the observation in \Cref{subsec:exp-real} that real-world image distributions can be well approximated by \MoLRG\ distributions, we conclude that the intrinsic dimension can be estimated by:
\begin{equation}\label{eq:estimate_id}
    \mathrm{ID} \approx \rank\left(\nabla_{\bm x_{t}}\bm x_{\bm \theta}(\bm x_t, t)\right).
\end{equation}

We evaluate the intrinsic dimension $\mathrm{ID}$ over four different datasets: CIFAR-10, CelebA, FFHQ, and AFHQ. We resize images from FFHQ and AFHQ such that $n = 3072$ for all datasets. We calculate the numerical rank of the Jacobian $\nabla_{\bm x_{t}}\bm x_{\bm \theta}(\bm x_t, t)$ through
\begin{align}\label{eq:numerical_rank}
    \rank\left(\nabla_{\bm x_{t}}\bm x_{\bm \theta}(\bm x_t, t)\right) := \argmin \left\{r \in [1,n]:  \frac{\sum_{i = 1}^{r} \sigma_i^2 \left(\nabla_{\bm x_{t}}\bm x_{\bm \theta}(\bm x_t, t)\right)}{\sum_{i = 1}^{n} \sigma_i^2 \left(\nabla_{\bm x_{t}}\bm x_{\bm \theta}(\bm x_t, t)\right)} > \eta^2 \right\},
\end{align} 
with $\eta = 0.99$ and recall $\sigma_i\left( \bm A\right)$ denotes the $i$-th singular value of matrix $\bm A$.

\begin{figure*}[t]
\begin{center}
	\includegraphics[width=.5\textwidth]{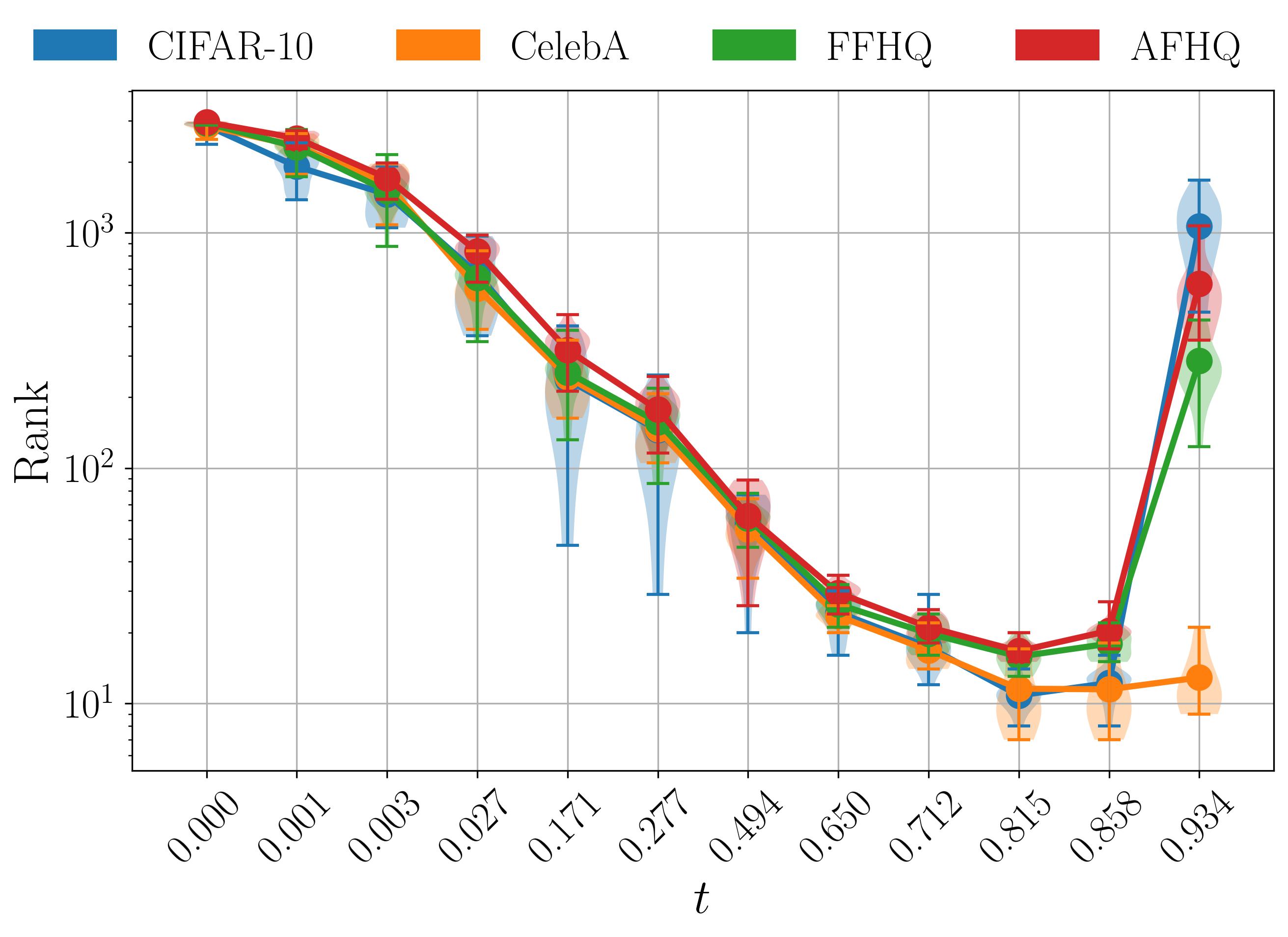}
    \caption{{\bf Low-rank property of the denoising autoencoder of trained diffusion models.} We plot the numerical rank of the Jacobian of the denoising autoencoder, i.e., $\nabla_{\bm x_t} \bm x_{\bm \theta}(\bm x_t, t)$, against the timestep $t$ by training diffusion models on different datasets. We train diffusion models on image datasets CIFAR-10, CelebA, FFHQ, and AFHQ. The experimental details are provided in Appendix \ref{app:exp_setting_real_dataset_rank}. }  \label{fig:low-rank-real-data} 
\end{center}
\end{figure*} 

To select a timestep $t$ to estimate the intrinsic dimension, we evaluate $\rank\left(\nabla_{\bm x_{t}}\bm x_{\bm \theta}(\bm x_t, t)\right)$ at different timesteps across multiple datasets, as shown in \Cref{fig:low-rank-real-data}. Specifically, Given a random initial noise $\bm x_1 \sim \mathcal{N}(\bm 0, \bm I_n)$, we use the diffusion model to generate a sequence of images $\{\bm x_t\}$ according to the reverse ODE in \eqref{eq:reve}. Along the sampling trajectory $\{\bm x_t\}$, we estimate $\rank\left(\nabla_{\bm x_{t}}\bm x_{\bm \theta}(\bm x_t, t)\right)$ at each timestep. For the experiments, we utilize the Elucidating Diffusion Model (EDM) with the EDM noise scheduler \cite{karras2022elucidating} and DDPM++ architecture \cite{song2020denoising}. Moreover, we employ an 18-step Heun's solver for sampling and present the results for 12 of these steps ($t = 0, 0.001, 0.003, 0.027, 0.171, 0.277$, $0.494, 0.650, 0.712, 0.815, 0.858, 0.934$). For each dataset, we random sample 15 initial noise $\bm x_1$, calculate the mean of $\rank(\nabla_{\bm x_{t}}\bm x_{\bm \theta}(\bm x_t, t))$ along the trajectory $\{\bm x_t\}$.

As shown in \Cref{fig:low-rank-real-data}, the plot of rank against $t$ exhibits a U-shaped curve, with the lowest rank consistently occurring around $t = 0.815$ across all datasets. Timesteps too close to $t = 0$ or $t = 1$ are unsuitable for estimating the intrinsic dimension: when $t$ approaches $1$, $\bm x_{\bm \theta}(\bm x_t, t)$ becomes less accurate because the training loss \eqref{eq:em loss} assigns a small weight $\lambda_t$ to such timesteps;  when $t$ approaches $0$, \cite{karras2022elucidating} parameterize $\bm x_{\bm \theta}(\bm x_t, t)$ to be $\bm x_t$, causing $\rank\left(\nabla_{\bm x_{t}}\bm x_{\bm \theta}(\bm x_t, t)\right) = n$ to be naturally very high. Thus, we select $t = 0.815$, the timestep that achieves the lowest rank, to estimate the intrinsic dimension of real-world datasets. The estimated $\mathrm{ID}$ is shown in \Cref{tab:id}.

\begin{figure}[t]
    \centering
    \begin{subfigure}{0.55\textwidth}
        \includegraphics[width=1\textwidth]{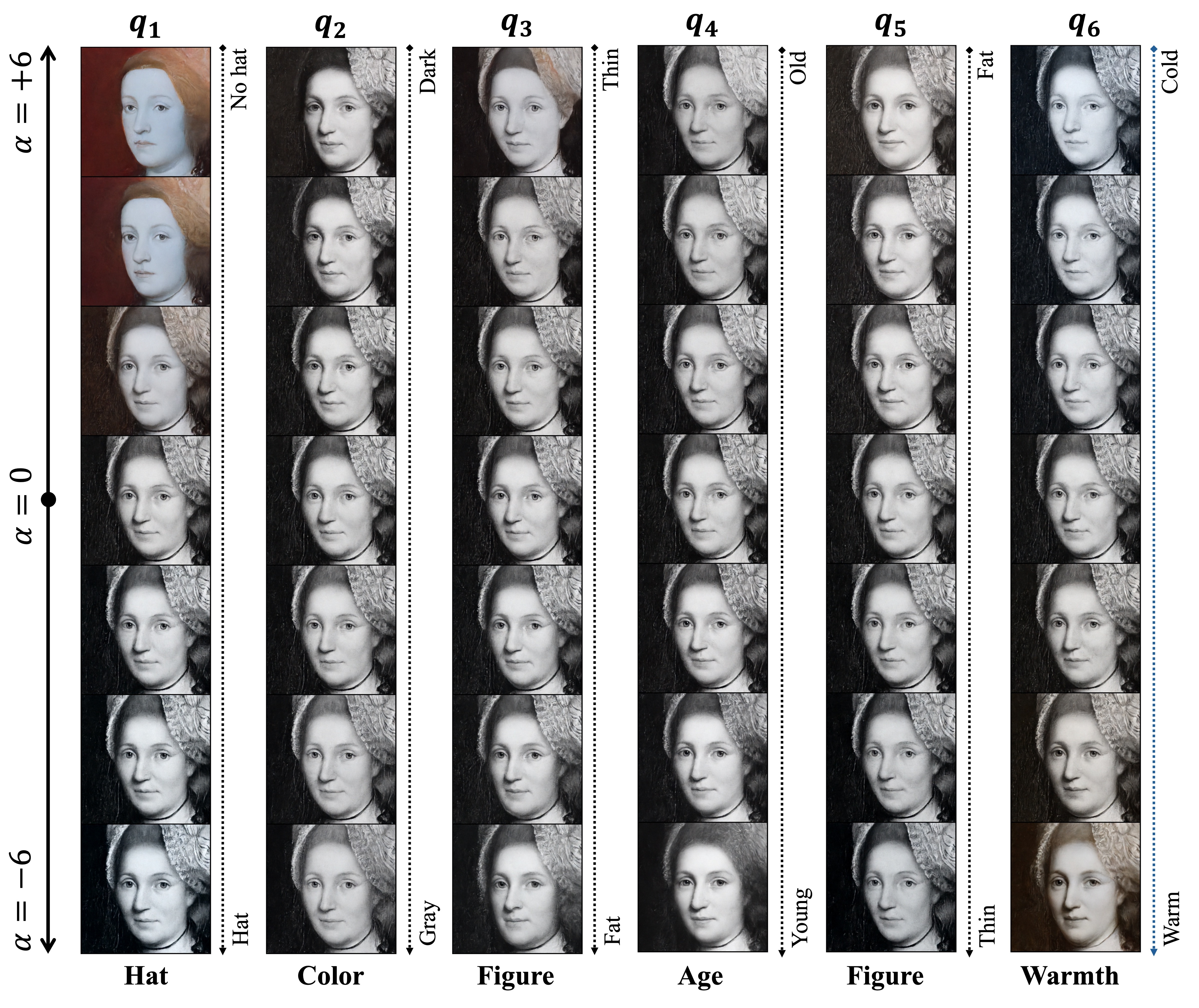}
    \caption{}
    \end{subfigure} 
    \begin{subfigure}{0.216\textwidth}
        \includegraphics[width =1\linewidth]{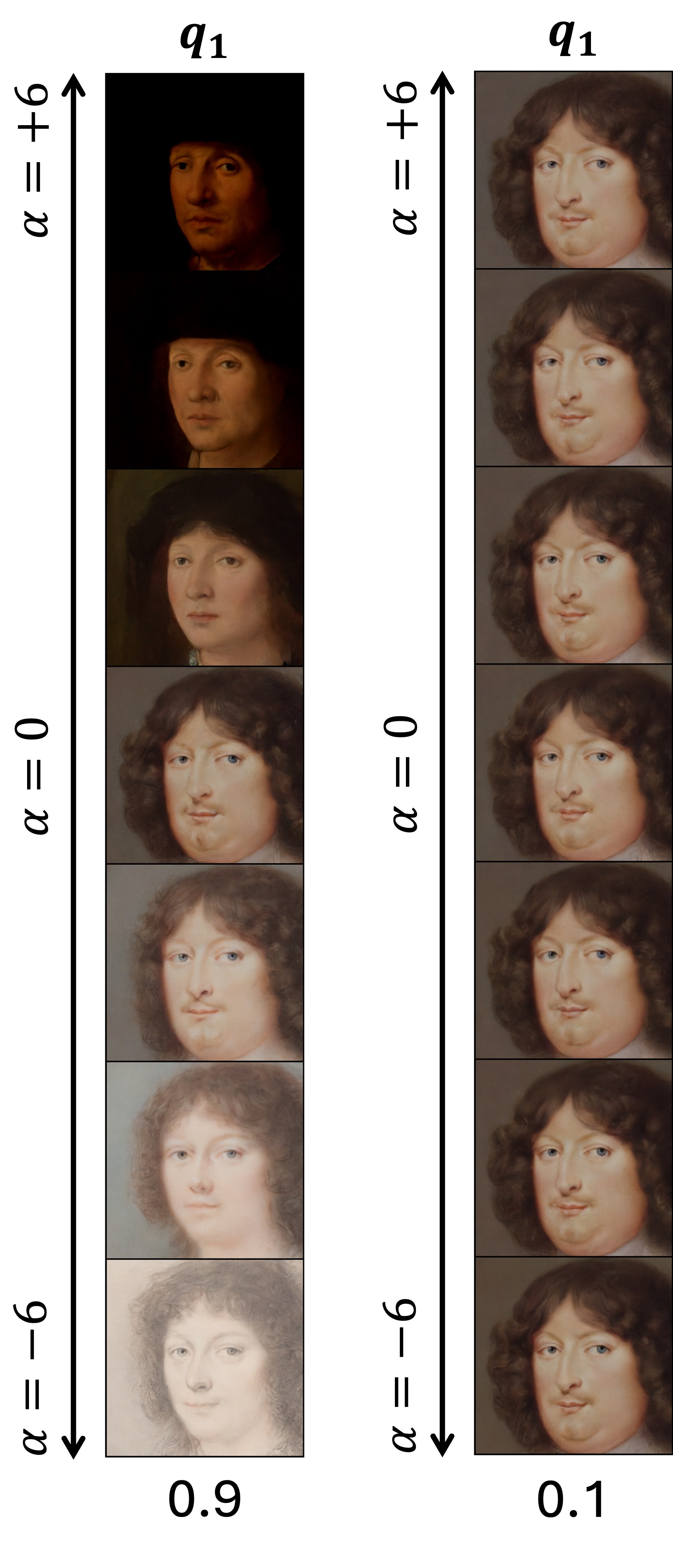}
    \caption{}
    \end{subfigure} \\
    \begin{subfigure}{0.55\textwidth}
        \includegraphics[width =1\linewidth]{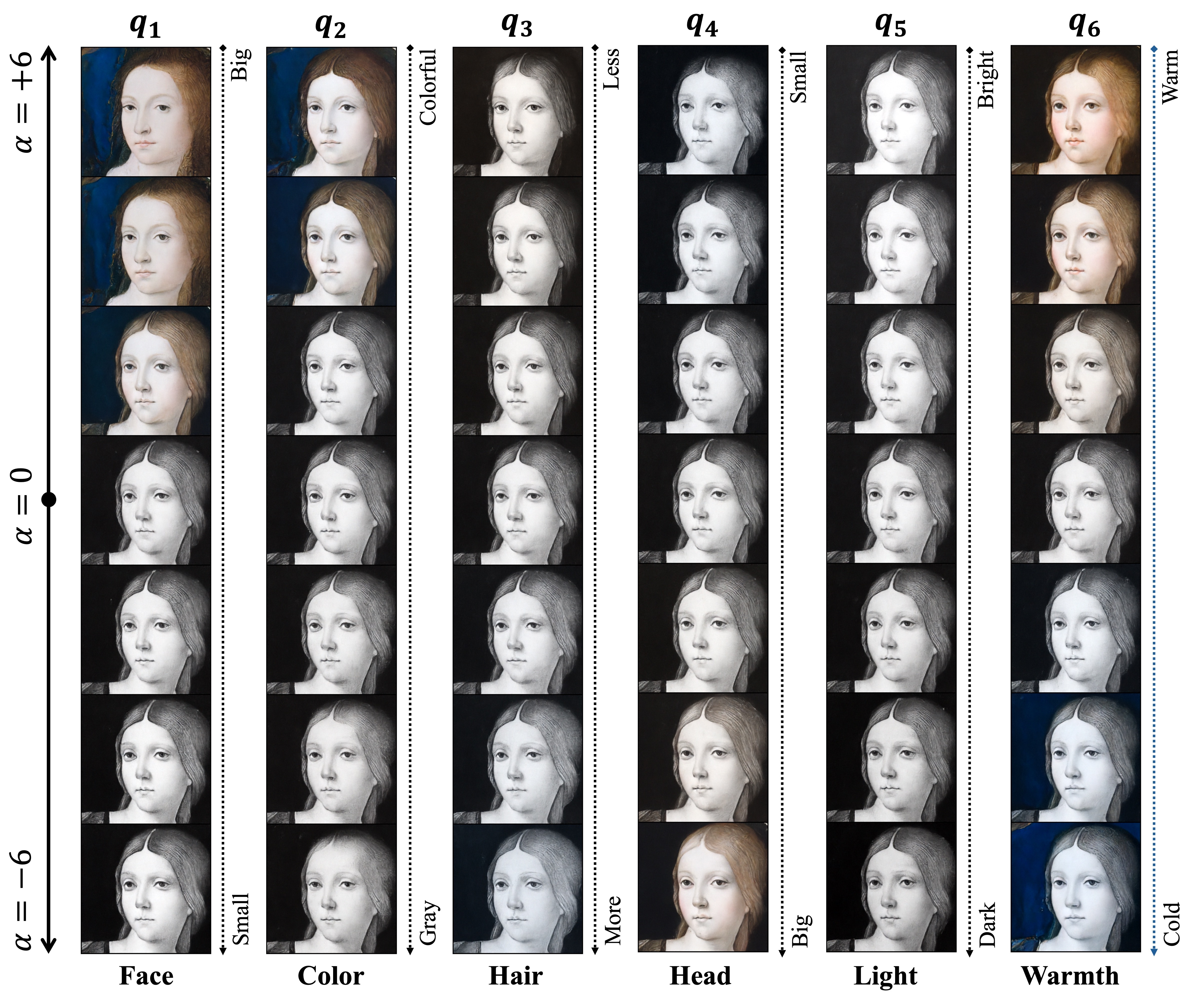}
   \caption{}
    \end{subfigure} 
    \caption{\textbf{Correspondence between the singular vectors of the Jacobian of the DAE and semantic image attributes.} (a,c) Additional examples when $t = 0.7$. (b) Ablation studies when $t = 0.1$ and $0.9$.}
    \label{fig:meta_more}
\end{figure}

\section{Auxiliary Results}\label{app sec:auxi}

First, we present a probabilistic result to prove \Cref{thm:2}, which provides an optimal estimate of the small singular values of a matrix with i.i.d. Gaussian entries. This lemma is proved in \cite[Theorem 1.1]{rudelson2009smallest} for subgaussian random variables. Note that a random variable $\xi$ is called subgaussian if there exists $c > 0$ such that for all $ t > 0$, 
\begin{align*}
    \mathbb{P}\left(|\xi| > t\right) \le 2\exp\left( -\frac{t^2}{c^2} \right).
\end{align*}
We say that the minimal $c$ in this inequality is the subgaussian moment of $\xi$. 

\begin{lemma}\label{lem:Gau}
     Let $\bm A$ be an $m \times n$ random matrix, where $m \ge n$, whose elements are independent copies of a subgaussian random variable with mean zero and unit variance. It holds for every $\varepsilon > 0$ that 
     \begin{align*}
         \mathbb{P}\left( \sigma_{\min}(\bm A) \ge \varepsilon(\sqrt{m}-\sqrt{n-1}) \right) \ge 1 - \left(c_1\varepsilon\right)^{m-n+1} - \exp\left(-c_2m\right),
     \end{align*}
     where $c_1,c_2 > 0$ are constants depending polynomially only on the subgaussian moment.  
\end{lemma}
Next, we present a probabilistic bound on the deviation of the norm of a weighted sum of squared Gaussian random variables from its mean. This is a direct extension of \cite[Theorem 5.2.2]{vershynin2018high}. 

\begin{lemma}\label{lem:norm gaus}
Let $\bx \sim \mN(\bm 0, \bI_d)$ be a Gaussian random vector and $\lambda_1,\dots,\lambda_d > 0$ be constants. It holds for any $t > 0$ that  
\begin{align}\label{eq:norm gaus}
\mathbb{P}\left(  \left|\sqrt{\sum_{i=1}^d \lambda_i^2x_i^2} - \sqrt{\sum_{i=1}^d \lambda_i^2}\right| \ge t + 2 \lambda_{\max} \right) \le 2\exp\left( -\frac{t^2}{2\lambda_{\max}^2} \right),
\end{align}
where $\lambda_{\max} = \max\{\lambda_i: i \in [d] \}$. 
\end{lemma}

Based on the above lemma, we can further show the following concentration inequalities to estimate the norm of the standard normal Gaussian random vector. 
\begin{lemma}\label{lem:norm Gau}
    Suppose that $\bm a_i \overset{i.i.d.}{\sim} \mathcal{N}(\bm 0,\bm I_d)$ is a Gaussian random vector for each $i \in [N]$. % The following statements hold: \\
    It holds for all $i \in [N]$ with probability at least $1 - 2N^{-1}$ that
    \begin{align}\label{eq:norm Gau}
        \left|\|\bm a_i\| - \sqrt{d} \right| \le 2\sqrt{\log N} + 2.
    \end{align}   
    % (ii) Let $\bm V \in \mO^{n\times d}$ be given. For all $i \in C_k^\star$ and all $k \in [K]$, it holds with probability at least $1-2N^{-1}$ that 
    % \begin{align}\label{eq:norm Ua}
    % \left| \|\bm V^T\bm U_k^\star\bm a_i\| - \|\bm V^T\bm U_k^\star\|_F  \right| \le 2\sqrt{\log N} + 2.
    % \end{align} 
\end{lemma}
\begin{proof}
    Applying \Cref{lem:norm gaus} to $\bm{a}_i \sim \mN(\bm{0},\bm{I}_{d})$, together with setting $t=2\sqrt{\log N}$ and $\lambda_j=1$ for all $j\in [d]$, yields  
    \begin{align*} 
    \mathbb{P}\left(\left|\| \ba_i\| - \sqrt{d}\right| \ge 2\sqrt{\log N} + 2 \right) \le 2N^{-2}.
    \end{align*} 
    This, together with the union bound, yields that \eqref{eq:norm Gau} holds with probability $1 - 2N^{-1}$.  
\end{proof}    

Next, we present a spectral concentration bound for sample covariance matrices formed from arbitrary subsets of Gaussian random vectors; see \cite[Theorem 1]{liu2026concentrationinequalitycovariancematrix}.

\begin{lemma}\label{lem:cov arbitrary subset}
Let \(\{\bm a_i\}_{i=1}^M\subseteq\mathbb R^d\) be i.i.d. standard Gaussian random vectors. Fix \(\kappa\in(0,1]\) and \(\eta\in(0,1/2)\), and define
\[
\gamma := \Phi^{-1}\left(\frac{1+(1-\eta)\kappa}{2}\right),
\quad
\mu := (1-\eta)\kappa
-\sqrt{\frac{2\gamma^2}{\pi}}\exp\left(-\frac{\gamma^2}{2}\right).
\]
Then there exists an absolute constant \(c>0\) such that, with probability at least
\[
1 - 2\left(\frac{22}{\eta\sqrt{\mu}}+1\right)^d
\left( \exp\left(-\frac{\eta^2\kappa M}{2}\right)
+ \exp\left(-c\eta^2\mu^2 M\right)
\right) - 2\exp(-2M),
\] 
it holds for every subset \(\Omega\subseteq[M]\) satisfying \(|\Omega|\ge \kappa M\) that 
\[
\lambda_{\min}\left(\sum_{i\in\Omega}\bm a_i\bm a_i^T\right)
\ge (1-\eta)^2\mu M.
\]
\end{lemma}

\begin{lemma}\label{lem:trace}
    Let $\bm A, \bm B \in \R^{n\times n}$  be positive semi-definite matrices. Then, it holds that
    \begin{align}
        \langle \bm A, \bm B \rangle \ge \lambda_{\min}(\bm A) \mathrm{Tr}(\bm B).
    \end{align}
\end{lemma}
\begin{proof}
    Let $\bm U\bm \Lambda\bm U^T = \bm A$ be an eigenvalue decomposition of $\bm A$, where $\bm U \in \mathcal{O}^n$ and $\bm \Lambda =\mathrm{diag}(\lambda_1,\dots,\lambda_n)$ is a diagonal matrix with diagonal entries $\lambda_1 \ge \dots \ge \lambda_n \ge 0$ being the eigenvalues. Then, we compute
    \begin{align*}
        \langle \bm A, \bm B \rangle = \langle \bm U\bm \Lambda\bm U^T, \bm B \rangle = \langle \bm \Lambda, \bm U\bm B\bm U^T \rangle \ge \lambda_{\min}(\bm A) \mathrm{Tr}(\bm U\bm B\bm U^T) = \lambda_{\min}(\bm A) \mathrm{Tr}(\bm B),
    \end{align*}
    where the inequality follows from $\lambda_i\ge 0$ for all $i \in [n]$ and $\bm B$ is a positive semidefinite matrix. 
\end{proof}

\end{document}